\documentclass[ejs]{imsart}

\RequirePackage{amsthm,amsmath,amsfonts,amssymb}
\RequirePackage[numbers]{natbib}
\RequirePackage[colorlinks,citecolor=blue,urlcolor=blue]{hyperref}
\RequirePackage{graphicx, subcaption, caption}
\usepackage{mathtools}
\usepackage{subcaption, graphicx}

\arxiv{2007.06827}
\startlocaldefs
\theoremstyle{plain}

\newtheorem{theorem}{Theorem}[section]
\newtheorem{lemma}[theorem]{Lemma}
\theoremstyle{definition}
\newtheorem{definition}[theorem]{Definition}
\newtheorem{assumption}{Assumption}
\newtheorem{corollary}[theorem]{Corollary}
\theoremstyle{example}
\newtheorem{example}{Example}
\theoremstyle{remark}
\newtheorem*{remark}{Remark}

\theoremstyle{remark}

\DeclarePairedDelimiter\ceil{\lceil}{\rceil}
\DeclarePairedDelimiter\floor{\lfloor}{\rfloor}

\endlocaldefs

\begin{document}
\begin{frontmatter}
\title{Early stopping and polynomial smoothing in regression with reproducing kernels}
\runtitle{Early stopping and polynomial smoothing}

\begin{aug}
\author[A]{\fnms{Yaroslav}~\snm{Averyanov}\ead[label=e1]{yaroslavmipt@gmail.com}}
\and
\author[B]{\fnms{Alain}~\snm{Celisse}\ead[label=e2]{alain.celisse@univ-paris1.fr}}
\address[A]{Inria MODAL project-team\printead[presep={,\ }]{e1}}

\address[B]{Laboratoire SAMM,
Paris 1 Panthéon-Sorbonne University\printead[presep={,\ }]{e2}}
\runauthor{Y. Averyanov and A. Celisse}
\end{aug}

\begin{abstract}
In this paper, we study the problem of early stopping for iterative learning algorithms in a reproducing kernel Hilbert space (RKHS) in the nonparametric regression framework. In particular, we work with the gradient descent and (iterative) kernel ridge regression algorithms. We present a \textit{data-driven} rule to perform early stopping without a validation set that is based on the so-called minimum discrepancy principle. This method enjoys only one assumption on the regression function: it belongs to a reproducing kernel Hilbert space (RKHS). The proposed rule is proved to be minimax-optimal over different types of kernel spaces, including finite-rank and Sobolev smoothness classes. The proof is derived from the fixed-point analysis of the localized Rademacher complexities, which is a standard technique for obtaining optimal rates in the nonparametric regression literature. In addition to that, we present simulation results on artificial datasets that show the comparable performance of the designed rule with respect to other stopping rules such as the one determined by $V-$fold cross-validation.
\end{abstract}

\begin{keyword}[class=MSC]
\kwd[Primary ]{62G05}
\kwd[; secondary ]{62G08}
\end{keyword}

\begin{keyword}
\kwd{Nonparametric regression}
\kwd{Reproducing kernels}
\kwd{Early stopping}
\kwd{Localized Rademacher complexities}
\end{keyword}

\end{frontmatter}

\section{Introduction}

\textit{Early stopping rule} (ESR) is a form of regularization based on choosing when to stop an iterative algorithm based on some design criterion. Its main idea is lowering the computational complexity of an iterative algorithm while preserving its statistical optimality. This approach is quite old and initially was developed for Landweber iterations to solve ill-posed matrix problems in the 1970s \cite{engl1996regularization, wahba1987three}. 
%
Recent papers provided some insights for the connection between early stopping and boosting methods \cite{bartlett2007adaboost, buhlmann2003boosting, wei2017early, zhang2005boosting}, gradient descent, and Tikhonov regularization in a reproducing kernel Hilbert space (RKHS) \cite{bauer2007regularization, raskutti2014early, Yao2007}. For instance, \cite{buhlmann2003boosting} established the first optimal in-sample convergence rate of $L^2$-boosting with early stopping. Raskutti et al. \cite{raskutti2014early} provided a result on a stopping rule that achieves the minimax-optimal rate for kernelized gradient descent and ridge regression over different smoothness classes. This work established an important connection between the localized Radamacher complexities \cite{bartlett2005local, koltchinskii2006local, wainwright2019high}, that characterizes the size of the explored function space, and early stopping. The main drawback of the result is that one needs to know the RKHS-norm of the regression function or its tight upper bound in order to apply this early stopping rule in practice. Besides that, this rule is design-dependent, which limits its practical application. In the subsequent work, \cite{wei2017early} showed how to control early stopping optimality via the localized Gaussian complexities in RKHS for different boosting algorithms ($L^2$-boosting, LogitBoost, and AdaBoost). Another theoretical result for a not data-driven ESR was built by \cite{blanchard2016convergence}, where the authors proved a minimax-optimal (in the ${L}_2(\mathbb{P}_X)$ out-of-sample norm) stopping rule for conjugate gradient descent in the nonparametric regression setting. \cite{2015arXiv151005684A} proposed a different approach, where the authors focused on both time/memory computational savings, combining early stopping with Nystrom subsampling technique. 

Some stopping rules, that (potentially) could be applied in practice, were provided by \cite{blanchard2018optimal, blanchard2018early} and \cite{stankewitz2019smoothed}, and were based on the so-called \textit{minimum discrepancy principle} \cite{blanchard2016convergence, blanchard2012discrepancy, engl1996regularization, hansen2010discrete}. This principle consists of monitoring the empirical risk and determining the first time at which a given learning algorithm starts to fit the noise. In the papers mentioned, the authors considered spectral filter estimators such as gradient descent, Tikhonov (ridge) regularization, and spectral cut-off regression for the linear Gaussian sequence model, and derived several oracle-type inequalities for the proposed ESR. The main deficiency of the works \cite{blanchard2018optimal, blanchard2018early, stankewitz2019smoothed} is that the authors dealt only with the linear Gaussian sequence model, and the minimax optimality result was restricted to the spectral cut-off estimator. It is worth mentioning that \cite{stankewitz2019smoothed} introduced the so-called \textit{polynomial smoothing} strategy to achieve the optimality of the minimum discrepancy principle ESR over Sobolev balls for the spectral cut-off estimator. More recently, \cite{celisse2021analyzing} studied a minimum discrepancy principle stopping rule and its modified (they called it smoothed as well) version, where they provided the range of values of the regression function regularity, for which these stopping rules are optimal for different spectral filter estimators in RKHS. 

\textbf{Contribution.} Hence, to the best of our knowledge, there is no \textit{fully data-driven} stopping rule for gradient descent or ridge regression in RKHS that does not use a validation set, does not depend on the parameters of the model such as the RKHS-norm of the regression function, and explains why it is statistically optimal. In our paper, we combine techniques from \cite{blanchard2018optimal}, \cite{raskutti2014early}, and \cite{stankewitz2019smoothed} to construct such an ESR. Our analysis is based on the bias and variance trade-off of an estimator, and we try to catch the iteration of their intersection by means of the \textit{minimum discrepancy principle} \cite{blanchard2018optimal, blanchard2012discrepancy, celisse2021analyzing} and the \textit{localized Rademacher complexities} \cite{bartlett2005local, koltchinskii2006local, mendelson2002geometric, wainwright2019high}. In particular, for the kernels with infinite rank, we propose to use a special technique \cite{blanchard2012discrepancy, stankewitz2019smoothed} for the empirical risk in order to reduce its variance. Further, we introduce new notions of \textit{smoothed empirical Rademacher complexity} and \textit{smoothed critical radius} to achieve minimax optimality bounds for the functional estimator based on the proposed rule. This can be done by solving the associated fixed-point equation. It implies that the bounds in our analysis cannot be improved (up to numeric constants). It is important to note that in the present paper, we establish an important connection between a smoothed version of the \textit{statistical dimension} of $n$-dimensional kernel matrix, introduced by \cite{yang2017randomized} for randomized projections in kernel ridge regression, with early stopping (see Section \ref{optimality_section} for more details). We also show how to estimate the variance $\sigma^2$ of the model, in particular, for the infinite-rank kernels. In the meanwhile, we provide experimental results on artificial data indicating the consistent performance of the proposed rules.

\textbf{Outline of the paper.} The organization of the paper is as follows. In Section \ref{sec:2}, we introduce the background on nonparametric regression and reproducing kernel Hilbert space. There, we explain the updates of two spectral filter iterative algorithms: gradient descent and (iterative) kernel ridge regression, that will be studied. In Section \ref{sec:3}, we clarify how to compute our first early stopping rule for finite-rank kernels and provide an oracle-type inequality (Theorem \ref{th:1}) and an upper bound for the risk error of this stopping rule with fixed covariates (Corollary \ref{corollary_empirical_norm}). After that, we present a similar upper bound for the risk error with random covariates (Theorem \ref{th:2}) that is proved to be minimax-rate optimal. By contrast, Section \ref{sec:4} is devoted to the development of a new stopping rule for infinite-rank kernels based on the \textit{polynomial smoothing} \cite{blanchard2012discrepancy, stankewitz2019smoothed} strategy. There, Theorem \ref{th:3} shows, under a quite general assumption on the eigenvalues of the kernel operator, a high probability upper bound for the performance of this stopping rule measured in the $L_2(\mathbb{P}_n)$ in-sample norm. In particular, this upper bound leads to minimax optimality over Sobolev smoothness classes. In Section \ref{sec:5}, we compare our stopping rules to other rules, such as methods using hold-out data and $V-$fold cross-validation. After that, we propose using a strategy for the estimation of the variance $\sigma^2$ of the regression model. Section \ref{sec:6} summarizes the content of the paper and describes some perspectives. Supplementary and more technical proofs are deferred to Appendix.

\section{Nonparametric regression and reproducing kernel framework} \label{sec:2}

\subsection{Probabilistic model and notation}

The context of the present work is that of nonparametric regression, where an i.i.d. sample $\{(x_i, y_i), \ i=1, \ldots, n \}$ of cardinality $n$ is given, with $x_i \in \mathcal{X} \ (\textnormal{feature space})$ and $\ y_i \in \mathbb{R}$. 
The goal is to estimate the regression function $f^*: \mathcal{X} \to \mathbb{R}$ from the model
\begin{equation}\label{main}
    y_i = f^*(x_i) + \overline{\varepsilon}_i, \qquad i = 1, \ldots, n,
\end{equation}
where the error variables $\overline{\varepsilon}_i$ are i.i.d. zero-mean Gaussian random variables $\mathcal{N}(0, \sigma^2)$, with $\sigma > 0$. In all what follows (except for Section \ref{sec:5}, where results of empirical experiments are reported), the values of $\sigma^2$ is assumed to be known as in \cite{raskutti2014early} and \cite{wei2017early}.

Along the paper, calculations are mainly derived in the \emph{fixed-design} context, where the $\{ x_i \}_{i=1}^n$ are assumed to be fixed, and only the error variables $\{ \overline{\varepsilon}_i \}_{i=1}^n$ are random. 
In this context, the performance of any estimator $\widehat{f}$ of the regression function $f^*$ is measured in terms of the so-called \emph{empirical norm}, that is, the $L_2(\mathbb{P}_n)$-norm defined by
\begin{equation*}
    \lVert \widehat{f} - f^* \rVert_n^2 \coloneqq \frac{1}{n}\sum_{i=1}^n \Big[ \widehat{f}(x_i) - f^*(x_i) \Big]^2 , 
\end{equation*}
where $\lVert h \rVert_n \coloneqq \sqrt{ 1/n \sum_{i=1}^n h(x_i)^{2} }$ for any bounded function $h$ over $\mathcal{X}$, and $\langle \cdot, \cdot \rangle_n$ denotes the related inner-product defined by $\langle h_1, h_2 \rangle_n \coloneqq 1/n \sum_{i=1}^n h_1(x_i) h_2(x_i)$ for any functions $h_1$ and $h_2$ bounded over $\mathcal{X}$.
In this context, $\mathbb{P}_{\varepsilon}$ and $\mathbb{E}_{\varepsilon}$ denote the probability and expectation, respectively, with respect to the $\{ \overline{\varepsilon}_i\}_{i=1}^n$.

By contrast, Section~\ref{sec.random.design} discusses some extensions of the previous results to the \emph{random design} context, where both the covariates $\{ x_i \}_{i=1}^n$ and the responses $\{ y_i \}_{i=1}^n$ are random variables. 
In this random design context, the performance of an estimator $\widehat{f}$ of $f^*$ is measured in terms of the $L_2(\mathbb{P}_X)$-norm defined by
\begin{equation*}
    \lVert \widehat{f} - f^* \rVert_2^2 \coloneqq \mathbb{E}_{X} \Big[ (\widehat{f}(X) - f^*(X))^2 \Big],
\end{equation*}
where $\mathbb{P}_X$ denotes the probability distribution of the $\{ x_i \}_{i=1}^n$.
In what follows,
$\mathbb{P}$ and $\mathbb{E}$, respectively, state for the probability and expectation with respect to the couples
$\{ (x_i,y_i) \}_{i=1}^n$.

\paragraph{Notation.} 
Throughout the paper, $\lVert \cdot \rVert$ and $\langle \cdot, \cdot \rangle$ are the usual Euclidean norm and inner product in $\mathbb{R}^n$. 
We shall write $a_n \lesssim b_n$ whenever $a_n \leq C b_n$ for some numeric constant $C > 0$ for all $n \geq 1$. $a_n \gtrsim b_n$ whenever $a_n \geq C b_n$ for some numeric constant $C > 0$ for all $n \geq 1$. Similarly, $a_n \asymp b_n$ means $a_n \lesssim b_n$ and $b_n \gtrsim a_n$. $\left[ M \right] \equiv \{1, \ldots, M \}$ for any $M \in \mathbb{N}$. For $a \geq 0$, we denote by $\floor*{a}$ the largest natural number that is smaller than or equal to $a$. We denote by $\ceil*{a}$ the smallest natural number that is greater than or equal to $a$. Throughout the paper, we use the notation $c, c_1, \widetilde{c}, C, \widetilde{C}, \ldots$ to show that numeric constants $c, c_1, \widetilde{c}, C, \widetilde{C}, \ldots$ do not depend on the parameters considered. Their values may change 
from line to line.

\subsection{Statistical model and assumptions}
\subsubsection{Reproducing Kernel Hilbert Space (RKHS)}
Let us start by introducing a reproducing kernel Hilbert space (RKHS) denoted by $\mathcal{H}$ \cite{aronszajn1950theory, berlinet2011reproducing, gu2013smoothing, wahba1990spline}.
Such a RKHS $\mathcal{H}$ is a class of functions associated with a \emph{reproducing kernel} $\mathbb{K}: \mathcal{X}^2 \to \mathbb{R}$ and endowed with an inner-product denoted by $\langle \cdot,\cdot \rangle_{\mathcal{H}}$, and satisfying $\langle \mathbb{K}(\cdot, x), \mathbb{K}(\cdot, y) \rangle_{\mathcal{H}} = \mathbb{K}(x,y)$ for all $x,y\in\mathcal{X}$.
Each function within $\mathcal{H}$ admits a representation as an element of $L_2(\mathbb{P}_{X})$, which justifies the slight abuse when writing $\mathcal{H} \subset L_2(\mathbb{P}_{X})$ (see \cite{cucker2002mathematical} and \cite[Assumption~3]{celisse2021analyzing}).

Assuming the RKHS $\mathcal{H}$ is separable, under suitable regularity conditions (e.g., a continuous positive-semidefinite kernel), Mercer's theorem \cite{scholkopf2001learning}  guarantees that the kernel can be expanded as
$$\mathbb{K}(x, x^\prime) = \sum_{k=1}^{+\infty} \mu_k \phi_k(x) \phi_k(x^\prime),\quad \forall  x,x^\prime \in\mathcal{X},$$ 
where $\mu_1 \geq \mu_2 \geq \ldots \geq 0$ and $\{ \phi_k \}_{k=1}^{+\infty}$ are, respectively, the eigenvalues and corresponding eigenfunctions of the kernel integral operator $T_{\mathbb{K}}$, given by
\begin{align}\label{kernel.integral.operator}
    T_{\mathbb{K}}(f)(x) = \int_{\mathcal{X}} \mathbb{K}(x,u) f(u) d\mathbb{P}_X(u),\quad \forall f \in L_2(\mathbb{P}_X), \ x \in \mathcal{X}.
\end{align}
It is then known that the family $\{ \phi_k \}_{k=1}^{+\infty}$ is an orthonormal basis of $L_2(\mathbb{P}_{X})$, while $\{ \sqrt{\mu_k} \phi_k \}_{k=1}^{+\infty}$ is an orthonormal basis of $\mathcal{H}$.
Then, any function $f \in \mathcal{H} \subset L_2(\mathbb{P}_{X})$ can be expanded as $f = \sum_{k=1}^{+\infty} \sqrt{\mu_k} \theta_k \phi_k$,
where for all $k$ such that $\mu_k > 0$, the coefficients $\{ \theta_k \}_{k=1}^{\infty}$ are 
\begin{equation} \label{coefficients}
    \theta_k = \langle f, \sqrt{\mu_k}\phi_k \rangle_{\mathcal{H}} = \frac{1}{\sqrt{\mu_k}} \langle f, \phi_k \rangle_{L_2(\mathbb{P}_X)} = \int_{\mathcal{X}} \frac{ f(x) \phi_k(x)}{\sqrt{\mu_k}} d \mathbb{P}_{X}(x).    
\end{equation}
Therefore, each functions $f,g \in \mathcal{H}$ can be represented by the respective sequences $\{ a_k \}_{k=1}^{+\infty}, \{ b_k \}_{k=1}^{+\infty} \in \ell_2(\mathbb{N})$ such that
\begin{equation*}
f = \sum_{k=1}^{+\infty} a_k \phi_k, \quad \mbox{and} \quad g = \sum_{k=1}^{+\infty} b_k \phi_k ,
\end{equation*}
with the inner-product in the Hilbert space $\mathcal{H}$ given by $\langle f, g \rangle_{\mathcal{H}} = \sum_{k=1}^{+\infty} \frac{a_k b_k}{\mu_k}.$ This leads to the following representation of $\mathcal{H}$ as an  ellipsoid
\begin{align*}
    \mathcal{H} = \left\{ f = \sum_{k=1}^{+\infty} a_k \phi_k,\quad \sum_{k=1}^{+\infty} a_k^2<+\infty, \mbox{ and } \sum_{k=1}^{+\infty} \frac{a_k^2}{\mu_k}<+\infty \right\} .
\end{align*}

\subsubsection{Main assumptions}

From the initial model given by Eq. \eqref{main}, we make the following assumption.
\begin{assumption}[Statistical model] \label{a1}
Let $\mathbb{K}(\cdot, \cdot)$ denote a reproducing kernel as defined above, and $\mathcal{H}$ is the induced separable RKHS. Then, there exists a constant $R > 0$ such that the $n$-sample $(x_1,y_1), \ldots,(x_n,y_n) \in\mathcal{X}^n \times \mathbb{R}^n$ satisfies the statistical model
\begin{align}\label{assum.a1}
    y_i = f^*(x_i) + \overline{\varepsilon}_i, \quad \mbox{with}\quad f^* \in \mathbb{B}_{\mathcal{H}}(R) = \{ f \in \mathcal{H}: \lVert f \rVert_{\mathcal{H}} \leq R \},
\end{align}
where the $\{ \overline{\varepsilon}_i \}_{i=1}^n$ are i.i.d. Gaussian random variables with $\mathbb{E}[\overline{\varepsilon}_i \mid x_i] = 0$ and $\mathbb{V}[\overline{\varepsilon}_i \mid x_i] = \sigma^2$.
\end{assumption}
The model from Assumption~\ref{a1} can be vectorized as
\begin{equation}\label{vector.model}
    Y = [y_1, ..., y_n]^\top = F^* + \overline{\varepsilon} \in \mathbb{R}^n,
\end{equation}
where $F^* = [f^*(x_1), \ldots, f^*(x_n)]^\top$ and $\overline{\varepsilon} = [\overline{\varepsilon}_1, \ldots, \overline{\varepsilon}_n]^\top$, which turns to be useful all along the paper. 
%
%
%
%
%
%

In the present paper, we make a boundness assumption on the reproducing kernel $\mathbb{K}(\cdot, \cdot)$.
\begin{assumption} \label{a2}
Let us assume that the measurable reproducing kernel $\mathbb{K}(\cdot, \cdot)$ is uniformly bounded on its support, meaning that there exists a constant $B > 0$ such that $$\underset{x \in \mathcal{X}}{\sup} \Big[ \mathbb{K}(x, x) \Big]  = \underset{x \in \mathcal{X}}{\sup} || \mathbb{K}(\cdot, x) ||_{\mathcal{H}}^2 \leq B .$$
Moreover in what follows, we assume that $B=1$ without loss of generality.
\end{assumption}

Assumption \ref{a2} holds for many kernels. On the one hand, it is fulfilled with an unbounded domain $\mathcal{X}$ with a bounded kernel (e.g., Gaussian, Laplace kernels). On the other hand, it amounts to assume the domain $\mathcal{X}$ is bounded with an unbounded kernel such as the polynomial or Sobolev kernels \cite{scholkopf2001learning}. 
Let us also mention that Assumptions~\ref{a1} and~\ref{a2} (combined with the reproducing property) imply that $f^*$ is uniformly bounded since 
\begin{equation} \label{h_norm_infty_norm}
    \lVert f^* \rVert_{\infty} = \underset{x \in \mathcal{X}}{\sup} \left| \langle f^*, \mathbb{K}(\cdot, x) \rangle_{\mathcal{H}} \right| \leq \lVert f^* \rVert_{\mathcal{H}} \underset{x \in \mathcal{X}}{\sup}\lVert \mathbb{K}(\cdot, x) \rVert_{\mathcal{H}} \leq R.    
\end{equation}

\medskip

Considering now the Gram matrix $K = \{\mathbb{K}(x_i, x_j)\}_{1\leq i,j\leq n}$, the related \emph{normalized Gram matrix} $K_n = \{ \mathbb{K}(x_i, x_j) / n\}_{1\leq i,j \leq n}$ turns out to be symmetric and positive semidefinite. This entails the existence of 
the empirical eigenvalues $\widehat{\mu}_1, \ldots, \widehat{\mu}_n$ (respectively, the eigenvectors $\widehat{u}_1, \ldots, \widehat{u}_n$) such that $K_n \widehat{u}_i = \widehat{\mu}_i \cdot \widehat{u}_i $ for all $i \in [n]$.
Remark that Assumption~\ref{a2} implies $0\leq \max( \widehat{\mu}_1,\mu_1) \leq 1$.

For technical convenience, it turns out to be useful rephrasing the model \eqref{vector.model} by using the SVD of the normalized Gram matrix $K_n$. This leads to the new (rotated) model
\begin{equation} \label{rotated_model}
    Z_i = \langle \widehat{u}_i, Y \rangle = G_i^* + \varepsilon_i, \quad  i = 1, \ldots, n,
\end{equation}
where $G_i^* = \langle \widehat{u}_i, F^* \rangle $, and $\varepsilon_i = \langle \widehat{u}_i, \overline{\varepsilon} \rangle$ is a zero-mean Gaussian random variable with the variance $\sigma^2$.

\subsection{Spectral filter algorithms}
\label{sec.spectral.filters}
Spectral filter algorithms were first introduced for solving ill-posed inverse problems with deterministic noise \cite{engl1996regularization}. Among others, one typical example of such an algorithm is the gradient descent algorithm (that is named as well as $L^2$-boosting \cite{buhlmann2003boosting}). They were more recently brought to the supervised learning community, for instance, by \cite{bauer2007regularization, article, gerfo2008spectral, Yao2007}. 
For estimating the vector $F^*$ from Eq. \eqref{vector.model} in the fixed-design context, such a spectral filter estimator is a linear estimator, which can be expressed as 
\begin{align}\label{spectral.estimator.vector}
    F^\lambda \coloneqq \left(f^\lambda(x_1), \ldots, f^\lambda(x_n)\right)^\top = K_n g_\lambda(K_n) Y,
\end{align}
where $g_{\lambda}:\ [0,1]\to \mathbb{R}$ is called the \emph{admissible spectral filter function} \cite{bauer2007regularization, gerfo2008spectral}. For example, the choice $g_{\lambda}(\xi) = \frac{1}{\xi + \lambda}$, corresponds to the kernel ridge estimator with regularization parameter $\lambda>0$ (see \cite{blanchard2018optimal,celisse2021analyzing} for other possible choices)
%
%
    
    
    
    %
%
%
From the model expressed in the empirical eigenvectors basis \eqref{rotated_model}, the resulting spectral filter estimator \eqref{spectral.estimator.vector} can be expressed as 
\begin{equation} \label{iterations}
    G^{\lambda (t)}_i = \langle \widehat{u}_i, F^{\lambda (t)} \rangle =  \gamma_i^{(t)} Z_i, \quad\forall i=1,\ldots,n,
\end{equation}
where $t \mapsto \lambda(t) > 0$ is a decreasing function mapping $t$ to a regularization parameter value at time $t$, and $t \mapsto \gamma_i^{(t)}$ is defined by
\begin{equation*}
    \gamma_i^{(t)} = \widehat{\mu}_i g_{\lambda(t)}(\widehat{\mu}_i), \quad \forall i = 1, \ldots, n.
\end{equation*}
Under the assumption that $\underset{t \to 0}{\lim}g_{\lambda(t)}(\mu) = 0, \ \mu \in (0, 1]$, it can be proved that $\gamma_i^{(t)}$ is a non-decreasing function of $t$, $\gamma_i^{(0)} = 0$, and $\underset{t \to \infty}{\lim} \gamma_i^{(t)} = 1$. Moreover, $\widehat{\mu}_i = 0$ implies $\gamma_i^{(t)} = 0$, as it is the case for the kernels with a finite rank, that is, when $\mathrm{rk}(K_n) \leq r$ almost surely. 

Thanks to the remark above, we define the following convenient notations $f^t \coloneqq f^{\lambda(t)}$ (for functions) and $F^t \coloneqq F^{\lambda(t)}$ (for vectors), with a continuous time $t \geq 0$, by
\begin{equation} \label{functional_iterations}
    f^t =  g_{\lambda(t)}(S_n^* S_n)S_n^* Y,
\end{equation}
where $S_n: \mathcal{H} \to \mathbb{R}^n$ is the sampling operator and $S_n^*$ is its adjoint, i.e. $( S_nf )_i = f(x_i)$ and $K_n = S_n S_n^*$.

In what follows, we introduce an assumption on a $\gamma_i^{(t)}$ function that will play a crucial role in our analysis. 
\begin{assumption} \label{additional_assumption_gd_krr}
    \begin{equation*}
        c \min \{1, \eta t \widehat{\mu}_i \} \leq \gamma_i^{(t)} \leq \min \{1, \eta t \widehat{\mu}_i \}, \quad i = 1, \ldots, n     
    \end{equation*}
    for some positive constants $c \in (0, 1)$ and $\eta > 0$.
\end{assumption}
Let us mention two famous examples of spectral filter estimators that satisfy Assumption \ref{additional_assumption_gd_krr} with $c=1/2$ (see Lemma \ref{gamma_bounds} in Appendix). These examples will be further studied in the present paper.
\begin{itemize}
  \item Gradient descent (GD) with a constant step-size $0<\eta<1/\widehat{\mu}_1$ and $\eta t \to +\infty$ as $t \to +\infty$: 
  \begin{equation}
    \gamma_i^{(t)} = 1 - (1 - \eta \widehat{\mu}_i)^t, \quad \forall t \geq 0, \ \forall i=1,\ldots,n.    
  \end{equation}
  %
%
The constant step-size $\eta$ can be replaced by any non-increasing sequence $\{\eta (t)\}_{t=0}^{+\infty}$ satisfying \cite{raskutti2014early}
\begin{itemize}
    \item $(\widehat{\mu}_1)^{-1} \geq \eta(t) \geq \eta(t+1) \geq \dots$, for $t = 0, 1, \ldots$,
    
    \item $\sum_{s = 0}^{t - 1} \eta(s) \to +\infty$ as $t \to + \infty$.
\end{itemize}
  \item Kernel ridge regression (KRR) with the regularization parameter $\lambda(t) = 1/(\eta t)$ with $\eta>0$: 
  \begin{equation}
      \gamma_i^{(t)} = \frac{\widehat{\mu}_i}{\widehat{\mu}_i + \lambda(t)}, \quad \forall t > 0, \ \forall i=1,\ldots,n.    
  \end{equation}
  The linear parameterization $\lambda(t) = 1/(\eta t)$ is chosen for theoretical convenience.   

\end{itemize} 
%
%
The examples of $\gamma_i^{(t)}$ above were derived for $F^0 = [f^0(x_1), \ldots, f^0(x_n)]^\top = [0, \ldots, 0]^\top$ as an initialization condition without loss of generality.

\subsection{Key quantities}

From a set of parameters (stopping times) $\mathcal{T} \coloneqq \{t \geq 0 \}$ for an iterative learning algorithm, the present goal is to design $\widehat{t} = \widehat{t}(\{x_i, y_i \}_{i=1}^n)$ from the data $\{ x_i , y_i \}_{i=1}^n$ such that the functional estimator $f^{\widehat{t}}$ is as close as possible to the optimal one among $\mathcal{T}.$ 

Numerous classical model selection procedures for choosing $\widehat{t}$ already exist, e.g. the (generalized) cross validation \cite{wahba1977practical}, AIC and BIC criteria \cite{akaike1998information, schwarz1978estimating}, the unbiased risk estimation \cite{cavalier2002oracle}, or Lepski's balancing principle \cite{mathe2003geometry}.
Their main drawback in the present context is that they require the practitioner to calculate all the estimators $\{f^t, \ t \in \mathcal{T}\}$ in the first step, and then choose the optimal estimator among the candidates in a second step, which can be computationally demanding. 

By contrast, early stopping is a less time-consuming approach. It is based on observing one estimator at each $t \in \mathcal{T}$ and deciding to stop the learning process according to some criterion. Its aim is to reduce the computational cost induced by this selection procedure while preserving the statistical optimality properties of the output estimator.

\medskip

The prediction error (risk) of an estimator $f^t$ at time $t$ is split into a bias and a variance term \cite{raskutti2014early} as
\begin{equation*}
R(t) = \mathbb{E}_{\varepsilon} \lVert f^t - f^* \rVert_n^2  = \lVert \mathbb{E}_{\varepsilon} f^t - f^* \rVert_n^2 + \mathbb{E}_{\varepsilon} \lVert f^t - \mathbb{E}_{\varepsilon}f^t\rVert_n^2 = B^2(t) + V(t) \end{equation*}
with
\begin{equation}
B^2(t)  = \frac{1}{n}\sum_{i=1}^n  (1 - \gamma_i^{(t)})^2 (G_i^*)^2, \ \ \ \ \ \ V(t) = \frac{\sigma^2}{n}\sum_{i=1}^n (\gamma_i^{(t)})^2.
\end{equation}
The bias term is a non-increasing function of $t$ converging to zero, while the variance term is a non-decreasing function of $t$. 
Assume further that $\textnormal{rk}(T_{\mathbb{K}}) \leq r$, which implies that $\textnormal{rk}(K_n) \leq r$ almost surely, then the empirical risk $R_t$ is introduced with the notation of Eq.~\eqref{rotated_model}.
\begin{equation} \label{empirical_risk}
    R_t 
    = \frac{1}{n}\sum_{i=1}^n (1 - \gamma_i^{(t)})^2 Z_i^2 = \frac{1}{n}\sum_{i=1}^r (1 - \gamma_i^{(t)})^2 Z_i^2 + \frac{1}{n}\sum_{i=r+1}^n Z_i^2,
\end{equation}
%
%
An illustration of the typical behavior of the risk, empirical risk, bias, and variance is displayed by Figure~\ref{fig:bvr}. 
\begin{figure}[!htb]
\centering
  \includegraphics[scale=0.49]{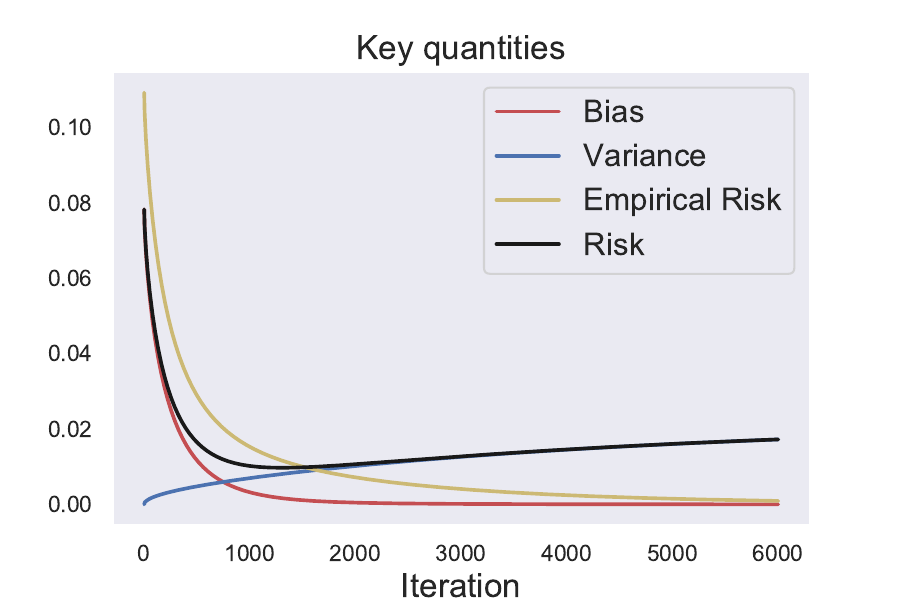}
  \caption{Bias, variance, risk, and empirical risk behavior.}
  \label{fig:bvr}
\end{figure}
Our main concern is formulating a data-driven stopping rule (a mapping from the data $\{(x_i, y_i)\}_{i=1}^n$ to a positive time $\widehat{t}$) so that the prediction errors $\mathbb{E}_{\varepsilon}\lVert f^{\widehat{t}} - f^* \rVert_n^2$ or, equivalently, $\mathbb{E}\lVert f^{\widehat{t}} - f^* \rVert_2^2$ are as small as possible. 
%
%
%
%
%
\vspace{0.1cm}

The analysis of the forthcoming early stopping rules involves the use of a model complexity measure known as the \emph{localized empirical Rademacher complexity} \cite{bartlett2005local, koltchinskii2006local, wainwright2019high} that we generalize to its $\alpha-$smoothed version, for $\alpha \in [0, 1]$.
\begin{definition}
    %
    For any $\epsilon > 0$, $\alpha \in [0, 1]$, consider the localized smoothed empirical Rademacher complexity of $\mathcal{H}$ as 
    \begin{equation} \label{empirical_rademacher_complexity_def}
    \widehat{\mathcal{R}}_{n,\alpha}(\epsilon, \mathcal{H}) = R \left[ \frac{1}{n}\sum_{j=1}^r \widehat{\mu}_j^{\alpha} \textnormal{min}\{ \epsilon^2, \widehat{\mu}_j \} \right]^{1/2}.
    \end{equation}
\end{definition}
%
%
%
It corresponds to a rescaled sum of the empirical eigenvalues truncated at $\epsilon^2$ and smoothed by $\{\widehat{\mu}_i^{\alpha} \}_{i=1}^r$. 

For a given RKHS $\mathcal{H}$ and noise level $\sigma$, let us finally define the \textit{empirical smoothed critical radius} $\widehat{\epsilon}_{n, \alpha}$ as the smallest positive value $\epsilon$ such that
\begin{equation} \label{RK_critical_radius_empirical}
    \frac{\widehat{\mathcal{R}}_{n, \alpha}(\epsilon, \mathcal{H})}{\epsilon R} \leq \frac{2R\epsilon^{1+\alpha}}{\sigma}.
\end{equation}
There is an extensive literature on the empirical critical equation and related empirical critical radius \cite{bartlett2005local, mendelson2002geometric, raskutti2014early}, and it is out of the scope of the present paper providing an exhaustive review on this topic.
Nevertheless, Appendix~\ref{auxiliary} establishes that the smoothed critical radius $\widehat{\epsilon}_{n, \alpha}$ does exist, is unique and achieves the equality in Ineq. (\ref{RK_critical_radius_empirical}).
Constant $2$ in Ineq. \eqref{RK_critical_radius_empirical} is for theoretical convenience only. If $\alpha = 0$, $\widehat{\mathcal{R}}_{n, \alpha}(\epsilon, \mathcal{H}) \equiv \widehat{\mathcal{R}}_n(\epsilon, \mathcal{H})$, and $\widehat{\epsilon}_{n, \alpha} \equiv \widehat{\epsilon}_n$.
\section{Data-driven early stopping rule and minimum discrepancy principle} \label{sec:3}
Let us start by recalling that the expression of the empirical risk in Eq.~\eqref{empirical_risk} gives that the empirical risk is a non-increasing function of $t$ (as illustrated by Fig.~\ref{fig:bvr} as well). This is consistent with the intuition that the amount of available information within the residuals decreases as $t$ grows. If there exists time $t$ such that $f^{t} \approx f^*$, then the empirical risk is approximately equal to $\sigma^2$ (level of noise), that is,
\begin{equation} \label{approximation}
  \mathbb{E}_{\varepsilon}R_t = \mathbb{E}_{\varepsilon}\Big[ \lVert F^t - Y \rVert_n^2 \Big] \approx \mathbb{E}_{\varepsilon} \Big[ \lVert F^* - Y \rVert_n^2 \Big] = \mathbb{E}_{\varepsilon} \Big[\lVert \varepsilon \rVert_n^2 \Big] = \sigma^2. 
\end{equation}
By introducing the reduced empirical risk $\widetilde{R}_t, \ t \geq 0,$ and recalling that $\textnormal{rk}(K_n) \leq r$, 
\begin{equation} \label{reduced_empir_risk_def}
\mathbb{E}_{\varepsilon}R_t = \mathbb{E}_{\varepsilon} \left[ \frac{1}{n}\sum_{i=1}^n (1 - \gamma_i^{(t)})^2 Z_i^2 \right] = \mathbb{E}_{\varepsilon} \underbrace{\left[ \frac{1}{n}\sum_{i=1}^r (1 - \gamma_i^{(t)})^2 Z_i^2 \right]}_{\coloneqq \widetilde{R}_t} + \frac{n - r}{n}\sigma^2 \overset{(\textnormal{i})}{\approx} \sigma^2,
\end{equation}
where $(\textnormal{i})$ is due to Eq. (\ref{approximation}). This heuristic argument gives rise to a first deterministic stopping rule $t^*$ involving the reduced empirical risk and given by
\begin{equation} \label{t_star}
    t^* = \inf \left\{ t > 0 \ | \ \mathbb{E}_{\varepsilon} \widetilde{R}_t \leq \frac{r \sigma^2}{n} \right\}.
\end{equation}
Since $t^*$ is \textit{not achievable} in practice, 
an estimator of $t^*$ is given by the data-driven stopping rule $\tau$ based on the so-called minimum discrepancy principle
\begin{equation}\label{tau}
    \tau = \inf \left\{ t > 0 \ | \ \widetilde{R}_t \leq \frac{r \sigma^2}{n}\right\}.
\end{equation}

The existing literature considering the MDP-based stopping rule usually defines $\tau$ by the event $\{ R_t \leq \sigma^2 \}$ \cite{blanchard2018optimal, blanchard2016convergence, blanchard2012discrepancy, engl1996regularization, hansen2010discrete, stankewitz2019smoothed}. 
Notice that with a full-rank kernel matrix, the reduced empirical risk $\widetilde{R}_t$ is equal to the classical empirical risk $R_t$, leading then to the same stopping rule.
%
From a practical perspective, the knowledge of the rank of the Gram matrix avoids estimating the last $n-r$ components of the vector $G^*$, which are already known to be zero (see \cite[Section 4.1]{raskutti2014early} for more details).

%
\subsection{Finite-rank kernels}
\subsubsection{Fixed-design framework}
Let us start by discussing our results with the case of RKHS of finite-rank kernels with rank $r < n: \mu_i = 0, \ i > r$, and $\widehat{\mu}_i = 0, \ i > r$. Examples that include these kernels are the linear kernel $\mathbb{K}(x_1, x_2) = x_1^\top x_2$ and the polynomial kernel of degree $d \in \mathbb{N} $ $\mathbb{K}(x_1, x_2) = (1 + x_1^\top x_2)^d$. 

The following theorem applies to any functional estimator $\{ f^t \}_{t \in [0, T]}$ generated by (\ref{functional_iterations}) and initialized at $f^0 = 0$. The main part of the proof of this result consists of properly upper bounding $\mathbb{E}_{\varepsilon}|\mathbb{E}_{\varepsilon}\widetilde{R}_{t^*} - \widetilde{R}_{t^*} |$ and follows the same trend of Proposition 3.1 in \cite{blanchard2018optimal}.
\begin{theorem}\label{th:1}
    Under Assumptions \ref{a1} and \ref{a2}, given the stopping rule (\ref{tau}),
    \begin{equation} \label{general_res}
        \mathbb{E}_{\varepsilon} \lVert f^{\tau} - f^* \rVert_n^2 \leq 2(1+\theta^{-1}) \mathbb{E}_{\varepsilon} \lVert f^{t^*} - f^* \rVert_n^2 + 2(\sqrt{3} + \theta) \frac{\sqrt{r} \sigma^2}{n}
    \end{equation}
    for any positive $\theta$.
\end{theorem}
\begin{proof}[Proof of Theorem \ref{th:1}]
    In this proof, we will use the following inequalities: for any $a, b \geq 0, \ (a - b)^2 \leq |a^2 - b^2|$, and $2ab \leq \theta a^2 + \frac{1}{\theta} b^2$ for $\forall  \theta > 0$.

    Let us first prove the subsequent oracle-type inequality for the difference between $f^\tau$ and $f^{t^*}$. Consider
\begin{align*}
    \lVert f^{t^*} - f^{\tau} \rVert_n^2 & = \frac{1}{n} \sum_{i=1}^r \Big( \gamma_i^{(t^*)} - \gamma_i^{(\tau)}\Big)^2 Z_i^2 \leq \frac{1}{n} \sum_{i=1}^r |(1 - \gamma_i^{(t^*)})^2 - (1 - \gamma_i^{(\tau)})^2 | Z_i^2 \\ 
    & =  (\widetilde{R}_{t^*} - \widetilde{R}_{\tau})\mathbb{I}\left\{ \tau \geq t^*\right\} + (\widetilde{R}_{\tau} - \widetilde{R}_{t^*})\mathbb{I}\left\{ \tau < t^*\right\} \\
    %
    %
    & \leq (\widetilde{R}_{t^*} - \mathbb{E}_{\varepsilon}\widetilde{R}_{t^*})\mathbb{I}\left\{ \tau \geq t^*\right\} + (\mathbb{E}_{\varepsilon}\widetilde{R}_{t^*} - \widetilde{R}_{t^*})\mathbb{I}\left\{ \tau < t^* \right\} \\
    & \leq |\widetilde{R}_{t^*} - \mathbb{E}_{\varepsilon}\widetilde{R}_{t^*} |.
\end{align*}
From the definition of $\widetilde{R}_t$ (\ref{reduced_empir_risk_def}), one notices that
\begin{align*}
  | \widetilde{R}_{t^*} - \mathbb{E}_{\varepsilon}\widetilde{R}_{t^*} | & = \left| \sum_{i=1}^r (1 - \gamma_i^{(t^*)})^2 \Big[\frac{1}{n}(\varepsilon_i^2 - \sigma^2) + \frac{2}{n} \varepsilon_i G_i^* \Big]\right|.   
\end{align*}
From $\mathbb{E}_{\varepsilon}| X(\varepsilon)| \leq \sqrt{\text{var}_{\varepsilon}X(\varepsilon)}$ for $X(\varepsilon)$ centered and $\sqrt{a + b} \leq \sqrt{a} + \sqrt{b}$ for any $a, b \geq 0$, and $\mathbb{E}_{\varepsilon}\left( \varepsilon^{4} \right) \leq 3 \sigma^4$, it comes
\begin{align*}
\mathbb{E}_{\varepsilon} |\widetilde{R}_{t^*} - \mathbb{E}_{\varepsilon}\widetilde{R}_{t^*} | & \leq \sqrt{\frac{2 \sigma^2}{n^2} \sum_{i=1}^r (1 - \gamma_i^{(t^*)})^4 \left[ \frac{3}{2} \sigma^2 + 2 (G_i^*)^2 \right]} \\ 
&\leq \sqrt{\frac{3 \sigma^4}{n^2} \sum_{i=1}^r (1 - \gamma_i^{(t^*)})^2} + \sqrt{\frac{4 \sigma^2}{n^2} \sum_{i=1}^r (1 - \gamma_i^{(t^*)})^2 (G_i^*)^2} \\ 
& \leq \frac{\sqrt{3}\sigma^2 \sqrt{r}}{n} + \theta \frac{\sigma^2}{n} + \theta^{-1}B^2(t^*) \\ &\leq \theta^{-1}B^2(t^*) + (\sqrt{3} + \theta) \frac{\sqrt{r} \sigma^2}{n}.
\end{align*}

Applying the inequalities $(a + b)^2 \leq 2a^2 + 2b^2$ for any $a, b \geq 0$  and $B^2(t^*) \leq \mathbb{E}_{\varepsilon} \lVert f^{t^*} - f^* \rVert_n^2$, we arrive at
   \begin{align*}
   & \mathbb{E}_{\varepsilon} \lVert f^{\tau} - f^* \rVert_n^2 \\
   & \leq 2 \mathbb{E}_{\varepsilon} \lVert f^{t^*} - f^* \rVert_n^2 + 2  \mathbb{E}_{\varepsilon} \lVert f^{\tau} - f^{t^*} \rVert_n^2 \\
   %
   %
   & \leq 2(1 + \theta^{-1})\mathbb{E}_{\varepsilon}\lVert f^{t^*} - f^* \rVert_n^2 + 2 (\sqrt{3} + \theta)\frac{\sqrt{r}\sigma^2}{n}.
    \end{align*}
\end{proof}
First of all, it is worth noting that the risk of the estimator $f^{t^*}$ is proved to be \textit{optimal} for gradient descent and kernel ridge regression no matter the kernel we use (see Appendix \ref{finite_rank_appendix} for the proof), so it remains to focus on the remainder term on the right-hand side in Ineq. (\ref{general_res}). Theorem \ref{th:1} applies to any reproducing kernel, but one remarks that for infinite-rank kernels, $r = n$, and we achieve only the rate $\mathcal{O}\left(1/\sqrt{n}\right)$.
This rate is suboptimal since, for instance, RKHS with polynomial eigenvalue decay kernels (will be considered in the next subsection) has the minimax-optimal rate for the risk error of the order $\mathcal{O}\left(n^{-\frac{\beta}{\beta + 1}}\right)$, with $\beta > 1$. Therefore, the oracle-type inequality (\ref{general_res}) could be useful only for finite-rank kernels due to the fast $\mathcal{O}(\sqrt{r}/n)$ rate of the remainder term. 

Notice that, in order to make artificially the term $\mathcal{O}(\sqrt{r} / n)$ a remainder one (even for cases corresponding to infinite-rank kernels), \cite{blanchard2018optimal, blanchard2018early} introduced in the definitions of their stopping rules a restriction on the "starting time" $t_{0}$. However, in the mentioned work, this restriction incurred the price of possibility to miss the designed time $\tau$. 
Besides that, \cite{blanchard2018early} developed an additional procedure based on standard model selection criteria such as AIC-criterion for the spectral cut-off estimator to recover the "missing" stopping rule and achieve optimality over Sobolev-type ellipsoids. In our work, we removed such a strong assumption.

As a corollary of Theorem \ref{th:1}, one can prove that $f^{\tau}$ provides a minimax estimator of $f^*$ over the ball of radius $R$.
\begin{corollary} \label{corollary_empirical_norm}
Under Assumptions \ref{a1}, \ref{a2}, \ref{additional_assumption_gd_krr}, if a kernel has finite rank $r$, then 
\begin{equation}
    \mathbb{E}_{\varepsilon} \lVert f^{\tau} - f^* \rVert_n^2 \leq c_u R^2 \widehat{\epsilon}_n^2,
\end{equation}
where the constant $c_u$ is numeric.
\end{corollary}

\begin{proof}[Proof of Corollary \ref{corollary_empirical_norm}]
From Theorem \ref{th:1} and Lemma \ref{cl1l2} in Appendix,
\begin{equation}
    \mathbb{E}_{\varepsilon} \lVert f^{\tau} - f^* \rVert_n^2 \leq 16 (1 + \theta^{-1})R^2 \widehat{\epsilon}_n^2 + 2(\sqrt{3} + \theta)\frac{\sqrt{r}\sigma^2}{n}.
\end{equation}
Further, applying \cite[Section 4.3]{raskutti2014early}, $\widehat{\epsilon}_n^2 = c \frac{r \sigma^2}{n R^2}$, and it implies that
\begin{equation}
    \mathbb{E}_{\varepsilon} \lVert f^{\tau} - f^* \rVert_n^2 \leq \Big[ 16(1 + \theta^{-1}) + \frac{2(\sqrt{3} + \theta)}{c} \Big] R^2 \widehat{\epsilon}_n^2.
\end{equation}
\end{proof}
Note that the critical radius $\widehat{\epsilon}_n$ cannot be arbitrary small since it should satisfy Ineq. (\ref{RK_critical_radius_empirical}). As it will be clarified later, the squared empirical critical radius is essentially optimal.
\subsubsection{Random-design framework}\label{sec.random.design}
We would like to transfer the minimax optimality bound for the estimator $f^{\tau}$ from the empirical $L_2(\mathbb{P}_n)$-norm to the population $L_2(\mathbb{P}_X)$ norm by means of the so-called localized population Rademacher complexity. This complexity measure became a standard tool in empirical processes and nonparametric regression \cite{bartlett2005local, koltchinskii2006local, raskutti2014early, wainwright2019high}.

For any kernel function class studied in the paper, we consider the localized Rademacher complexity that can be seen as a population counterpart of the empirical Rademacher complexity (\ref{empirical_rademacher_complexity_def}) introduced earlier:
\begin{equation} \label{radamacher_complexity}
    \overline{\mathcal{R}}_n(\epsilon, \mathcal{H}) = R \left[ \frac{1}{n}\sum_{i=1}^{+\infty} \min \{ \mu_i, \epsilon^2 \} \right]^{1/2}.
\end{equation}

Using the localized population Rademacher complexity, we define its \textit{population critical radius} $\epsilon_n > 0$ to be the smallest positive solution $\epsilon$ that satisfies the inequality
\begin{equation} \label{RK_critical_radius}
    \frac{\overline{\mathcal{R}}_n(\epsilon, \mathcal{H})}{\epsilon R} \leq \frac{2 \epsilon R}{\sigma}.
\end{equation}

In  contrast  to the empirical critical radius $\widehat{\epsilon}_n$, this quantity is not data-dependent, since it is specified by the population eigenvalues of the kernel operator $T_{\mathbb{K}}$ underlying the RKHS. 

\begin{theorem}\label{th:2}
    Under Assumptions \ref{a1}, \ref{a2}, and \ref{additional_assumption_gd_krr}, given the stopping time (\ref{tau}), there is a positive numeric constant $\widetilde{c}_{u}$ so that for finite-rank kernels with rank $r$, with probability at least $1 - c \exp(-c_1 n \epsilon_n^2)$,
    \begin{equation} \label{upper_bound_on_l2_norm_finite_kernel} 
          \lVert f^{\tau} - f^* \rVert_2^2 \leq \widetilde{c}_{u}R^2 \epsilon_n^2
    \end{equation}
    In addition, the risk error of $\tau$ is bounded as
    \begin{equation} \label{upper_bound_on_risk_l2_finite_kernel}
    \mathbb{E}\lVert f^{\tau} - f^* \rVert_2^2 \leq \frac{\widetilde{c}r \sigma^2}{n} + \underbrace{C(\sigma, R)\exp(-cr)}_{\textnormal{remainder term}},
    \end{equation}
    where constant $C(\sigma, R)$ depends on $\sigma$ and $R$ only.
\end{theorem}


\begin{remark}
The full proof is deferred to Section \ref{proof_for_change_of_norm}. Regarding Ineq. (\ref{upper_bound_on_l2_norm_finite_kernel}), $\epsilon_n^2$ is proven to be the minimax-optimal rate for the $L_2(\mathbb{P}_X)$ norm in a RKHS (see \cite{bartlett2005local, mendelson2002geometric, raskutti2014early}). As for the risk error in Ineq. (\ref{upper_bound_on_risk_l2_finite_kernel}), the (exponential) remainder term should decrease to zero faster than $\frac{r \sigma^2}{n}$, and Theorem \ref{th:2} provides a rate $\mathcal{O}\left( \frac{r \sigma^2}{n} \right)$ that matches up to a constant the minimax bound (see, e.g., \cite[Theorem 2(a)]{raskutti2012minimax} with $s = 1$), when $f^*$ belongs to the $\mathcal{H}$-norm ball of a fixed radius $R$, thus not improvable in general. A similar bound for finite-rank kernels was achieved in \cite[Corollary 4]{raskutti2014early}.
\end{remark} 
We summarize our findings in the following corollary.
\begin{corollary} \label{finite_rank_corollary}
Under Assumptions \ref{a1}, \ref{a2}, \ref{additional_assumption_gd_krr} and a finite-rank kernel, the early stopping rule $\tau$ satisfies 
\begin{equation}
    \mathbb{E} \lVert f^{\tau} - f^* \rVert_2^2 \asymp \underset{\widehat{f}}{\inf} \underset{\lVert f^* \rVert_{\mathcal{H}} \leq R}{\sup} \mathbb{E} \lVert \widehat{f} - f^* \rVert_2^2,
\end{equation}
where the infimum is taken over all measurable functions of the input data.
\end{corollary}
\subsection{Practical behavior of $\tau$ with infinite-rank kernels} \label{pdk}
A typical example of RKHS that produces an infinite-rank kernel is the $k^{\textnormal{th}}$-order Sobolev spaces for some fixed integer $k \geq 1$ with Lebesgue measure on a bounded domain. We consider Sobolev spaces that consist of functions that have $k^{\textnormal{th}}$-order weak derivatives $f^{(k)}$ being Lebesgue integrable and $f^{(0)}(0) = f^{(1)}(0) = \ldots = f^{(k-1)}(0) = 0$. It is worth mentioning that for such classes, the eigenvalues of the kernel operator $\mu_i \asymp i^{-\beta}, \ i = 1, 2, \ldots$, with $\beta = 2k$. Another example of kernel with this decay condition for the eigenvalues is the Laplace kernel $\mathbb{K}(x_1, x_2) = e^{-|x_1 - x_2|}, \ x_1, x_2 \in \mathbb{R}$ (see \cite[p.402]{scholkopf2001learning}). 

Firstly, let us now illustrate the practical behavior of ESR (\ref{tau}) (its histogram) for gradient descent (\ref{iterations}) with the step-size $\eta = 1/(1.2 \widehat{\mu}_1)$ and one-dimensional Sobolev kernel $\mathbb{K}(x_1, x_2) = \min\{x_1, x_2\}$ that generates the reproducing space
\begin{equation} \label{H:def}
\mathcal{H} = \left\{f: [0, 1] \to \mathbb{R} \ | \ f(0) = 0, \int_{0}^1 (f^\prime(x))^2 dx < \infty \right\}.
\end{equation}
We deal with the model (\ref{main}) with two regression functions: a smooth piece-wise linear $f^*(x) = |x - 1/2|-1/2 $ and nonsmooth heavisine $f^*(x) = 0.093 \ [4 \  \textnormal{sin}(4 \pi x) - \textnormal{sign}(x - 0.3) - \textnormal{sign}(0.72 - x)]$ functions. The design points are random $x_i \overset{\textnormal{i.i.d.}}{\sim} \mathbb{U}[0, 1]$. The number of observations is $n = 200$. For both functions, $\lVert f^* \rVert_n \approx 0.28$, and we set up a  middle difficulty noise level $\sigma = 0.15$. The number of repetitions is $N = 200$.
\begin{figure}[htbp]
    \centering
    \begin{subfigure}{6cm}
        \centering
        \includegraphics[width=\linewidth]{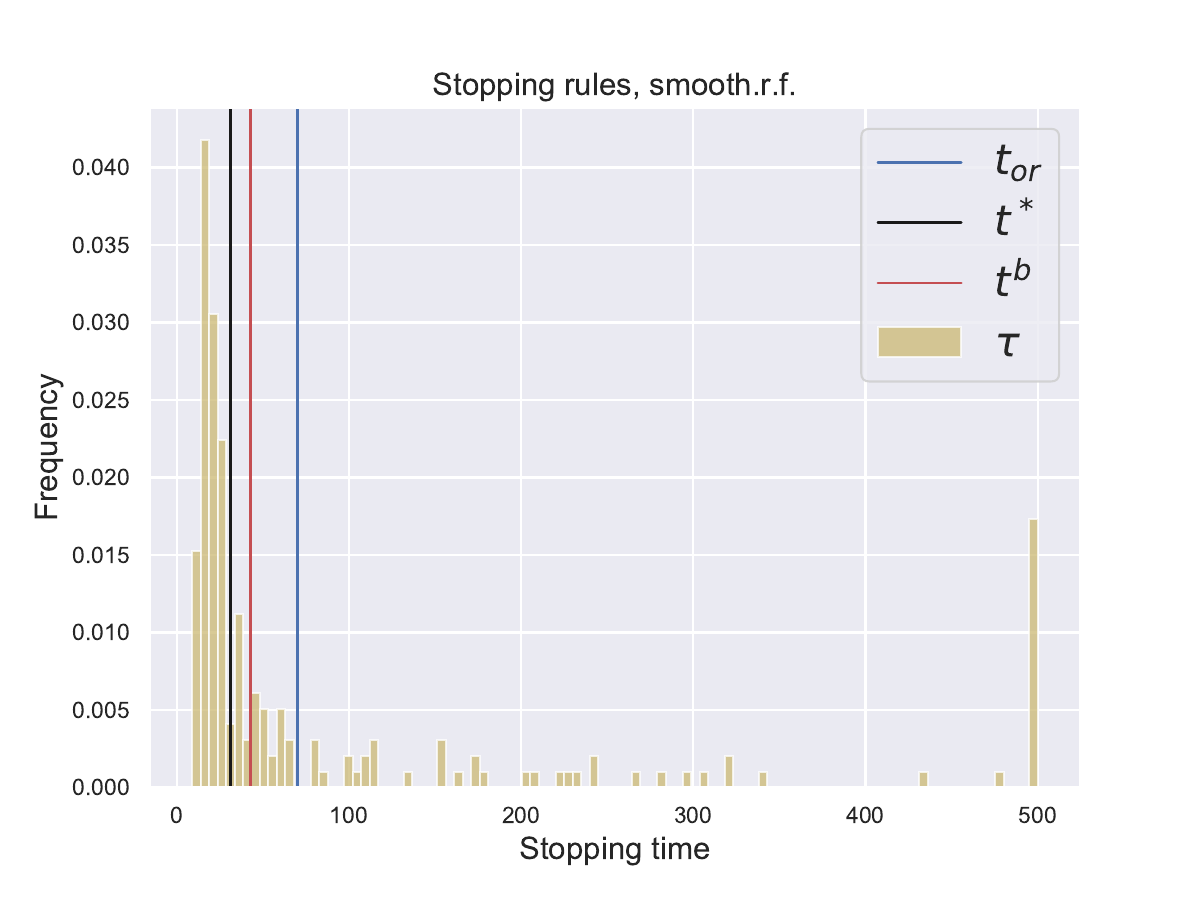} 
        \caption{}
        \label{fig:subfig1}
    \end{subfigure}
    \hfill
    \begin{subfigure}{6cm}
        \centering
        \includegraphics[width=\linewidth]{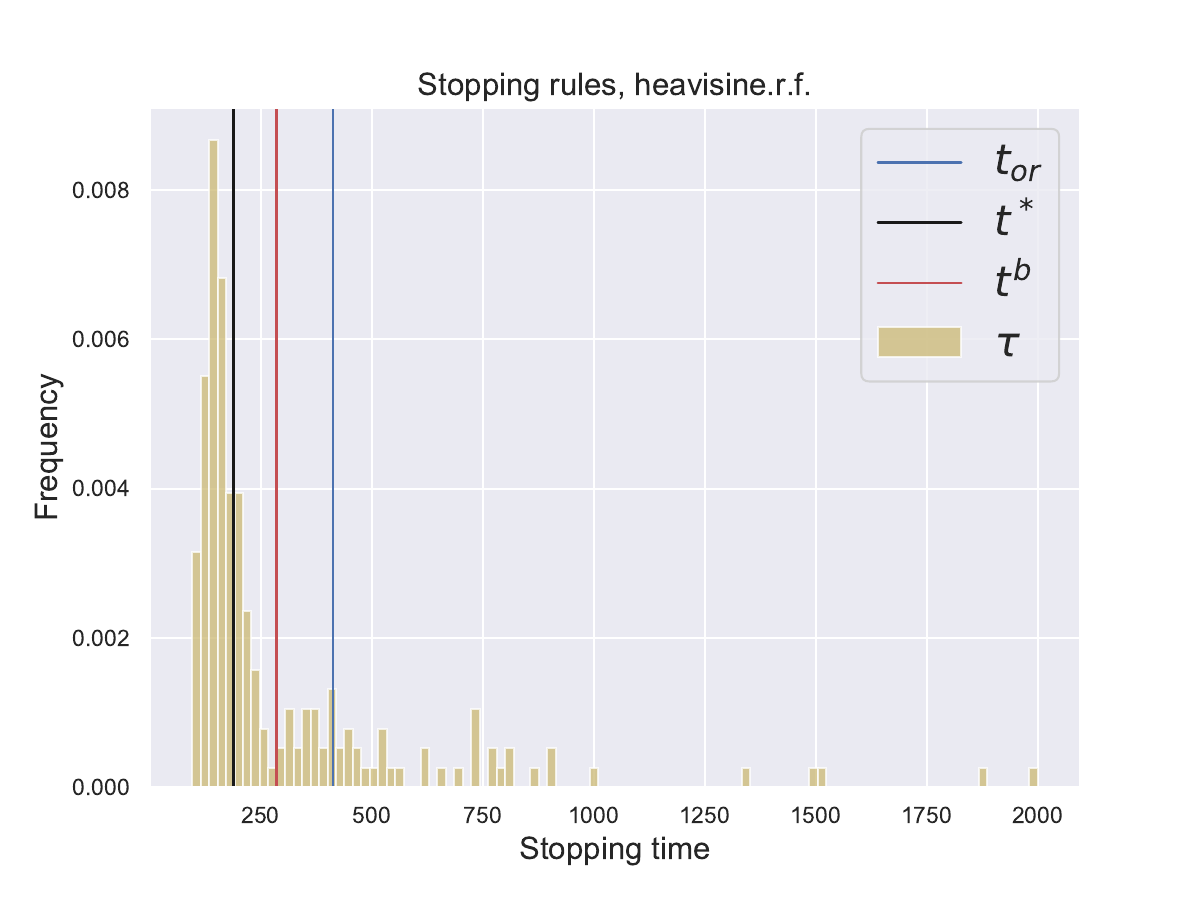} 
        \caption{}
        \label{fig:subfig2}
    \end{subfigure}
    \caption{Histogram of $\tau$ vs $t^*$ vs  $t^b \coloneqq \inf \{ t > 0 \ | \ B^2(t) \leq V(t) \}$ vs $t_{\textnormal{or}} \coloneqq \underset{t > 0}{\textnormal{argmin}} \left[ \mathbb{E}_{\varepsilon}\lVert f^t - f^* \rVert_n^2 \right]$ for kernel gradient descent with the step-size $\eta = 1 / (1.2 \widehat{\mu}_1)$ for the piece-wise linear $f^*(x) = |x - 1/2| - 1/2$ (panel (a)) and heavisine $f^*(x) = 0.093 \ [4 \ \textnormal{sin}(4 \pi x) - \textnormal{sign}(x - 0.3) - \textnormal{sign}(0.72 - x)]$ (panel (b)) regression functions, and the first-order Sobolev kernel $\mathbb{K}(x_1, x_2) = \min \{x_1, x_2 \}$.}
    \label{fig:hist}
\end{figure}
In panel (a) of Figure \ref{fig:hist}, we detect that our stopping rule $\tau$ has a high variance. 
However, if we change the signal $f^*$ from the smooth to nonsmooth one, the regression function does not belong anymore to $\mathcal{H}$ defined in (\ref{H:def}). In this case (panel (b) in Figure \ref{fig:hist}), the stopping rule $\tau$ performs much better than for the previous regression function. 
In order to get a stable early stopping rule that will be close to $t^*$, we propose using a special smoothing technique for the empirical risk. 
%
%
\section{Polynomial smoothing} \label{sec:4}
As was discussed earlier, the main issue of poor behavior of the stopping rule $\tau$ for infinite-rank kernels is the variability of the empirical risk around its expectation. 
A solution that we propose is to smooth the empirical risk by means of the eigenvalues of the normalized Gram matrix.
\subsection{Polynomial smoothing and minimum discrepancy principle rule}
We start by defining the squared $\alpha$-norm as $\lVert f \rVert_{n, \alpha}^2 \coloneqq \langle K_n^{\alpha} F, F \rangle_n$ for all $F = \left[f(x_1),  \ldots, f(x_n) \right]^\top  \in \mathbb{R}^n$ and $\alpha \in [0, 1]$, from which we also introduce the smoothed risk, bias, and variance of a spectral filter estimator as
\begin{equation*}
R_{\alpha}(t) = \mathbb{E}_{\varepsilon}\lVert f^t - f^*\rVert_{n, \alpha}^2  = \lVert \mathbb{E}_{\varepsilon}f^t - f^*\rVert_{n, \alpha}^2 + \mathbb{E}_{\varepsilon}\lVert f^t - \mathbb{E}_{\varepsilon}f^t\rVert_{n, \alpha}^2 = B^2_{\alpha}(t) + V_{\alpha}(t),
\end{equation*} 
with 
\begin{equation}
B^2_{\alpha}(t) = \frac{1}{n}\sum_{i=1}^n \widehat{\mu}_i^{\alpha} (1 - \gamma_i^{(t)})^2 (G_i^*)^2, \ \ \ \  V_{\alpha}(t) = \frac{\sigma^2}{n}\sum_{i=1}^n \widehat{\mu}_i^{\alpha} (\gamma_i^{(t)})^2.
\end{equation}
The smoothed empirical risk is
\begin{equation}
    R_{\alpha, t} = \lVert F^t - Y\rVert_{n, \alpha}^2 = \lVert G^t - Z \rVert_{n, \alpha}^2 = \frac{1}{n}\sum_{i=1}^n \widehat{\mu}_i^{\alpha} (1 - \gamma_i^{(t)})^2 Z_i^2 , \quad \textnormal{ for } t > 0.
\end{equation}
Recall that the kernel is bounded by $B = 1$, thus $\widehat{\mu}_i \leq 1$ for all $i = 1, \ldots, n$, then the smoothed bias $B_{\alpha}^2(t)$ and smoothed variance $V_{\alpha}(t)$ are smaller their non-smoothed counterparts. 

Analogously to the heuristic derivation leading to the stopping rule (\ref{tau}), the new stopping rule is based on the discrepancy principle applied to the $\alpha-$smoothed empirical risk, that is,
\begin{equation} \label{t_alpha}
    \tau_{\alpha} = \inf \left\{ t > 0 \ | \ R_{\alpha, t} \leq \sigma^2\frac{\mathrm{tr}(K_n^\alpha)}{n} \right\},
\end{equation}
where $\sigma^2 \mathrm{tr}(K_n^\alpha)/n = \sigma^2\sum_{i=1}^n \widehat{\mu}_i^{\alpha}/n$ is the natural counterpart of $r \sigma^2/n$ in the case of a full-rank kernel matrix and the $\alpha-$norm.

%

\subsection{Related work}
The idea of smoothing the empirical risk (the residuals) is not new in the literature. For instance, \cite{blanchard2016convergence, blanchard2010conjugate, blanchard2012discrepancy} discussed various smoothing strategies applied to (kernelized) conjugate gradient descent, and \cite{celisse2021analyzing} considered spectral regularization with spectral filter estimators. %
More closely related to the present work, \cite{stankewitz2019smoothed} studied a statistical performance improvement allowed by polynomial smoothing of the residuals (as we do here) but restricted to the spectral cut-off estimator.

In \cite{blanchard2010conjugate, blanchard2012discrepancy}, the authors considered the following statistical inverse problem: $z = Ax + \sigma \zeta$, where $A$ is a self-adjoint operator and $\zeta$ is Gaussian noise. In their case, for the purpose of achieving optimal rates, the usual discrepancy principle rule $\lVert Ax_m - z \rVert \leq \vartheta \delta$ ($m$ is an iteration number, $\vartheta$ is a parameter) was modified and took the form $\lVert \rho_{\lambda}(A)( A x_m - z )\rVert \leq \vartheta \delta$, where $\rho_{\lambda}(t) = \frac{1}{\sqrt{t + \lambda}}$ and $\delta$ is the normalized variance of Gaussian noise.

In \cite{blanchard2016convergence}, the minimum discrepancy principle was modified to the following: each iteration $m$ of conjugate gradient descent was represented by a vector $\widehat{\alpha}_m = K_n^{\dagger} Y$, $K_n^{\dagger}$ is the pseudo-inverse of the normalized Gram matrix, and the learning process was stopped if $\lVert Y - K_n \widehat{\alpha}_m \rVert_{K_n} < \Omega$ for some positive $\Omega$, where $\lVert \alpha \rVert_{K_n}^2 = \langle \alpha, K_n \alpha \rangle.$ Thus, this method corresponds (up to a threshold) to the stopping rule (\ref{t_alpha}) with $\alpha = 1.$ 

In the work \cite{stankewitz2019smoothed}, the authors concentrated on the inverse problem $Y = A \xi + \delta W$ and its corresponding Gaussian vector observation model $Y_i = \Tilde{\mu}_{i} \xi_i + \delta \varepsilon_i, \ i \in [r]$, where $\{ \Tilde{\mu}_i \}_{i=1}^r$ are the singular values of the linear bounded operator $A$ and $\{ \varepsilon_i \}_{i=1}^r$ are Gaussian noise variables. They recovered the signal $\{ \xi_i \}_{i=1}^r$ by a cut-off estimator of the form $\widehat{\xi}_i^{(t)} = \mathbb{I}\{i \leq t \}\widetilde{\mu}_{i}^{-1}Y_i, \ i \in [r]$. The discrepancy principle in this case was $\lVert (A A^\top)^{\alpha/2}(Y - A \widehat{\xi}^{(t)}) \rVert^2 \leq \kappa$ for some positive $\kappa.$ They found out that, if the smoothing parameter $\alpha$ lies in the interval $[\frac{1}{4p}, \frac{1}{2p})$, where $p$ is the polynomial decay of the singular values $\{ \widetilde{\mu}_i \}_{i=1}^r$, then the cut-off estimator is adaptive to Sobolev ellipsoids. Therefore, our work could be considered as an extension of \cite{stankewitz2019smoothed} in order to generalize the polynomial smoothing strategy to more complex filter estimators such as gradient descent and (Tikhonov) ridge regression in the reproducing kernel framework.

\subsection{Optimality result (fixed-design)} \label{optimality_section}
%
%
%
%
%
%
%
%
%
%
We pursue the analogy a bit further by defining the \textit{smoothed statistical dimension} as 
\begin{equation} \label{smoothed_statistical_dimension}
    d_{n, \alpha} \coloneqq \inf \left\{ j \in [n]: \widehat{\mu}_j \leq \widehat{\epsilon}_{n, \alpha}^2 \right\},
\end{equation}
and $d_{n, \alpha} = n$ if no such index exists. 
Combined with (\ref{empirical_rademacher_complexity_def}), this implies that
\begin{equation} \label{useful_inequalities_smooth_}
    \widehat{\mathcal{R}}_{n, \alpha}^2(\widehat{\epsilon}_{n, \alpha}, \mathcal{H}) \geq \frac{\sum_{j=1}^{d_{n, \alpha}} \widehat{\mu}_j^{\alpha}}{n}R^2 \widehat{\epsilon}_{n, \alpha}^2, \ \ \textnormal{ and } \ \ \widehat{\epsilon}_{n, \alpha}^{2(1+\alpha)} \geq \frac{\sigma^2\sum_{j=1}^{d_{n, \alpha}}\widehat{\mu}_j^{\alpha}}{4R^2 n}.
\end{equation}
Let us emphasize that \cite{yang2017randomized} already introduced the so-called \emph{statistical dimension} (corresponds to $d_{n, 0}$ in our notation). It appeared that the statistical dimension provides an upper bound on the minimax-optimal dimension of randomized projections for kernel ridge regression (see \cite[Theorem 2, Corollary 1]{yang2017randomized}). In our case, $d_{n, \alpha}$ can be seen as a ($\alpha$-smooth) version of the statistical dimension.

The purpose of the following result is to give more insight into understanding of Eq. (\ref{smoothed_statistical_dimension}) regarding the minimax risk.
\begin{theorem}[Lower bound from Theorem 1 in \cite{yang2017randomized}]  \label{theorem_yang}
For any regular kernel class, meaning that for any $k = 1, \ldots, n$, $ \widehat{\mu}_{k+1}^{-1} \sum_{i=k + 1}^n \widehat{\mu}_i \lesssim k$, and any estimator $\widetilde{f}$ of $f^* \in \mathbb{B}_{\mathcal{H}}(R)$ satisfying the nonparametric model defined in Eq.~\eqref{main}, we get
\begin{equation*}
    \underset{\lVert f^* \rVert_{\mathcal{H}} \leq R}{\sup}\mathbb{E}_{\varepsilon} \lVert \widetilde{f} - f^* \rVert_n^2 \geq c_l R^2 \widehat{\epsilon}_n^2,
\end{equation*}
for some numeric constant $c_l > 0$.
\end{theorem}
%
Firstly, in \cite{yang2017randomized}, the regularity assumption was formulated as $\sum_{d_{n, 0} + 1}^n \widehat{\mu}_i \lesssim d_{n, 0}\widehat{\epsilon}_n^2$, which directly stems from the assumption in Theorem \ref{theorem_yang}. Let us remark that the same assumption (as in Theorem \ref{theorem_yang}) has been already made by \cite[Assumption 6]{celisse2021analyzing}. Secondly, Theorem~\ref{theorem_yang} applies to any kernel, as long as the condition on the tail of eigenvalues is fulfilled, which is in particular true for the reproducing kernels from Section \ref{pdk}. Thus, the fastest achievable rate by an estimator of $f^*$ is $\widehat{\epsilon}_n^2$. 

A key property for the smoothing to yield optimal results is that the value of $\alpha$ has to be large enough to control the tail sum of the smoothed eigenvalues by the corresponding cumulative sum, which is the purpose of the assumption below.
%
%
\begin{assumption} \label{sufficient_smoothing}
There exists $\Upsilon = [\alpha_0, 1], \ \alpha_0 \geq 0$, such that for all $\alpha \in \Upsilon$ and $k \in \{1, \ldots, n\}$,
\begin{equation}
    \sum_{i=k + 1}^{+\infty} \mu_i^{2 \alpha} \leq \mathcal{M} \sum_{i=1}^{k}\mu_i^{2\alpha}, 
\end{equation}
where $\mathcal{M} \geq 1$ denotes a numeric constant.
\end{assumption}
%
%
%
We enumerate several classical examples for which this assumption holds. 
\begin{example}[$\beta$-polynomial eigenvalue decay kernels]
Let us assume that the kernel operator satisfy that there exist numeric constants $0 < c \leq C$ such that 
\begin{equation} \label{beta-polynomial}
    c i^{-\beta} \leq \mu_i \leq Ci^{-\beta}, \ \ i = 1, 2, \ldots,
\end{equation}
For the polynomial eigenvalue-decay kernels, Assumption~\ref{sufficient_smoothing} holds with
\begin{equation}\label{lower.bound.alpha}
    \mathcal{M} = 2^{2\beta - 1} \left( \frac{C}{c} \right)^{2} \quad \textnormal{and} \quad 1 \geq \alpha \geq \frac{1}{\beta + 1}=\alpha_0.
\end{equation}
\end{example}

\begin{example}[$\gamma$-exponential eigenvalue-decay kernels]
Let us assume that the eigenvalues of the kernel operator satisfy that there exist numeric constants $0 < c \leq C $ and a constant $\gamma>0$ such that
\begin{align*}
    c e^{-i^{\gamma}} \leq \mu_i \leq C e^{-i^{\gamma}}, i = 1, 2, \ldots.
\end{align*} 
Instances of kernels within this class include the Gaussian kernel with respect to the Lebesgue measure on the real line (with $\gamma = 2$) or on a compact domain (with $\gamma = 1$) (up to $\log$ factor in the exponent, see \cite[Example 13.21]{wainwright2019high}).
Then, Assumption~\ref{sufficient_smoothing} holds with
\begin{equation*}
    \mathcal{M} = \Big( \frac{C}{c} \Big)^2 \frac{\int_{0}^{\infty} e^{-y^{\gamma}}dy}{\int_{2^{-1/\gamma}}^{2 /(2\alpha_0)^{1/\gamma}} e^{-y^{\gamma}}dy} \quad  \textnormal{and}\quad \alpha \in [\alpha_0, 1], \quad \mbox{for any}\quad \alpha_0 \in (0, 1).
\end{equation*}
\end{example}

For \emph{any regular kernel class} satisfying the above assumption, the next theorem provides a high probability bound on the performance of $f^{\tau_{\alpha}}$ (measured in terms of the $L_2(\mathbb{P}_n)$-norm), which depends on the smoothed empirical critical radius.
%
\begin{theorem}[Upper bound on empirical norm]\label{th:3}
Under Assumptions~\ref{a1}, \ref{a2}, \ref{additional_assumption_gd_krr}, and \ref{sufficient_smoothing}, for any regular kernel and $\alpha \leq \frac{1}{2}$, the stopping time (\ref{t_alpha}) satisfies
\begin{equation} \label{main_inequality}
    \lVert f^{\tau_{\alpha}} - f^* \rVert_n^2 \leq c_u R^2 \widehat{\epsilon}_{n, \alpha}^2 
\end{equation}
with probability at least $1 - c \exp \Big[ - c_1 \frac{R^2}{\sigma^2}n \widehat{\epsilon}_{n, \alpha}^{2(1+\alpha)} \Big]$ for some positive constants $c_1$ and $c_u$, where $c_1$ depends only on $\mathcal{M}$, $c_u$ and $c$ are numeric. Moreover, 
\begin{equation} \label{in_expectation}
    \mathbb{E}_{\varepsilon} \lVert f^{\tau_{\alpha}} - f^* \rVert_n^2 \leq C R^2 \widehat{\epsilon}_{n, \alpha}^2 + 20 \max \{ \sigma^2, R^2 \} \exp \left[ - c_3 \frac{R^2}{\sigma^2}n \widehat{\epsilon}_{n, \alpha}^{2(1+\alpha)} \right], 
\end{equation}
where the constant $C$ is numeric, constant $c_3$ only depending on $\mathcal{M}$. 
\end{theorem}
The complete proof of Theorem \ref{th:3} is given in Appendix~\ref{polynomial_appendix}.  
The main message is that the final performance of the estimator $f^{\tau_{\alpha}}$ is controlled by the smoothed critical radius $\widehat{\epsilon}_{n, \alpha}^2$. From the existing literature on the empirical critical radius \cite{raskutti2012minimax, raskutti2014early, wainwright2019high, yang2017randomized}, it is already known that the non-smooth version $\widehat{\epsilon}_n^2$ is the typical quantity that leads to minimax rates in the RKHS (see also Theorem~\ref{theorem_yang}). 
%
%
The behavior of $\widehat{\epsilon}_{n,\alpha}^2$ with respect to $n$ is likely to depend on $\alpha$, as emphasized by the notation. Intuitively, this suggests that there could exist a range of values of $\alpha$, for which $\widehat{\epsilon}_{n,\alpha}^2$ is of the same order as (or faster than) $\widehat{\epsilon}_{n}^2$, leading therefore to optimal rates. 

Another striking aspect of Ineq.~\eqref{in_expectation} is related to the additional terms involving the exponential function in Ineq. \eqref{in_expectation}. As far as \eqref{main_inequality} is a statement with "high probability", this term is expected to converge to 0 at a rate depending on $n\widehat{\epsilon}_{n,\alpha}^2$. Therefore, the final convergence rate as well as the fact that this term is (or not) negligible will depend on $\alpha$.

As a consequence of Theorem \ref{theorem_yang}, as far as there exist values of $\alpha$ such that $\widehat{\epsilon}_{n, \alpha}^2 $ is at most as large as $\widehat{\epsilon}_n^2$, the estimator $f^{\tau_\alpha}$ is optimal.
\subsection{Consequences for $\beta$-polynomial eigenvalue-decay kernels}

The leading idea in the present section is identifying values of $\alpha$, for which the bound (\ref{main_inequality}) from Theorem~\ref{th:3} scales as $R^2 \widehat{\epsilon}_n^2$.

Let us recall the definition of a polynomial decay kernel from \eqref{beta-polynomial}:
\begin{equation*} 
    c i^{-\beta} \leq \mu_i \leq C i^{-\beta}, \ i = 1, 2, \ldots, \ \ \textnormal{ for } \beta > 1 \textnormal{ and numeric constants } c, C > 0.
\end{equation*}
One typical example of the reproducing kernel satisfying this condition is the Sobolev kernel on $[0, 1] \times [0, 1]$ given by $\mathbb{K}(x, x^\prime) = \min \{x, x^\prime \}$ with $\beta = 2$ \cite{raskutti2014early}. The corresponding RKHS is the first-order Sobolev class, that is, the class of functions that are almost everywhere differentiable with the derivative in $L_2[0, 1]$. 
\begin{lemma} \label{epsilons_comparaison}
For any $\beta$-polynomial eigenvalue decay kernel, there exist numeric constants $c_1, c_2 > 0$ such that for $\alpha < 1/\beta$, one has 
\begin{align*}
    c_1 \widehat{\epsilon}_n^2 \leq \widehat{\epsilon}_{n, \alpha}^2 \leq c_2 \widehat{\epsilon}_n^2  \asymp \left( \frac{\sigma^2}{2 R^2 n} \right)^{\frac{\beta}{\beta + 1}} .
\end{align*} 
\end{lemma}
The proof of Lemma~\ref{epsilons_comparaison} was deferred to Lemma \ref{critical_radius_empirical_smooth} in Appendix~\ref{general_appendix} and is not reproduced here. 
Therefore, if $\alpha \beta < 1$, then $\widehat{\epsilon}_{n, \alpha}^2 \asymp \widehat{\epsilon}_n^2 \asymp \left( \frac{\sigma^2}{2 R^2 n} \right)^{\frac{\beta}{\beta + 1}}$. 
Let us now recall from \eqref{lower.bound.alpha} that Assumption~\ref{sufficient_smoothing} holds for $\alpha\geq (\beta+1)^{-1}$.
All these arguments lead us to the next result, which establishes the minimax optimality of $\tau_\alpha$ with any kernel satisfying the $\beta$-polynomial eigenvalue-decay assumption, as long as $\alpha \in \left[\frac{1}{\beta + 1}, \min \left\{ \frac{1}{\beta}, \frac{1}{2} \right\} \right)$.
\begin{corollary} \label{pdk_corollary}
Under Assumptions~\ref{a1},~\ref{a2}, \ref{additional_assumption_gd_krr}, and the $\beta$-polynomial eigenvalue decay (\ref{beta-polynomial}), for any $\alpha \in \left[\frac{1}{\beta + 1},\min \left\{\frac{1}{\beta}, \frac{1}{2} \right\}\right)$, the early stopping rule $\tau_{\alpha}$ satisfies 
\begin{equation}
    \mathbb{E}_{\varepsilon}\lVert f^{\tau_{\alpha}} - f^* \rVert_n^2 \asymp \underset{\widehat{f}}{\inf} \underset{\lVert f^* \rVert_{\mathcal{H}} \leq R}{\sup} \mathbb{E}_{\varepsilon} \lVert \widehat{f} - f^* \rVert_n^2,
\end{equation}
where the infimum is taken over all measurable functions of the input data.
\end{corollary}
Corollary~\ref{pdk_corollary} establishes an optimality result in the fixed-design framework since as long as $(\beta+1)^{-1} \leq \alpha < \min \left\{\beta^{-1}, \frac{1}{2}\right\}$, the upper bound matches the lower bound up to multiplicative constants. Moreover, this property holds uniformly with respect to $\beta>1$, provided the value of $\alpha$ is chosen appropriately.
An interesting feature of this bound is that the optimal value of $\alpha$ only depends on the (polynomial) decay rate of the empirical eigenvalues of the normalized Gram matrix. This suggests that any effective estimator of the unknown parameter $\beta$ could be plugged into the above (fixed-design) result and would lead to an optimal rate. 
%
Note that \cite{stankewitz2019smoothed} has emphasized a similar trade-off for the smoothing parameter $\alpha$ (polynomial smoothing), considering the spectral cut-off estimator in the Gaussian sequence model.
Regarding convergence rates, Corollary~\ref{pdk_corollary} combined with Lemma~\ref{epsilons_comparaison} suggests that the convergence rate of the expected risk is of the order $\mathcal{O}\left(n^{-\frac{\beta}{\beta+1}}\right)$. This is the same as the already known one in nonparametric regression in the random design framework \cite{raskutti2014early, stone1985additive}, which is known to be minimax-optimal as long as $f^*$ belongs to the RKHS $\mathcal{H}$.

\section{Empirical comparison with existing stopping rules} \label{sec:5}
The present section aims at illustrating the practical behavior of several stopping rules discussed along the paper as well as making a comparison with existing alternative stopping rules.

\subsection{Stopping rules involved}

The empirical comparison is carried out between the stopping rules $\tau$ \eqref{tau} and $\tau_{\alpha}$ with $\alpha \in \left[\frac{1}{\beta + 1}, \min \left\{ \frac{1}{\beta}, \frac{1}{2} \right\}\right)$ \eqref{t_alpha}, and four alternative stopping rules that are briefly described in the what follows. For the sake of comparison, most of them correspond to early stopping rules already considered in \cite{raskutti2014early}.

\subsubsection*{Hold-out stopping rule}
We consider a procedure based on the hold-out idea \cite{arlot2010survey}. Data $\{(x_i, y_i)\}_{i=1}^n$ are split into two parts: the training sample $S_{\textnormal{train}} = (x_{\textnormal{train}}, y_{\textnormal{train}})$ and the test sample $S_{\textnormal{test}} = (x_{\textnormal{test}}, y_{\textnormal{test}})$ so that the training sample and test sample represent a half of the whole dataset. 
We train the learning algorithm for $t = 0, 1, \ldots$ and estimate the risk for each $t$ by $R_{\textnormal{ho}}(f^t) = \frac{1}{n}\sum_{i \in S_{\textnormal{test}}}((\widehat{y}_{\textnormal{test}})_i - y_i)^2$, where $(\widehat{y}_{\textnormal{test}})_i$ denotes the output of the algorithm trained at iteration $t$ on $S_{\textnormal{train}}$ and evaluated at the point $x_i$ of the test sample. 
The final stopping rule is defined as
\begin{equation} \label{t_ho}
    \widehat{\textnormal{T}}_{\textnormal{HO}} = \textnormal{argmin} \Big\{ t \in \mathbb{N} \ | \ R_{\textnormal{ho}} (f^{t + 1}) > R_{\textnormal{ho}} (f^t) \Big\} - 1.
\end{equation}
Although it does not completely use the data for training (loss of information), the hold-out strategy has been proved to output minimax-optimal estimators in various contexts (see, for instance, \cite{article, caponnetto2010cross} with Sobolev spaces and $\beta \leq 2$).

\subsubsection*{V-fold stopping rule}
The observations $\{(x_i, y_i)\}_{i=1}^n$ are randomly split into $V = 4$ equal sized blocks. 
At each round (among the $V$ ones), $V - 1$ blocks are devoted to training $S_{\textnormal{train}} = (x_{\textnormal{train}}, y_{\textnormal{train}})$, and the remaining one serves for the test sample $S_{\textnormal{test}} = (x_{\textnormal{test}}, y_{\textnormal{test}})$.
At each iteration $t = 1, \ldots$, the risk is estimated by $R_{\textnormal{VFCV}}(f^t) = \frac{1}{V - 1} \sum_{j=1}^{V-1} \frac{1}{n/V}\sum_{i \in S_{\textnormal{test}}(j)} ((\widehat{y}_{\textnormal{test}})_i - y_i)^2$, where $\widehat{y}_{\textnormal{test}}$ was described for the hold-out stopping rule. 
The final stopping time is 
\begin{equation} \label{t_vf}
    \widehat{\textnormal{T}}_{\textnormal{VFCV}} = \textnormal{argmin} \big\{ t \in \mathbb{N} \ | \  R_{\textnormal{VFCV}}(f^{t+1}) > R_{\textnormal{VFCV}}(f^t) \big\} - 1.
\end{equation}
V-fold cross validation is widely used in practice since, on the one hand, it is more computationally tractable than other splitting-based methods such as leave-one-out or leave-p-out (see the survey \cite{arlot2010survey}), and on the other hand, it enjoys a better statistical performance than the hold-out (lower variability). 
\subsubsection*{Raskutti-Wainwright-Yu stopping rule (from \cite{raskutti2014early})}
The use of this stopping rule heavily relies on the assumption that $\lVert f^* \rVert_{\mathcal{H}}^2 $ is known, which is a strong requirement in practice.
It controls the bias-variance trade-off by using upper bounds on the bias and variance terms. The latter involves the localized empirical Rademacher complexity $\widehat{\mathcal{R}}_{n}\left(\frac{1}{\sqrt{\eta t}}, \mathcal{H}\right)$. It stops as soon as (upper bound of) the bias term becomes smaller than (upper bound on) the variance term, which leads to
\begin{equation} \label{t_w}
    \widehat{\textnormal{T}}_{\textnormal{RWY}} = \textnormal{argmin} \Big\{ t \in \mathbb{N} \ | \  \widehat{\mathcal{R}}_n\Big(\frac{1}{\sqrt{\eta t}}, \mathcal{H}\Big) > (2 e \sigma \eta t)^{-1}  \Big\} - 1.
\end{equation}
\subsubsection*{Theoretical minimum discrepancy-based stopping rule $t^*$}
The fourth stopping rule is the one introduced in \eqref{t_star}. It relies on the minimum discrepancy principle and involves the (theoretical) expected empirical risk $\mathbb{E}_{\varepsilon} R_t$:
\begin{equation*}
    t^* = \inf \left\{ t \in \mathbb{N} \ | \ \mathbb{E}_{\varepsilon} R_t \leq \sigma^2 \right\}.
\end{equation*}
This stopping time is introduced for comparison purposes only since it cannot be computed in practice. 
This rule is proved to be optimal (see Appendix~\ref{finite_rank_appendix}) for \textit{any bounded reproducing kernel}, so it could serve as a reference in the present empirical comparison.
\subsubsection*{Oracle stopping rule}
The "oracle" stopping rule defines the first time the risk curve starts to increase. 
\begin{equation} \label{t_or}
    t_{\textnormal{or}} = \textnormal{argmin} \big\{ t \in \mathbb{N} \ | \ \mathbb{E}_{\varepsilon}\lVert f^{t+1} - f^* \rVert_n^2 > \mathbb{E}_{\varepsilon}\lVert f^t - f^* \rVert_n^2 \big\} - 1. 
\end{equation}
In situations where only one global minimum does exists for the risk, this rule coincides with the global minimum location. 
Its formulation reflects the realistic constraint that we do not have access to the whole risk curve (unlike in the classical model selection setup).

\subsection{Simulation design}
Artificial data are generated according to the regression model $y_j = f^*(x_j) + \varepsilon_j$, $j = 1,\ldots, n$, where $\varepsilon_j \overset{\textnormal{i.i.d.}}{\sim} \mathcal{N}(0, \sigma^2)$ with the equidistant $x_j = j/n, \ j = 1, \ldots, n$, and  $\sigma = 0.15$. 
The same experiments have been also carried out with uniform $x_i \sim \mathbb{U}[0, 1]$ (not reported here) without any change regarding the conclusions.
The sample size $n$ varies from $40$ to $400$. 

%
The gradient descent algorithm \eqref{iterations} has been used with the step-size $\eta = (1.2\, \widehat{\mu}_1)^{-1}$ and initialization $F^0 = [0, \ldots, 0]^\top$.

The present comparison involves two regression functions with the same $L_2(\mathbb{P}_n)$-norms of the signal $\lVert f^* \rVert_n \approx 0.28$: $(i)$ a piecewise linear function called "smooth" $f^*(x) = |x - 1/2|-1/2$, and $(ii)$ a "sinus" $f^*(x) = 0.4 \ \textnormal{sin}(4 \pi x)$. 

%
To ease the comparison, the piecewise linear regression function was set up as in \cite[Figure 3]{raskutti2014early}.

The case of finite-rank kernels is addressed in Section~\ref{sec.finite.rank.simuls} with the so-called polynomial kernel of degree $3$ defined by $\mathbb{K}(x_1, x_2) = (1 + x_1^{\top}x_2)^3$ on the unit square $[0, 1] \times [0, 1]$.  
By contrast, Section~\ref{sec.poly.decay.simuls} tackles the polynomial decay kernels with the first-order Sobolev kernel $\mathbb{K}(x_1, x_2) = \min \{ x_1, x_2 \}$ on the unit square $[0, 1] \times [0, 1]$.

The performance of the early stopping rules is measured in terms of the $L_2(\mathbb{P}_n)$ squared norm $\lVert f^t - f^* \rVert_n^2$ averaged over $N = 100$ independent trials.

For our simulations, we use a variance estimation method that is described in Section~\ref{variance_decay}. This method is asymptotically unbiased, which is sufficient for our purposes.


\subsection{Results of the simulation experiments}\label{sec.simul.experiments.results}

\subsubsection{Finite-rank kernels}\label{sec.finite.rank.simuls}

\begin{figure}[htbp]
    \centering
    \begin{subfigure}{6cm}
        \centering
        \includegraphics[width=\linewidth]{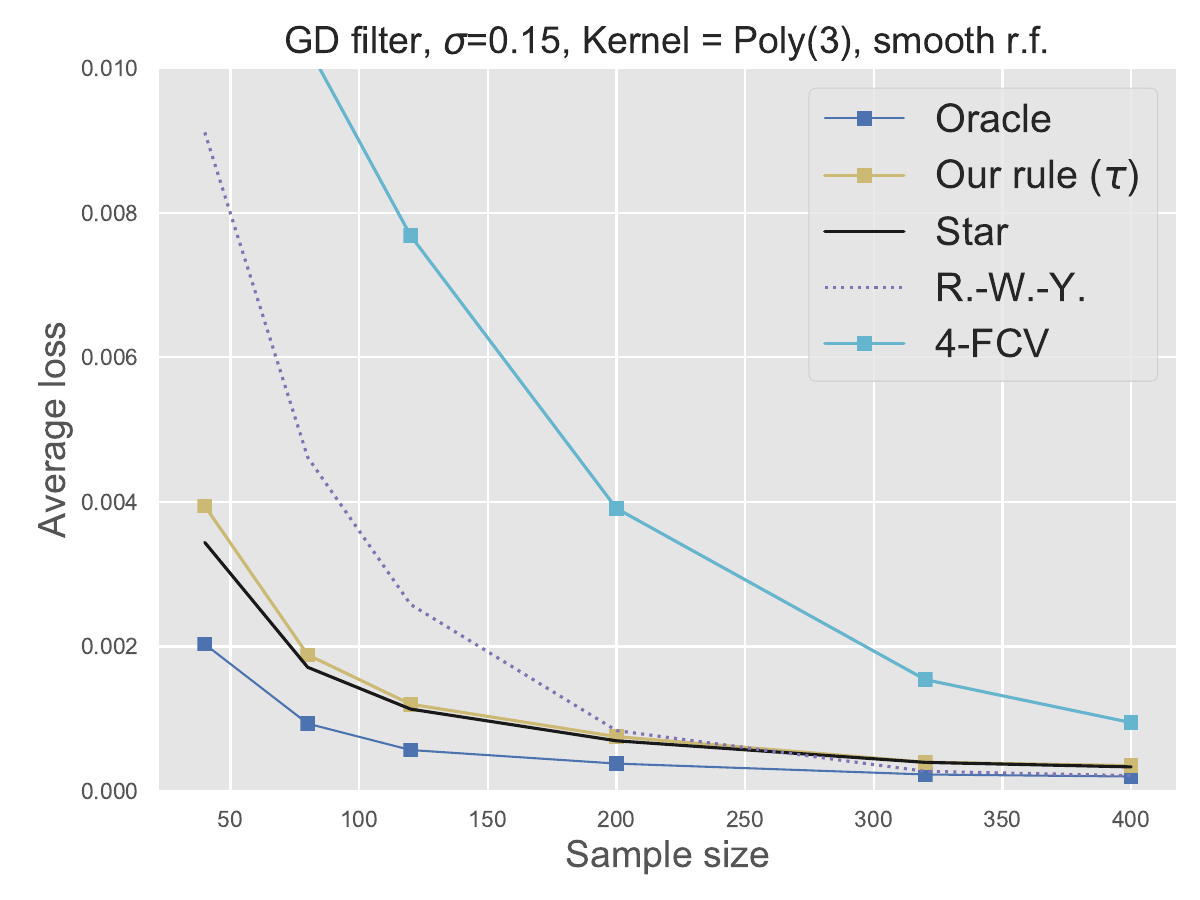} 
        \caption{\label{smooth}}
    \end{subfigure}
    \hfill
    \begin{subfigure}{6cm}
        \centering
        \includegraphics[width=\linewidth]{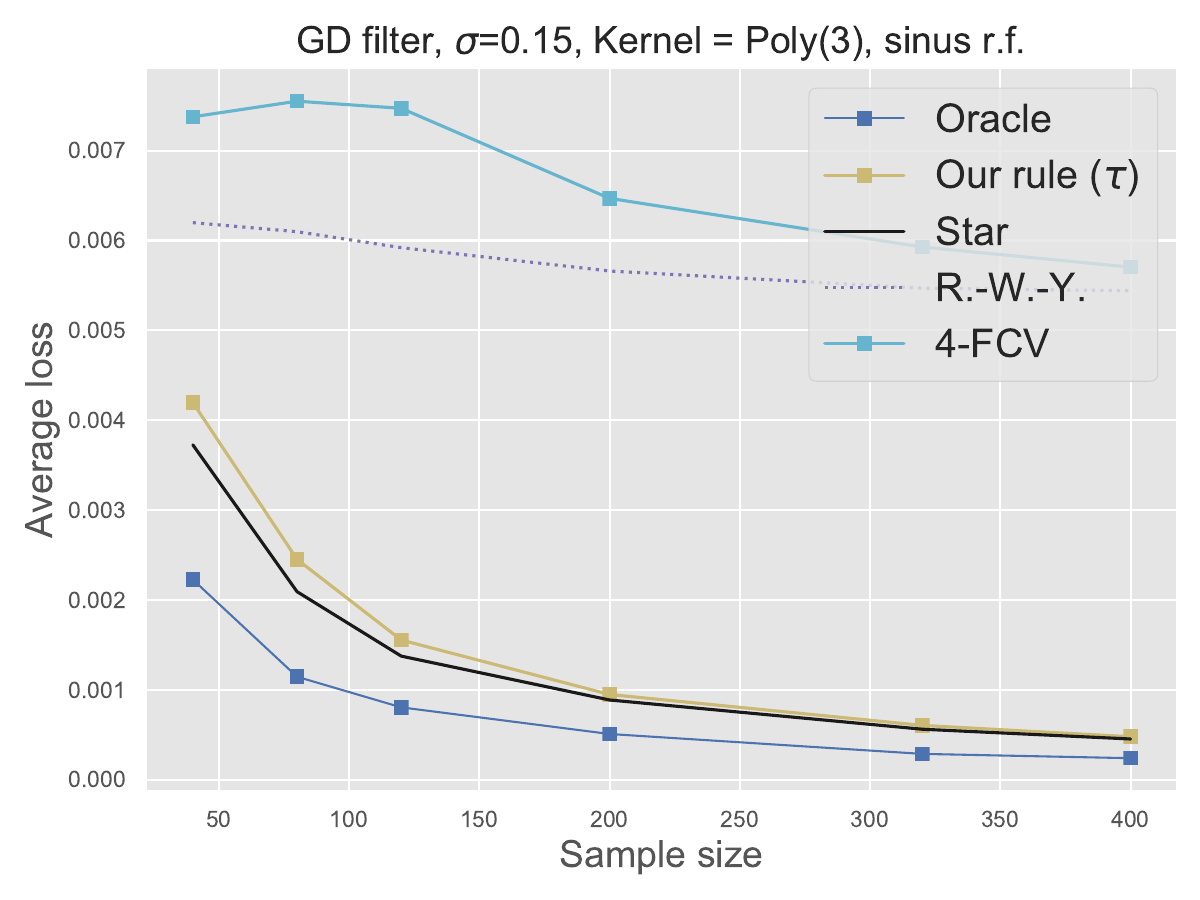} 
        \caption{\label{sinus}}
    \end{subfigure}
    \caption{Kernel gradient descent with the step-size $\eta = 1 / (1.2 \widehat{\mu}_1)$ and polynomial kernel $\mathbb{K}(x_1, x_2) = (1 + x_1^{\top}x_2)^3, \ x_1, x_2 \in [0, 1]$, for the estimation of two noised regression functions: the smooth $f^*(x) = |x - 1/2| - 1/2$ for panel (a), and the "sinus" $f^*(x) = 0.4 \ \textnormal{sin}(4 \pi x)$ for panel (b), with the equidistant covariates $x_j = j/n$. Each curve corresponds to the $L_2(\mathbb{P}_n)$ squared norm error for the stopping rules (\ref{t_or}), (\ref{t_star}), (\ref{t_w}), (\ref{t_vf}), (\ref{tau}) averaged over $100$ independent trials, versus the sample size $n = \{40, 80, 120, 200, 320, 400 \}$.}
    \label{fig:finite}
\end{figure}

Figure~\ref{fig:finite} displays the (averaged) $L_2(\mathbb{P}_n)$-norm error of the oracle stopping rule (\ref{t_or}), our stopping rule $\tau$ (\ref{tau}), $t^*$ (\ref{t_star}), minimax-optimal stopping rule $\widehat{\textnormal{T}}_{\textnormal{RWY}}$ (\ref{t_w}), and $4$-fold cross validation stopping time $\widehat{\textnormal{T}}_{\textnormal{VFCV}}$ (\ref{t_vf}) versus the sample size. 
Figure~\ref{smooth} shows the results for the piecewise linear regression function whereas Figure~\ref{sinus} corresponds to the "sinus" regression function. 

All the curves decrease as $n$ grows. 
From these graphs, the overall worst performance is achieved by $\widehat{\textnormal{T}}_{\textnormal{VFCV}}$, especially with a small sample size, which can be due to the additional randomness induced by the preliminary random splitting with $4-FCV$.
By contrast, the minimum discrepancy-based stopping rules ($\tau$ and $t^*$) exhibit the best performances compared to the results of $\widehat{\textnormal{T}}_{\textnormal{VFCV}}$ and $\widehat{\textnormal{T}}_{\textnormal{RWY}}$. 
The averaged mean-squared error of $\tau$ is getting closer to the one of $t^*$ as the number of samples $n$ increases, which was expected from the theory and also intuitively, since $\tau$ has been introduced as an estimator of $t^*$. 
From Figure~\ref{smooth}, $\widehat{\textnormal{T}}_{\textnormal{RWY}}$ is less accurate for small sample sizes, but improves a lot as $n$ grows up to achieving a performance similar to that of $\tau$. This can result from the fact that $\widehat{\textnormal{T}}_{\textnormal{RWY}}$ is built from upper bounds on the bias and variance terms, which are likely to be looser with a small sample size, but achieve an optimal convergence rate as $n$ increases.
On Figure~\ref{sinus}, the reason why $\tau$ exhibits (strongly) better results than $\widehat{\textnormal{T}}_{\textnormal{RWY}}$ owes to the main assumption on the regression function, namely that $ \lVert f^* \rVert_{\mathcal{H}} \leq 1$. This could be violated for the "sinus" function. 

\subsubsection{Polynomial eigenvalue decay kernels}\label{sec.poly.decay.simuls}
Figure~\ref{comp:pdk} displays the resulting (averaged over $100$ repetitions) $L_2(\mathbb{P}_n)$-error of $\tau_\alpha$ (with $\alpha = (\beta + 1)^{-1} = 0.33$) \eqref{t_alpha}, $\widehat{\textnormal{T}}_{\textnormal{RWY}}$ \eqref{t_w}, $t^*$ \eqref{t_star}, and $\widehat{\textnormal{T}}_{\textnormal{HO}}$ \eqref{t_ho} versus the sample size. 
\begin{figure}[htbp] \label{fig:finite_kernel}
    \centering
    \begin{subfigure}{6cm}
        \centering
        \includegraphics[width=\linewidth]{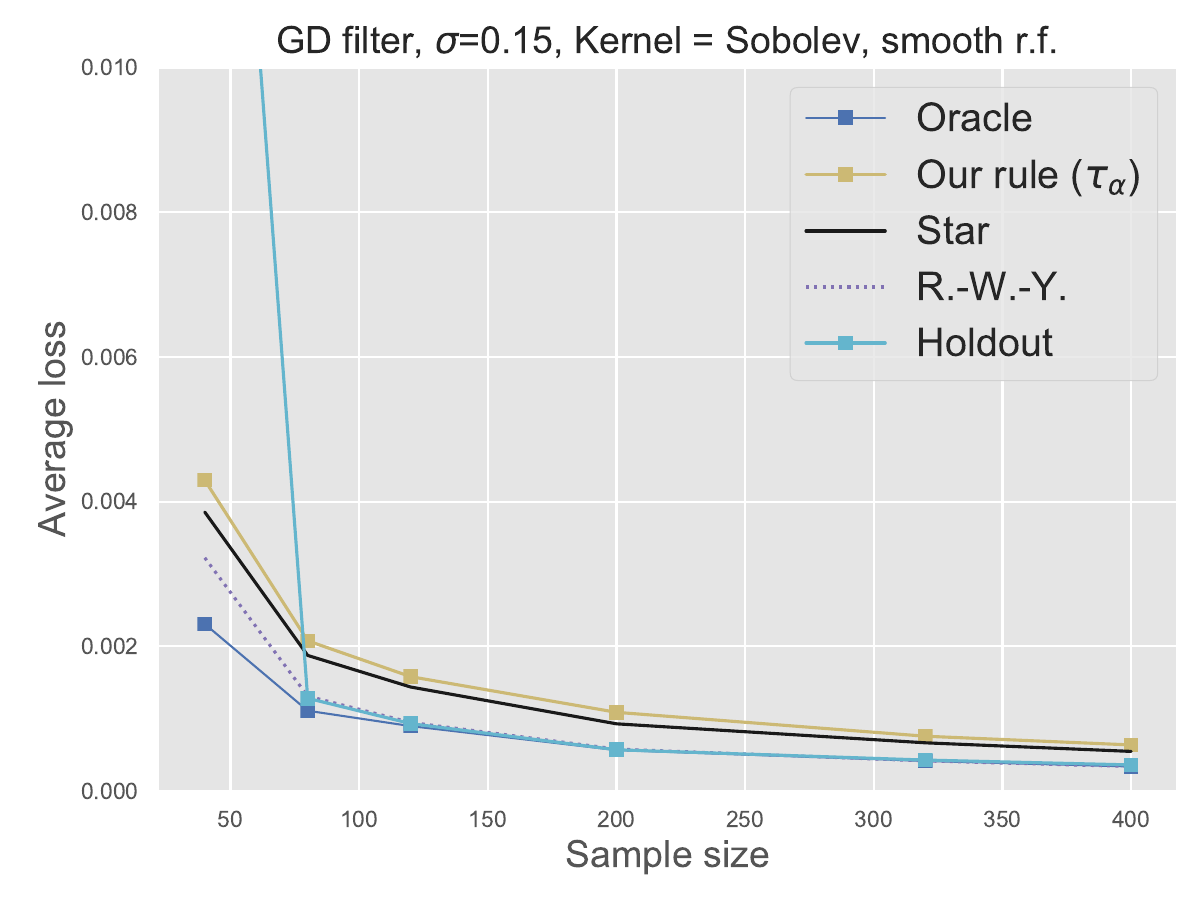} 
        \caption{\label{fig.smooth.poly.Decay}}
    \end{subfigure}
    \hfill
    \begin{subfigure}{6cm}
        \centering
        \includegraphics[width=\linewidth]{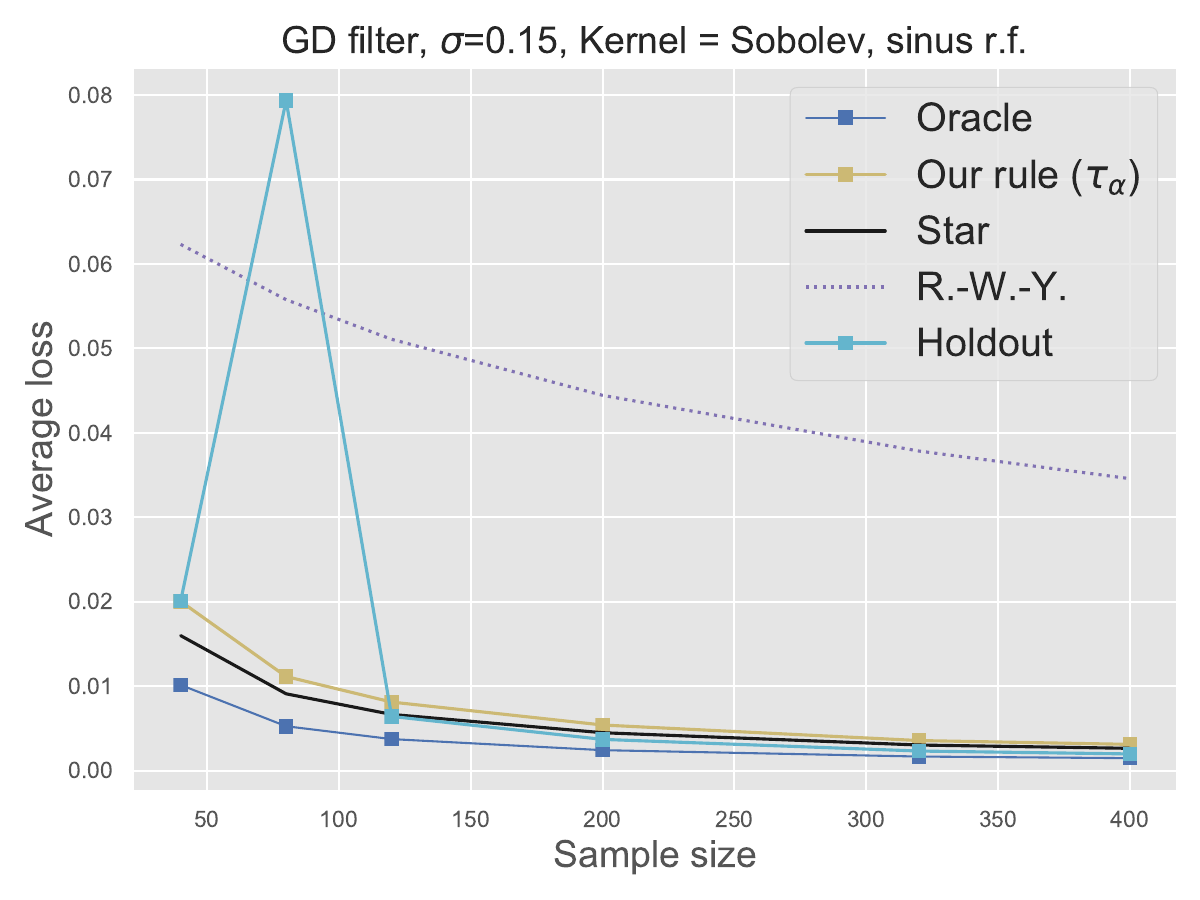} 
        \caption{\label{fig.sinus.poly.Decay}}
    \end{subfigure}
    \caption{Kernel gradient descent (\ref{iterations}) with the step-size $\eta = 1 / (1.2 \widehat{\mu}_1)$ and Sobolev kernel $\mathbb{K}(x_1, x_2) = \min \{ x_1, x_2\}, \ x_1, x_2 \in [0, 1]$ for the estimation of two noised regression functions: the smooth $f^*(x) = |x - 1/2| - 1/2$ for panel (a) and the "sinus" $f^*(x) = 0.4 \ \textnormal{sin}(4\pi x)$ for panel (b), with the equidistant covariates $x_j = j/n$. Each curve corresponds to the $L_2(\mathbb{P}_n)$ squared norm error for the stopping times (\ref{t_or}), (\ref{t_star}), (\ref{t_w}), (\ref{t_ho}), (\ref{t_alpha}) with $\alpha = 0.33$, averaged over $100$ independent trials, versus the sample size $n = \{40, 80, 120, 200, 320, 400 \}$.}
    \label{comp:pdk}
\end{figure}
Figure~\ref{fig.smooth.poly.Decay} shows that all stopping rules seem to work equivalently well, although there is a slight advantage for $\widehat{\textnormal{T}}_{\textnormal{HO}}$ and $\widehat{\textnormal{T}}_{\textnormal{RWY}}$ compared to $t^*$ and $\tau_\alpha$. However, as $n$ grows to $400$, the performances of all stopping rules become very close to each other. Let us emphasize that the true value of $\beta$ is not known in these experiments. Therefore, the value $(\beta + 1)^{-1} = 0.33$ has been estimated from the decay of the empirical eigenvalue of the normalized Gram matrix. This can explain why the performance of $\tau_\alpha$ remains worse than that of $\widehat{\textnormal{T}}_{\textnormal{RWY}}$.

The story described by Figure~\ref{fig.sinus.poly.Decay} is somewhat different. The first striking remark is that $\widehat{\textnormal{T}}_{\textnormal{RWY}}$ completely fails on this example, which still stems from the (unsatisfied) constraint on the $\mathcal{H}$-norm of $f^*$.
However, the best performance is still achieved by the Hold-out stopping rule, although $\tau_\alpha$ and $t^*$ remain very close to the latter.
The fact that $t^*$ remains close to the oracle stopping rule (without any need for smoothing) supports the idea that the minimum discrepancy is a reliable principle for designing an effective stopping rule. The deficiency of $\tau$ (by contrast to $\tau_\alpha$) then results from the variability of the empirical risk, which does not remain close enough to its expectation. This bad behavior is then balanced by introducing the polynomial smoothing at level $\alpha$ within the definition of $\tau_\alpha$, which enjoys close to optimal practical performances.

Let us also mention that $\widehat{\textnormal{T}}_{\textnormal{HO}}$ exhibit some variability, in particular, with small sample sizes as illustrated by Figures~\ref{fig.smooth.poly.Decay} and~\ref{fig.sinus.poly.Decay}.

The overall conclusion is that the smoothed minimum discrepancy-based stopping time $\tau_\alpha$ leads to almost optimal performances provided $\alpha= (\beta+1)^{-1}$, where $\beta$ quantifies the polynomial decay of the empirical eigenvalues $\{ \widehat{\mu}_i \}_{i=1}^n$.
\subsection{Estimation of variance and decay rate for polynomial eigenvalue decay kernels} \label{variance_decay}
The purpose of the present section is to describe two strategies for estimating: $(i)$ the decay rate of the empirical eigenvalues of the normalized Gram matrix, and $(ii)$ the variance parameter $\sigma^2$.
\subsubsection{Polynomial decay parameter estimation}
From the empirical version of the polynomial decay assumption \eqref{beta-polynomial}, one can easily derive upper and lower bounds for $\beta$ as $\frac{\log(\widehat{\mu}_i / \widehat{\mu}_{i+1}) - \log(C / c)}{\log(1 + 1/i)} \leq \beta \leq \frac{\log(\widehat{\mu}_i / \widehat{\mu}_{i+1}) + \log(C / c)}{\log(1 + 1/i)}$.
The difference between these upper and lower bounds is equal to $\frac{2 \log (C / c)}{\log(1 + 1/i)}$, which is minimized for $i=1$. Then the best precision on the estimated value of $\beta$ is reached with $i=1$, which yields the estimator $\widehat{\beta} = \frac{\log ( \widehat{\mu}_1 / \widehat{\mu}_2 )}{\log2}$. 
%
%
\subsubsection{Variance parameter estimation}
There is a bunch of suggestions for variance estimation with linear smoothers; see, e.g., Section 5.6 in the book \cite{wasserman2006all}.
In our simulation experiments, two cases are distinguished: the situation where the reproducing kernel has finite rank $r$, and the situation where $\textnormal{rk}(T_{\mathbb{K}}) = \infty$. In both cases, an asymptotically unbiased estimator of $\sigma^2$ is designed. 
\paragraph{Finite-rank kernel.}
With such a finite-rank kernel, the estimation of the noise is made from the coordinates $\{ Z_i \}_{i=r+1}^n$ corresponding to the situation, where $G_i^* = 0, \ i > r$ (see Section 4.1.1 in \cite{raskutti2014early}).
Actually, these coordinates (which are pure noise) are exploited to build an easy-to-compute estimator of $\sigma^2$, that is,
\begin{equation} \label{sigma_est_const}
    \widehat{\sigma}^2 = \frac{\sum_{i=n - r + 1}^n Z_i^2}{n - r}.
\end{equation}
\paragraph{Infinite-rank kernel.}
If $rk(T_{\mathbb{K}}) = \infty$, we suggest using the following result.
%
%
%
%
\begin{lemma} \label{var_est}
For any regular kernel (see Theorem \ref{theorem_yang}), any value of $t$ satisfying $ \eta t  \cdot \widehat{\epsilon}_n^2  \to +\infty$ as $n\to +\infty$ yields that $\widehat{\sigma}^2 = \frac{R_{t}}{\frac{1}{n}\sum_{i=1}^n (1 - \gamma_i^{(t)})^2}$ is an asymptotically unbiased estimator of $\sigma^2$.
\end{lemma}
A sketch of the proof of Lemma~\ref{var_est} is given in Appendix~\ref{lemma_var_est}. 
Based on this lemma, we suggest taking $t = T$, where $T$ is the maximum number of iterations allowed to execute due to computational constraints. 
Notice that as long as we access closed-form expressions of the estimator, there is no need to compute all estimators for $t$ between $1\leq t\leq T$.  
The final estimator of $\sigma^2$ used in the experiments of Section~\ref{sec.simul.experiments.results} is given by
\begin{equation} \label{sigma_est_full}
    \widehat{\sigma}^2 = \frac{R_{T}}{\frac{1}{n} \sum_{i=1}^n (1 - \gamma_i^{(T)})^2}.
\end{equation}
\section{Conclusion} \label{sec:6}
In this paper, we describe spectral filter estimators (e.g., gradient descent, kernel ridge regression) for the non-parametric regression function estimation in RKHS. Two new data-driven early stopping rules $\tau$ (\ref{tau}) and $\tau_{\alpha}$ (\ref{t_alpha}) for these iterative algorithms are designed. In more detail, we show that for the infinite-rank reproducing kernels, $\tau$ has a high variance due to the variability of the empirical risk around its expectation, and we proposed a way to reduce this variability by means of smoothing the empirical $L_2(\mathbb{P}_n)$-norm (and, as a consequence, the empirical risk) by the eigenvalues of the normalized kernel matrix. We demonstrate in Corollaries \ref{finite_rank_corollary} and \ref{pdk_corollary} that our stopping times $\tau$ and $\tau_{\alpha}$ yield minimax-optimal rates, in particular, for finite-rank kernel classes and Sobolev spaces. It is worth emphasizing that computing the stopping times requires \textit{only} the estimation of the variance $\sigma^2$ and computing $(\widehat{\mu}_1, \ldots, \widehat{\mu}_n)$. Theoretical results are confirmed empirically: $\tau$ and $\tau_{\alpha}$ with the smoothing parameter $\alpha = (\beta + 1)^{-1}$, where $\beta$ is the polynomial decay rate of the eigenvalues of the normalized Gram matrix, perform favorably in comparison with stopping rules based on hold-out data and 4-fold cross-validation.

There are various open questions that could be tackled after our results. A deficiency of our strategy is that the construction of $\tau$ and $\tau_{\alpha}$ is based on the assumption that the regression function belongs to a known RKHS, which restricts (mildly) the smoothness of the regression function. We would like to understand how our results extend to other loss functions besides the squared loss (for example, in the classification framework), as it was done in \cite{wei2017early}. Another research direction could be to use early stopping with fast approximation techniques for kernels \cite{rudi2015less} to avoid calculation of all eigenvalues of the normalized Gram matrix that can be prohibited for large-scale problems.

\appendix \label{appendix}
\section{Useful results} \label{general_appendix}
In this section, we present several auxiliary lemmas that are repeatedly used in the paper. 
%
%
%
%
\begin{lemma}\cite[$\eta_t = \eta t$ in Lemma 8 and $\nu = \eta t$ in Lemma 13]{raskutti2014early} \label{gamma_bounds}
For any bounded kernel, with $\gamma_i^{(t)}$ corresponding to gradient descent or kernel ridge regression, for every $t \geq 0$,
\begin{align}
     \frac{1}{2}\min \{1, \eta t \widehat{\mu}_i \} \leq \gamma_i^{(t)} \leq \min \{1, \eta t \widehat{\mu}_i \}, \ \ i= 1,\ldots, n.
    %
\end{align}

\end{lemma}

The following result shows the magnitude of the smoothed critical radius for polynomial eigenvalue decay kernels.
\begin{lemma} \label{critical_radius_empirical_smooth}
Assume that $\widehat{\mu}_i \leq C i^{-\beta}, \ i = 1, 2, \ldots, n$,  for $\alpha \beta < 1$, one has
\begin{equation*}
    \widehat{\epsilon}_{n, \alpha}^2 \asymp \left[ \sqrt{\frac{C^\alpha}{1 - \alpha \beta}} + \sqrt{\frac{C^{1+\alpha}}{\beta (1 + \alpha) - 1}} \right]^{\frac{2 \beta}{\beta + 1}} \left[ \frac{\sigma^2}{2 R^2 n} \right]^{\frac{\beta}{\beta + 1}}.    
\end{equation*}
\end{lemma}

\begin{proof}[Proof of Lemma~\ref{critical_radius_empirical_smooth}]
For every $M(\epsilon) \in (0, n]$ and $\alpha \beta < 1$, we have
\begin{align*}
    \widehat{\mathcal{R}}_{n, \alpha}(\epsilon, \mathcal{H}) &\leq R \sqrt{\frac{1}{n}} \sqrt{\sum_{j=1}^{n} \min \{C j^{-\beta}, \epsilon^2 \} C^{\alpha}j^{-\beta \alpha}}\\
    &\leq R \sqrt{\frac{C^{\alpha}}{n}}\sqrt{\sum_{j=1}^{\floor*{M(\epsilon)}} j^{-\beta \alpha}}\epsilon + R \sqrt{\frac{C^{1+\alpha}}{n}}\sqrt{\sum_{j = \ceil*{M(\epsilon)}}^{n} j^{-\beta - \beta \alpha}}\\
    &\leq R \sqrt{\frac{C^\alpha}{1 - \alpha \beta} \frac{M(\epsilon)^{1-\alpha \beta}}{n}}\epsilon + R \sqrt{\frac{C^{1+\alpha}}{n}} \sqrt{\frac{1}{\beta (1 + \alpha) - 1} \frac{1}{M(\epsilon)^{\beta (1 + \alpha) - 1}}}
\end{align*}
Set $M(\epsilon) = \epsilon^{-2/\beta}$ that implies $\sqrt{M(\epsilon)^{1-\alpha \beta}} \epsilon = \epsilon^{1 - \frac{1 - \alpha \beta}{\beta}}$, and 
\begin{equation*}
    \widehat{\mathcal{R}}_{n, \alpha}(\epsilon, \mathcal{H}) \leq R \left[ \sqrt{\frac{C^{\alpha}}{1 - \alpha \beta}} + \sqrt{\frac{C^{1+\alpha}}{\beta (1 + \alpha) - 1}} \right] \epsilon^{1 - \frac{1 - \alpha \beta}{\beta}}\frac{1}{\sqrt{n}}.
\end{equation*}
Therefore, the smoothed critical inequality $\widehat{\mathcal{R}}_{n,\alpha}(\epsilon, \mathcal{H}) \leq  \frac{2 R^2}{\sigma} \epsilon^{2+\alpha}$ is satisfied for 
\begin{equation} \label{final_result_for_smoothed_epsilon_n}
    \widehat{\epsilon}_{n, \alpha}^2 = \widetilde{c} \left[ \sqrt{\frac{C^\alpha}{1 - \alpha \beta}} + \sqrt{\frac{C^{1+\alpha}}{\beta (1 + \alpha) - 1}} \right]^{\frac{2 \beta}{\beta + 1}} \left[ \frac{\sigma^2}{2 R^2 n} \right]^{\frac{\beta}{\beta + 1}}.
\end{equation}
Notice that $M(\widehat{\epsilon}_{n, \alpha}) \asymp \left( \frac{R^2}{\sigma^2} \right)^{\frac{1}{\beta + 1}}n^{\frac{1}{\beta + 1}} \lesssim \left( \frac{R^2}{\sigma^2} \right)^{\frac{1}{\beta + 1}} n$. Besides that, due to Lemma \ref{smallest_positive_solution}, one can choose a positive constant $\widetilde{c}$ in Eq. (\ref{final_result_for_smoothed_epsilon_n}) such that $M(\widehat{\epsilon}_{n, \alpha}) \leq n$.  
\end{proof}
%
%
%
%
For the next two lemmas define the positive self-adjoint trace-class covariance operator 
\begin{equation*}
    \Sigma \coloneqq \mathbb{E}_X\left[ \mathbb{K}(\cdot, X) \otimes \mathbb{K}(\cdot, X) \right],
\end{equation*}
where $\otimes$ is the Kronecker product between two elements in $\mathcal{H}$ such that $(a \otimes b)u = a\langle b, u \rangle_{\mathcal{H}}$, for every $u \in \mathcal{H}$. We know that $\Sigma$ and $T_{\mathbb{K}}$ have the same eigenvalues $\{ \mu_j \}_{j=1}^{\infty}$. Moreover, we introduce the smoothed empirical covariance operator as 
\begin{equation}
        \widehat{\Sigma}_{n, \alpha} \coloneqq \frac{1}{n}\sum_{j=1}^n \widehat{\mu}_j^{2\alpha} \mathbb{K}(\cdot, x_j) \otimes \mathbb{K}(\cdot, x_j).
    \end{equation}
\begin{lemma} \label{right_tail_bound_smoothed}
    For each $a > 0$, any $1 \leq k \leq n$, $\alpha \in [0, 1/2]$, and $\theta > 1$, one has
    \begin{equation*}
        \mathbb{P}_X\left( \sum_{j=1}^k \mu_j^{2\alpha} > \frac{\theta }{\theta - 1}\sum_{j=1}^{k}\widehat{\mu}_j^{2\alpha}  + \frac{a\left(1 + 3\theta \right) \theta}{3(\theta - 1) n} \right) \leq 2 \exp(-a)
    \end{equation*}
\end{lemma}
\begin{proof}
    Let $\Pi_k$ be the orthogonal projection from $\mathcal{H}$ onto the span of the eigenfunctions $(\phi_j: j = 1, \ldots, k)$. Then by the variational characterization of partial traces, one has $\sum_{j=1}^k \mu_j^{2\alpha} = \textnormal{tr}\left( \Pi_k \Sigma^{2\alpha} \right)$ and $\sum_{j=1}^k \widehat{\mu}_j^{2\alpha} \geq \textnormal{tr}\left( \Pi_k \widehat{\Sigma}_{n, \alpha} \right)$. One concludes that 
    \begin{equation*}
        \sum_{j=1}^k \mu_j^{2\alpha} - \sum_{j=1}^k \widehat{\mu}_j^{2\alpha} \leq \textnormal{tr}\left( \Pi_k \left( \Sigma^{2\alpha} - 
        \widehat{\Sigma}_{n, \alpha} \right) \right).
    \end{equation*}
    By reproducing property and Mercer's theorem, $\lVert \Pi_k \mathbb{K}\left( \cdot, X \right) \rVert_{\mathcal{H}}^2 = \sum_{i=1}^k \mu_i \phi_i^2(X)$, and
    \begin{align*}
        \sum_{j=1}^k \mu_j^{2\alpha} - \sum_{j=1}^k \widehat{\mu}_j^{2\alpha} &\leq \mathbb{E}_X\lVert \Pi_k \Sigma^{\alpha - \frac{1}{2}}\mathbb{K}\left( \cdot, X \right) \rVert_{\mathcal{H}}^2 - \frac{1}{n}\sum_{j=1}^n \widehat{\mu}_j^{2\alpha}\lVert \Pi_k \mathbb{K}\left( \cdot, x_j \right) \rVert_{\mathcal{H}}^2 \\
        &\leq \mid \mathbb{E}_X\lVert \Pi_k \Sigma^{\alpha - \frac{1}{2}}\mathbb{K}\left( \cdot, X \right) \rVert_{\mathcal{H}}^2 - \frac{1}{n}\sum_{j=1}^n \widehat{\mu}_j^{2\alpha}\lVert \Pi_k \mathbb{K}\left( \cdot, x_j \right) \rVert_{\mathcal{H}}^2 \mid.
    \end{align*}
    Since $\widehat{\mu}_j^{2\alpha}\lVert \Pi_k \mathbb{K}\left(\cdot, x_j \right) \rVert_{\mathcal{H}}^2 \leq 1$, one has $\mathbb{E}_X \left[ \widehat{\mu}_j^{4\alpha} \lVert \Pi_k \mathbb{K}\left( \cdot, x_j \right) \rVert_{\mathcal{H}}^4 \right] \leq \sum_{i=1}^k \mu_i$, and by Bernstein's inequality, for any $a > 0$,
    \begin{equation*}
        \mathbb{P}_X\left(\sum_{j=1}^k \mu_j^{2\alpha} > \sum_{j=1}^k \widehat{\mu}_j^{2 \alpha} + \sqrt{\frac{2 a \left( \sum_{j=1}^k \mu_j \right) }{n}} + \frac{a}{3n} \right) \leq 2 \exp (-a).
    \end{equation*}
    Then, by using $\sum_{j=1}^k \mu_j \leq \sum_{j=1}^k \mu_j^{2\alpha}$ when $\alpha \in [0, 1/2]$, and $\sqrt{2 x y } \leq \theta x + \frac{y}{\theta}$ for any $\theta > 0$, one gets
    \begin{equation*}
        \mathbb{P}_X\left( \left( 1 - \frac{1}{\theta} \right) \sum_{j=1}^k \mu_j^{2\alpha} >  \sum_{j=1}^k \widehat{\mu}_j^{2\alpha} + \frac{a\left( 1 + 3\theta  \right)}{3n} \right) \leq 2 \exp (-a),
    \end{equation*}
    for any $a > 0$.
\end{proof}
\begin{lemma} \label{left_tail_bound_smoothed}
    For each $a > 0$, any $0 \leq k \leq n$, $\alpha \in [0, 1/2]$, and $\theta > 1$, one has
    \begin{equation*}
        \mathbb{P}_X \left( \sum_{j>k} \widehat{\mu}_j^{2\alpha} > \frac{\theta + 1}{\theta} \sum_{j > k} \mu_j^{2\alpha} + \frac{a\left( 1 + 3 \theta \right)}{3 n}  \right) \leq \exp(-a).
    \end{equation*} 
\end{lemma}
\begin{proof}
    The proof of \cite[Lemma 33]{celisse2021analyzing} could be easily generalized to the smoothed version by using the proof of Lemma \ref{right_tail_bound_smoothed}.
    Let $\Pi_k$ be the orthogonal projection from $\mathcal{H}$ onto the span of the population eigenfunctions $\left(\phi_j: j > k \right)$. Then by the variational characterization of partial traces, one has $\sum_{j>k}\mu_j^{2\alpha} = \textnormal{tr}\left( \Pi_k \Sigma^{2\alpha} \right)$ and $\sum_{j>k}\widehat{\mu}_j^{2\alpha} \leq \textnormal{tr}\left( \Pi_k \widehat{\Sigma}_{n, \alpha} \right)$.
    One concludes that
    \begin{equation*}
        \sum_{j>k}\widehat{\mu}_j^{2\alpha} - \sum_{j>k}\mu_j^{2\alpha} \leq \textnormal{tr}\left( \Pi_k \left(\widehat{\Sigma}_{n, \alpha} - \Sigma^{2\alpha}\right) \right).
    \end{equation*}
    Hence,
    \begin{equation*}
        \sum_{j>k} \widehat{\mu}_j^{2\alpha} - \sum_{j>k}\mu_j^{2\alpha} \leq \frac{1}{n}\sum_{j=1}^n \widehat{\mu}_j^{2\alpha} \lVert \Pi_k \mathbb{K}\left( \cdot, x_j \right)  \rVert_{\mathcal{H}}^2 - \mathbb{E}_X \lVert \Pi_k \Sigma^{\alpha - 1/2}\mathbb{K}\left( \cdot, X \right) \rVert_{\mathcal{H}}^2 .
    \end{equation*}
    Since $\widehat{\mu}_j^{2\alpha}\lVert \Pi_k \mathbb{K}(\cdot, x_j) \rVert_{\mathcal{H}}^2 \leq 1$ and by using the reproducing property and Mercer's theorem, $\lVert \Pi_k \mathbb{K}\left( \cdot, X \right) \rVert_{\mathcal{H}}^2 = \sum_{j>k}\mu_j \phi_j^2(X)$, one has 
    \begin{equation*}
        \mathbb{E}_X \left[ \widehat{\mu}_j^{4\alpha} \lVert \Pi_k \mathbb{K}\left( \cdot, x_j \right) \rVert_{\mathcal{H}}^4 \right] \leq \sum_{j>k}\mu_j.
    \end{equation*}
    Bernstein's inequality yields that for any $a > 0$,
    \begin{equation*}
        \mathbb{P}_X\left( \sum_{j>k}\widehat{\mu}_j^{2\alpha} > \sum_{j>k}\mu_j^{2\alpha} + \sqrt{\frac{2a \left( \sum_{j>k}\mu_j \right)}{n}} + \frac{a}{3n} \right) \leq \exp(-a).
    \end{equation*}
    Using the inequalities $\sum_{j>k}\mu_j \leq \sum_{j>k}\mu_j^{2\alpha}$, when $\alpha \in [0, 1/2]$, and
    \begin{equation*}
        \sqrt{\frac{2a \left( \sum_{j>k}\mu_j \right) }{n}} \leq \frac{1}{\theta}\sum_{j>k}\mu_j + \frac{a\theta}{n},
    \end{equation*}
    one gets 
    \begin{equation*}
        \mathbb{P}_X\left( \sum_{j>k}\widehat{\mu}_j^{2\alpha} > \left(1 + \frac{1}{\theta}\right) \sum_{j>k}\mu_j^{2\alpha} + \frac{a\left(1 + 3 \theta\right)}{3n} \right) \leq \exp(-a),
    \end{equation*}
    for any $a > 0$ and $\theta > 1$.
\end{proof}
\begin{corollary} \label{corollary_for_assumption}
    Assumption \ref{sufficient_smoothing}, Lemma \ref{right_tail_bound_smoothed}, and Lemma \ref{left_tail_bound_smoothed} imply that for any $1 \leq k \leq n$, $a > 0$, $\theta > 1$, and $\alpha \in [\alpha_0, 1/2]$, 
    \begin{equation}
        \sum_{j = k + 1}^{n}\widehat{\mu}_j^{2\alpha} \leq \frac{(\theta + 1)\mathcal{M}}{\theta - 1}\sum_{j=1}^k \widehat{\mu}_j^{2 \alpha} + \frac{a(1 + 3 \theta)}{3 n}\left( \frac{\mathcal{M}(\theta + 1)}{\theta - 1} + 1 \right) 
    \end{equation}
    with probability (over $\{x_i \}_{i=1}^n$) at least $1 - 3 \exp(-a)$.
\end{corollary}

\section{Handling the smoothed bias and variance} \label{smoothed_bias_variance}
%
%
%
%
\begin{lemma} \label{smooth_bias_upper_bound}
Under Assumptions \ref{a1}, \ref{a2},
\begin{equation} 
    B_{\alpha}^2(t) \leq  \frac{R^2}{(\eta t)^{1+\alpha}}, \quad \alpha \in [0, 1]. 
\end{equation}
\end{lemma}

\begin{proof}[Proof of Lemma \ref{smooth_bias_upper_bound}]
Proof of \cite[Lemma 7]{raskutti2014early} can be easily generalized to obtain the result.
\end{proof}
%
Here, we recall one concentration result from \cite[Section 4.1.2]{raskutti2014early}. For any $t >0$ and $\delta > 0$, one has $V(t) = \mathbb{E}_{\varepsilon}\left[v(t)\right]$, and
\begin{equation} \label{hanson-wright_in_t}
    \mathbb{P}_{\varepsilon} \Big( |v(t) - V(t)| \geq \delta \Big) \leq 2 \ \exp \left[ - \frac{c n \delta}{\sigma^2} \min \left\{ 1, \frac{R^2 \delta}{\sigma^2 \eta t \widehat{\mathcal{R}}_{n}^2(\frac{1}{\sqrt{\eta t}}, \mathcal{H})}  \right\} \right].
\end{equation}
\section{Auxiliary lemma for finite-rank kernels} \label{finite_rank_appendix}

Let us first transfer the critical inequality (\ref{RK_critical_radius_empirical}) from $\epsilon$ to $t$.

\begin{definition}\label{t_epsilon}
Set $\epsilon = \frac{1}{\sqrt{\eta t}}$ in (\ref{RK_critical_radius_empirical}), and let us define $\widehat{t}_{\epsilon}$ as the largest positive solution to the following fixed-point equation
    \begin{equation} \label{fixed_point_rademacher_in_t}
    \frac{\sigma^2 \eta t}{R^2} \widehat{\mathcal{R}}_n^2\left(\frac{1}{\sqrt{\eta t}}, \mathcal{H}\right) \leq \frac{4 R^2}{\eta t}.
    \end{equation}    
\end{definition}
Note that the empirical critical radius $\widehat{\epsilon}_n = \frac{1}{\sqrt{\eta \widehat{t}_{\epsilon}}}$, and such a point $\widehat{t}_{\epsilon}$ exists since $\widehat{\epsilon}_n$ exists and is unique \cite{mendelson2002geometric, bartlett2005local, raskutti2014early}. Moreover, $\widehat{t}_{\epsilon}$ provides the equality in Ineq. (\ref{fixed_point_rademacher_in_t}). 

Remark that at $t = t^*: B^2(t) = \frac{2 \sigma^2}{n}\sum_{i=1}^r \gamma_i^{(t)} - V(t) \geq \frac{\sigma^2}{n}\sum_{i=1}^r \gamma_i^{(t)}$. Thus, due to the construction of $\widehat{t}_{\epsilon}$ ( $\widehat{t}_{\epsilon}$ is the point of intersection of an upper bound on the bias and a lower bound on $\frac{\sigma^2}{2 n}\sum_{i=1}^r \gamma_i^{(t)}$) and monotonicity (in $t$) of all the terms involved, we get $t^* \leq \widehat{t}_{\epsilon}$.

\vspace{5mm}

\begin{lemma} \label{cl1l2}
Recall the definition of the stopping rule $t^*$ (\ref{t_star}). Under Assumptions \ref{a1}, \ref{a2}, and \ref{additional_assumption_gd_krr}, the following holds for any reproducing kernel:
\begin{equation*}
\mathbb{E}_{\varepsilon} \lVert f^{t^*} - f^* \rVert_n^2 \leq 8 R^2 \widehat{\epsilon}_n^2.
\end{equation*}
\end{lemma}

\begin{proof}[Proof of Lemma \ref{cl1l2}]
Let us define a proxy version of the variance term: $\widetilde{V}(t) \coloneqq \frac{\sigma^2}{n}\sum_{i=1}^r \gamma_i^{(t)}$. Moreover, for all $t > 0$,
\begin{equation}\label{exp_bias_at_t}
    \mathbb{E}_{\varepsilon}R_t = B^2(t) + \frac{\sigma^2}{n}\sum_{i=1}^n (1 - \gamma_i^{(t)})^2.
\end{equation}
From the fact that $\mathbb{E}_{\varepsilon}R_{t^*} = \sigma^2$, $\mathbb{E}_{\varepsilon}\lVert f^{t^*} - f^* \rVert_n^2 = B^2(t^*) + V(t^*) = 2 \widetilde{V}(t^*)$.  

Therefore, in order to prove the lemma, our goal is to get an upper bound on $\widetilde{V}(t^*)$. Since the function $\eta t \widehat{\mathcal{R}}_{n}^2(\frac{1}{\sqrt{\eta t}}, \mathcal{H})$ is monotonic in $t$ (see, for example, Lemma \ref{smallest_positive_solution}), and $t^* \leq \widehat{t}_{\epsilon}$, we conclude that
\begin{equation*}
    \widetilde{V}(t^*) \leq \frac{\sigma^2 \eta t^*}{R^2} \widehat{\mathcal{R}}_{n}^2\left(\frac{1}{\sqrt{\eta t^*}}, \mathcal{H}\right) \leq \frac{\sigma^2 \eta \widehat{t}_{\epsilon}}{R^2}\widehat{\mathcal{R}}_{n}^2\left(\frac{1}{\sqrt{\eta \widehat{t}_{\epsilon}}}, \mathcal{H}\right) = 4 R^2 \widehat{\epsilon}_n^2.
\end{equation*}

\end{proof}

\section{Proofs for polynomial smoothing (fixed design)} \label{polynomial_appendix}

In the proofs, we will need three additional definitions below.

\begin{definition}
In Definition \ref{empirical_rademacher_complexity_def}, set $\epsilon = \frac{1}{\sqrt{\eta t}}$, then for any $\alpha \in [0, 1]$, the smoothed critical inequality (\ref{empirical_rademacher_complexity_def}) is equivalent to
\begin{equation} \label{critical_inequality_in_t}
    \frac{\sigma^2 \eta t}{4} \widehat{\mathcal{R}}_{n, \alpha}^2\Big(\frac{1}{\sqrt{\eta t}}, \mathcal{H}\Big) \leq \frac{R^4}{(\eta t)^{1+\alpha}}. 
\end{equation}
Due to Lemma \ref{smallest_positive_solution}, the left-hand side of (\ref{critical_inequality_in_t}) is non-decreasing in $t$, and the right-hand side is non-increasing in $t$. 

\begin{definition} \label{t_epsilon_alpha_def}
For any $\alpha \in [0, 1]$, define the stopping rule $\widehat{t}_{\epsilon, \alpha}$ such that
    \begin{equation}
        \widehat{\epsilon}_{n, \alpha}^2 = \frac{1}{\eta \widehat{t}_{\epsilon, \alpha}},
    \end{equation}
\end{definition}
then Ineq. (\ref{critical_inequality_in_t}) becomes the equality at $t = \widehat{t}_{\epsilon, \alpha}$ thanks to the monotonicity and continuity of both terms in the inequality.
\end{definition}

Further, we define the stopping time $\widetilde{t}_{\epsilon, \alpha}$ and $\overline{t}_{\epsilon, \alpha}$, a lower bound and an upper bound on $t_{\alpha}^* \coloneqq \inf \left\{t > 0 \ | \ \mathbb{E}_{\varepsilon}R_{\alpha, t} \leq \frac{\sigma^2}{n}\sum_{i=1}^n \widehat{\mu}_i^{\alpha} \right\}$, $\forall \alpha \in [0, 1]$.

\begin{definition} \label{definition_additional_stopping_times}
Define the smoothed proxy variance $\widetilde{V}_{\alpha}(t) \coloneqq \frac{\sigma^2}{n}\sum_{i=1}^n \widehat{\mu}_i^{\alpha}\gamma_i^{(t)}$ and the following stopping times
\begin{align}
    \begin{split}
    \overline{t}_{\epsilon, \alpha} &= \inf \big\{ t > 0 \ | \ 
    B_{\alpha}^2(t) = \frac{1}{2}\widetilde{V}_{\alpha}(t) \big\},\\
    \widetilde{t}_{\epsilon, \alpha} &= \inf \big\{t > 0 \ | \ B_{\alpha}^2(t) = 3 \widetilde{V}_{\alpha}(t) \big\}.
    \end{split}
\end{align}
\end{definition}
Notice that at $t = \widetilde{t}_{\epsilon, \alpha}$:
\begin{equation*}
   \frac{6 R^2}{(\eta t)^{1+\alpha}} \geq \frac{R^2}{(\eta t)^{1+\alpha}} \geq B_{\alpha}^2(t) = 3 \widetilde{V}_{\alpha}(t) \geq \frac{3}{2}\frac{\sigma^2}{R^2} \eta t \widehat{\mathcal{R}}_{n, \alpha}^2\Big(\frac{1}{\sqrt{\eta t}}, \mathcal{H}\Big).
\end{equation*}
At $t = \overline{t}_{\epsilon, \alpha}$:
\begin{equation*}
    \frac{R^2}{(\eta t)^{1+\alpha}} \geq B_{\alpha}^2(t) = \frac{1}{2}\widetilde{V}_{\alpha}(t) \geq \frac{\sigma^2 \eta t}{4R^2} \widehat{\mathcal{R}}_{n, \alpha}^2\Big(\frac{1}{\sqrt{\eta t}}, \mathcal{H}\Big).
\end{equation*}
Thus, $\widetilde{t}_{\epsilon, \alpha}$ and $\overline{t}_{\epsilon, \alpha}$ satisfy the smoothed critical inequality (\ref{critical_inequality_in_t}). Moreover, $\widehat{t}_{\epsilon, \alpha}$ is always greater than or equal to $\overline{t}_{\epsilon, \alpha}$ and $\widetilde{t}_{\epsilon, \alpha}$ since $\widehat{t}_{\epsilon, \alpha}$ is the largest value satisfying Ineq. (\ref{critical_inequality_in_t}). As a consequence of Lemma \ref{smallest_positive_solution} and continuity of (\ref{critical_inequality_in_t}) in $t$, one has $\frac{1}{\eta \widetilde{t}_{\epsilon, \alpha}} \asymp \frac{1}{\eta \overline{t}_{\epsilon, \alpha}} \asymp \frac{1}{\eta \widehat{t}_{\epsilon, \alpha}} = \widehat{\epsilon}_{n, \alpha}^2$.
We assume for simplicity that 
\begin{align*}
    \overline{\epsilon}_{n, \alpha}^2 &\coloneqq \frac{1}{\eta \overline{t}_{\epsilon, \alpha}} = c^\prime \frac{1}{\eta \widehat{t}_{\epsilon, \alpha}} = c^\prime \widehat{\epsilon}_{n, \alpha}^2, \\
    \widetilde{\epsilon}_{n, \alpha}^2 &\coloneqq \frac{1}{\eta \widetilde{t}_{\epsilon, \alpha}} = c^{\prime \prime} \frac{1}{\eta \widehat{t}_{\epsilon, \alpha}} = c^{\prime \prime}\widehat{\epsilon}_{n, \alpha}^2
\end{align*}
for some positive numeric constants $c^{\prime}, c^{\prime \prime} \geq 1$, that do not depend on $n$, due to the fact that $\widehat{t}_{\epsilon, \alpha} \geq \overline{t}_{\epsilon, \alpha}$ and $\widehat{t}_{\epsilon, \alpha} \geq \widetilde{t}_{\epsilon, \alpha}$.

The following lemma decomposes the risk error into several parts that will be further analyzed in subsequent Lemmas \ref{variance_deviation}, \ref{bias_deviation}.

\begin{lemma} \label{first_lemma}
Recall the definition of $\tau_{\alpha}$ (\ref{t_alpha}), then \begin{equation*}
    \lVert f^{\tau_{\alpha}} - f^* \rVert_n^2 \leq 2B^2(\tau_{\alpha}) + 2 v(\tau_{\alpha}),
\end{equation*}
where $v(t) = \frac{1}{n}\sum_{i=1}^n (\gamma_i^{(t)})^2 \varepsilon_i^2, \ t > 0$, is the stochastic part of the variance. 
\end{lemma}

\begin{proof}[Proof of Lemma \ref{first_lemma}]

Let us define the noise vector $\varepsilon \coloneqq [\varepsilon_1, ..., \varepsilon_n]^{\top}$ and, for each $t > 0$, two vectors that correspond to the bias and variance parts: 
\begin{equation} \label{bias_variance_parts}
\Tilde{b}^2(t) \coloneqq (g_t(K_n)K_n - I_n)F^*, \ \ \ \  \Tilde{v}(t) \coloneqq g_t(K_n)K_n \varepsilon.
\end{equation}
It gives the following expressions for the stochastic part of the variance and bias:
\begin{equation} \label{n_th_scalar_pr_bias_var}
    v(t) = \langle \Tilde{v}(t), \Tilde{v}(t) \rangle_n, \ \ B^2(t) = \langle \Tilde{b}^2(t), \Tilde{b}^2(t) \rangle_n.    
\end{equation}
General expression for the $L_2(\mathbb{P}_n)$-norm error at $\tau_{\alpha}$ takes the form
\begin{equation} \label{l_2_p_n_norm}
    \lVert f^{\tau_{\alpha}} - f^* \rVert_n^2 = B^2(\tau_{\alpha}) + v(\tau_{\alpha}) + 2 \langle \Tilde{b}^2(\tau_{\alpha}), \Tilde{v}(\tau_{\alpha}) \rangle_n.
\end{equation}
%
%
Therefore, applying the inequality $2 \left| \langle x, y \rangle_n \right| \leq \lVert x \rVert_n^2 + \lVert y \rVert_n^2$ for any $x, y \in \mathbb{R}^n$, and (\ref{n_th_scalar_pr_bias_var}), we obtain
%
\begin{equation} \label{wo_expectation}
    \lVert f^{\tau_{\alpha}} - f^* \rVert_n^2 \leq 2 B^2(\tau_{\alpha}) + 2 v(\tau_{\alpha}).
\end{equation}

\end{proof}

\subsection{Two deviation inequalities for $\tau_{\alpha}$}

This is the first deviation inequality for $\tau_{\alpha}$ that will be used in Lemma \ref{variance_deviation} to control the variance term.

\begin{lemma} \label{lemma_deviation_bound_1}
Recall Definition \ref{definition_additional_stopping_times} of $\overline{t}_{\epsilon, \alpha}$, then under Assumptions \ref{a1}, \ref{a2}, \ref{additional_assumption_gd_krr}, and \ref{sufficient_smoothing},  
\begin{equation*}
    \mathbb{P}_{\varepsilon} \left( \tau_{\alpha} > \overline{t}_{\epsilon, \alpha} \right) \leq 5 \exp \left[ - c_1 \frac{R^2}{\sigma^2} n\widehat{\epsilon}_{n, \alpha}^{2(1 + \alpha)} \right],
\end{equation*}
where a positive constant $c_1$ depends only on $\mathcal{M}$.
\end{lemma}

\begin{proof}[Proof of Lemma \ref{lemma_deviation_bound_1}]
Set $\kappa_{\alpha} \coloneqq \sigma^2 \textnormal{tr} K_n^{\alpha} / n$, then due to the monotonicity of the smoothed empirical risk, for all $t \geq t_{\alpha}^*$,
\begin{equation*}
    \mathbb{P}_{\varepsilon} \left( \tau_{\alpha} > t \right) = \mathbb{P}_{\varepsilon} \left( R_{\alpha, t} - \mathbb{E}_{\varepsilon}R_{\alpha, t}  > \kappa_{\alpha} - \mathbb{E}_{\varepsilon} R_{\alpha, t}  \right). 
\end{equation*}
Consider 
\begin{equation}
    R_{\alpha, t} - \mathbb{E}_{\varepsilon}R_{\alpha, t} = \underbrace{\frac{\sigma^2}{n}\sum_{i=1}^n \widehat{\mu}_i^{\alpha} (1 - \gamma_i^{(t)})^2 \left(\frac{\varepsilon_i^2}{\sigma^2} - 1\right)}_{\Sigma_1} + \underbrace{\frac{2}{n}\sum_{i=1}^n \widehat{\mu}_i^{\alpha} (1 - \gamma_i^{(t)})^2G_i^* \varepsilon_i}_{\Sigma_2}.
\end{equation}
Define 
\begin{equation*}
    \Delta_{t, \alpha} \coloneqq \kappa_{\alpha} - \mathbb{E}_{\varepsilon}R_{\alpha, t} = -B_{\alpha}^2(t) - V_{\alpha}(t) + 2\widetilde{V}_{\alpha}(t),
\end{equation*}
where $\widetilde{V}_{\alpha}(t) = \frac{\sigma^2}{n} \sum_{i=1}^n \widehat{\mu}_i^{\alpha}\gamma_i^{(t)}$.

Further, set $t = \overline{t}_{\epsilon, \alpha}$, and recall that $\eta \overline{t}_{\epsilon, \alpha} = \frac{\eta \widehat{t}_{\epsilon, \alpha}}{c^\prime}$ for $c^\prime \geq 1$. This implies 
\begin{align*}
    \Delta_{\overline{t}_{\epsilon, \alpha}, \alpha} \geq \frac{1}{2}\widetilde{V}_{\alpha}(\overline{t}_{\epsilon, \alpha}) &\geq \frac{\sigma^2}{4 n }\sum_{i=1}^n \widehat{\mu}_i^{\alpha} \min \left\{ 1, \frac{\eta \widehat{t}_{\epsilon, \alpha}}{c^\prime} \widehat{\mu}_i \right\} \\ &= \frac{\sigma^2 \eta \widehat{t}_{\epsilon, \alpha}}{4 n c^{\prime}} \sum_{i=1}^n \widehat{\mu}_i^{\alpha} \min \left\{ \frac{c^\prime}{\eta \widehat{t}_{\epsilon, \alpha}}, \widehat{\mu}_i \right\}\\
    &\geq \frac{\sigma^2 \eta \widehat{t}_{\epsilon, \alpha}}{4 c^{\prime} R^2} \widehat{\mathcal{R}}_{n, \alpha}^2\left( \frac{1}{\sqrt{\eta \widehat{t}_{\epsilon, \alpha}}}, \mathcal{H} \right)\\
    &= \frac{R^2}{c^\prime}\widehat{\epsilon}_{n, \alpha}^{2(1+\alpha)}.
\end{align*}
Then for the event $A$ from Corollary \ref{corollary_for_assumption}, by standard concentration results on linear and quadratic sums of Gaussian random variables (see, e.g., \cite[Lemma 1]{laurent2000adaptive}), 
\begin{align}
    \mathbb{P}_{\varepsilon} \left(\Sigma_1 > \frac{\Delta_{\overline{t}_{\epsilon, \alpha}, \alpha}}{2} \mid A \right) &\leq \exp \left[ - \frac{\Delta_{\overline{t}_{\epsilon, \alpha}, \alpha}^2}{16(\lVert a(\overline{t}_{\epsilon, \alpha}) \rVert^2 + \frac{\Delta_{\overline{t}_{\epsilon, \alpha}, \alpha}}{2} \lVert a(\overline{t}_{\epsilon, \alpha}) \rVert_{\infty})} \right],\\
    \mathbb{P}_{\varepsilon} \left( \Sigma_2 > \frac{\Delta_{\overline{t}_{\epsilon, \alpha}, \alpha}}{2} \right) &\leq \exp \left[ - \frac{n \Delta_{\overline{t}_{\epsilon, \alpha}, \alpha}^2}{32 \sigma^2 B_{\alpha}^2(\overline{t}_{\epsilon, \alpha})} \right],
\end{align}
where $a_i(\overline{t}_{\epsilon, \alpha}) = \frac{\sigma^2}{n}\widehat{\mu}_i^{\alpha}(1 - \gamma_i^{(\overline{t}_{\epsilon, \alpha})})^2, \ i \in [n]$.

\vspace{0.2cm}

In what follows, we simplify the bounds above.

\vspace{0.2cm}

Firstly, recall that $B = 1$, which implies $\widehat{\mu}_1 \leq 1$, and  $\lVert a(\overline{t}_{\epsilon, \alpha}) \rVert_{\infty} \leq \frac{\sigma^2}{n}$, and 
\begin{align*}
    \frac{1}{2}\Delta_{\overline{t}_{\epsilon, \alpha}, \alpha} \leq \frac{3}{4} \widetilde{V}_{\alpha}(\overline{t}_{\epsilon, \alpha}) \leq \frac{3}{4}\widetilde{V}_{\alpha}(\widehat{t}_{\epsilon, \alpha}) &\leq \frac{3}{4R^2}\sigma^2 \eta \widehat{t}_{\epsilon, \alpha} \widehat{\mathcal{R}}_{n, \alpha}^2\left( \frac{1}{\sqrt{\eta \widehat{t}_{\epsilon, \alpha}}}, \mathcal{H} \right) \\ &= 3 R^2 \widehat{\epsilon}_{n, \alpha}^{2(1+\alpha)}.    
\end{align*}
Secondly, we will upper bound the Euclidean norm of $a(\overline{t}_{\epsilon, \alpha})$. Recall Corollary \ref{corollary_for_assumption} with $a = \frac{R^2}{\sigma^2}n \widehat{\epsilon}_{n, \alpha}^{2(1+\alpha)}$ and $\theta = 2$, the definition of the smoothed statistical dimension $d_{n, \alpha} = \min \{ j \in [n]: \widehat{\mu}_j \leq \widehat{\epsilon}_{n, \alpha}^2 \}$, and Ineq. (\ref{useful_inequalities_smooth_}): $\widehat{\epsilon}_{n, \alpha}^{2(1+\alpha)} \geq \frac{\sigma^2 \sum_{i=1}^{d_{n, \alpha}}\widehat{\mu}_i^{\alpha}}{4 R^2 n}$, which implies
\begin{align*}
    \lVert a(\overline{t}_{\epsilon, \alpha}) \rVert^2 &= \frac{\sigma^4}{n^2}\sum_{i=1}^n \widehat{\mu}_i^{2 \alpha}\left(1 - \gamma_i^{(\overline{t}_{\epsilon, \alpha})}\right)^4 \leq \frac{\sigma^4}{n^2} \left[ \sum_{i=1}^{d_{n, \alpha}} \widehat{\mu}_i^{\alpha} + \sum_{i=d_{n, \alpha} + 1}^n \widehat{\mu}_i^{2\alpha} \right]\\
    &\leq \frac{\sigma^4}{n^2} \left[ \frac{4 n R^2 \widehat{\epsilon}_{n, \alpha}^{2(1 + \alpha)}}{\sigma^2} + 3\mathcal{M}\sum_{i=1}^{d_{n, \alpha}}\widehat{\mu}_i^{2\alpha} + \frac{7 (3 \mathcal{M} + 1) R^2}{3\sigma^2}\widehat{\epsilon}_{n, \alpha}^{2(1+\alpha)} \right] \\
    &\leq \frac{\sigma^2 R^2}{n} \left[ 4 + 12 \mathcal{M} + 3(3\mathcal{M} + 1) \right] \widehat{\epsilon}_{n, \alpha}^{2(1+\alpha)}.
\end{align*}
Finally, using the upper bound $B_{\alpha}^2(\overline{t}_{\epsilon, \alpha}) \leq \frac{R^2}{(\eta \overline{t}_{\epsilon, \alpha})^{1+\alpha}} \leq R^2 (c^{\prime})^2 \widehat{\epsilon}_{n, \alpha}^{2(1+\alpha)}$ for all $\alpha \in [0, 1]$ and the fact that $\mathbb{P}_{\varepsilon}(A) = \mathbb{P}_{X_1, \ldots, X_n}\left( \mathbb{I}(A) \right) = \mathbb{P}_{X_1, \ldots, X_n}(A)$ for the event $A$ from Corollary \ref{corollary_for_assumption}, one gets
\begin{align*}
    \mathbb{P}_{\varepsilon}\left( \Sigma_1 > \frac{\Delta_{\overline{t}_{\epsilon, \alpha}, \alpha}}{2} \right) &\leq \mathbb{P}_{\varepsilon}\left( \Sigma_1 > \frac{\Delta_{\overline{t}_{\epsilon, \alpha}, \alpha}}{2} \mid A \right) + \mathbb{P}_{X_1, \ldots, X_n}\left( A^c \right),\\
    \mathbb{P}_{\varepsilon} \left( \tau_{\alpha} > \overline{t}_{\epsilon, \alpha} \right) &\leq 5 \ \exp \left[ - c_1 \frac{R^2}{\sigma^2}n \widehat{\epsilon}_{n, \alpha}^{2(1 + \alpha)} \right],
\end{align*}
for some positive numeric $c_1 > 0$ that depends only on $\mathcal{M}$. 

\end{proof}

What follows is the second deviation inequality for $\tau_{\alpha}$ that will be further used in Lemma \ref{bias_deviation} to control the bias term.

\begin{lemma} \label{lemma_deviation_bound_2}
Recall Definition \ref{definition_additional_stopping_times} of $\widetilde{t}_{\epsilon, \alpha}$, then under Assumptions \ref{a1}, \ref{a2}, \ref{additional_assumption_gd_krr}, and \ref{sufficient_smoothing},
\begin{equation} \label{deviation_bound_2}
\mathbb{P}_{\varepsilon} \left( \tau_{\alpha} < \widetilde{t}_{\epsilon, \alpha} \right) \leq 5 \ \exp \left[- c_2 \frac{R^2}{\sigma^2} n \widehat{\epsilon}_{n, \alpha}^{2(1+\alpha)} \right]    
\end{equation}
for a positive constant $c_2$ that depends only on $\mathcal{M}$.

\end{lemma}

\begin{proof}[Proof of Lemma \ref{lemma_deviation_bound_2}]
Set $\kappa_{\alpha} \coloneqq \sigma^2 \textnormal{tr}K_n^{\alpha} / n$. Note that $\widetilde{t}_{\epsilon, \alpha} \leq t_{\alpha}^*$ by construction. 

Further, for all $t \leq t_{\alpha}^*$, due to the monotonicity of $R_{\alpha, t}$,
\begin{align*}
    \mathbb{P}_{\varepsilon}\Big( \tau_{\alpha} < t \Big) &= \mathbb{P}_{\varepsilon}\Big( R_{\alpha, t} - \mathbb{E}_{\varepsilon}R_{\alpha, t} \leq - (\mathbb{E}_{\varepsilon}R_{\alpha, t} - \kappa_{\alpha}) \Big)\\
    &\leq \mathbb{P}_{\varepsilon}\bigg( \underbrace{\frac{\sigma^2}{n} \sum_{i=1}^n \widehat{\mu}_i^{\alpha} (1 - \gamma_i^{(t)})^2 \Big( \frac{\varepsilon_i^2}{\sigma^2} - 1 \Big)}_{\Sigma_1} \leq - \frac{\mathbb{E}_{\varepsilon}R_{\alpha, t} - \kappa_{\alpha}}{2}\bigg)\\
    &+ \mathbb{P}_{\varepsilon} \bigg( \underbrace{\frac{2}{n}\sum_{i=1}^n \widehat{\mu}_i^{\alpha} (1 - \gamma_i^{(t)})^2 G_i^* \varepsilon_i}_{\Sigma_2} \leq - \frac{\mathbb{E}_{\varepsilon}R_{\alpha, t} - \kappa_{\alpha}}{2} \bigg).
\end{align*}

Consider $\Delta_{t, \alpha} \coloneqq \mathbb{E}_{\varepsilon}R_{\alpha, t} - \kappa_{\alpha} = B_{\alpha}^2(t) + V_{\alpha}(t) - 2 \widetilde{V}_{\alpha}(t)$. At $t = \widetilde{t}_{\epsilon, \alpha}$, we have $B_{\alpha}^2(t) = 3 \widetilde{V}_{\alpha}(t)$, thus
\begin{equation*}
    \Delta_{\widetilde{t}_{\epsilon, \alpha}, \alpha} \geq \widetilde{V}_{\alpha}(\widetilde{t}_{\epsilon, \alpha}).
\end{equation*}

Then for the event $A$ from Corollary \ref{corollary_for_assumption}, by standard concentration results on linear and quadratic sums of Gaussian random variables (see, e.g., \cite[Lemma 1]{laurent2000adaptive}),
\begin{align} \label{concentration_bounds}
    \begin{split}
    \mathbb{P}_{\varepsilon} \left( \Sigma_1 \leq - \frac{\Delta_{\widetilde{t}_{\epsilon, \alpha}, \alpha}}{2} \mid A \right) &\leq \exp \left[ - \frac{\widetilde{V}_{\alpha}^2(\widetilde{t}_{\epsilon, \alpha})}{16 \lVert a(\widetilde{t}_{\epsilon, \alpha}) \rVert^2} \right], \\
    \mathbb{P}_{\varepsilon} \left( \Sigma_2 \leq - \frac{\Delta_{\widetilde{t}_{\epsilon, \alpha}, \alpha}}{2} \right) &\leq \exp \left[ - \frac{- n \widetilde{V}_{\alpha}^2(\widetilde{t}_{\epsilon, \alpha})}{32 \sigma^2 B_{\alpha}^2(\widetilde{t}_{\epsilon, \alpha})} \right],
    \end{split}
\end{align}
where $a_i(\widetilde{t}_{\epsilon, \alpha}) = \frac{\sigma^2}{n} \widehat{\mu}_i^{\alpha}(1 - \gamma_i^{(\widetilde{t}_{\epsilon, \alpha})}), \ i \in [n]$. 

\vspace{0.2cm}

In what follows, we simplify the bounds above.

\vspace{0.2cm}

First, we deal with the Euclidean norm of $a_i(\widetilde{t}_{\epsilon, \alpha}), \ i \in [n]$. By $\widehat{\mu}_1 \leq 1$ and Corollary \ref{corollary_for_assumption} with $a = \frac{R^2}{\sigma^2}n \widehat{\epsilon}_{n, \alpha}^{2(1 + \alpha)}$ and $\theta = 2$, and Ineq. (\ref{useful_inequalities_smooth_}), it gives us
\begin{align} \label{a_tilde_t_bound}
    \begin{split}
    \lVert a(\widetilde{t}_{\epsilon, \alpha}) \rVert^2 = \frac{\sigma^4}{n^2} \sum_{i=1}^n \widehat{\mu}_i^{2 \alpha} (1 - \gamma_i^{(\widetilde{t}_{\epsilon, \alpha})})^4 &\leq \frac{\sigma^4}{n^2} \left[ \sum_{i=1}^{d_{n, \alpha}} \widehat{\mu}_i^{\alpha} + \sum_{i=d_{n, \alpha} + 1}^n \widehat{\mu}_i^{2\alpha} \right]\\
    &\leq \left[ 4 + 12 \mathcal{M} + 3(3\mathcal{M} + 1) \right] \frac{\sigma^2}{n}R^2  \widehat{\epsilon}_{n, \alpha}^{2(1+\alpha)}.
    \end{split}
\end{align}

Recall that $\eta \widetilde{t}_{\epsilon, \alpha} = \frac{\eta \widehat{t}_{\epsilon, \alpha}}{c^{\prime \prime}}$ for $c^{\prime \prime} \geq 1$. Therefore, it is sufficient to lower bound $\widetilde{V}_{\alpha}(\widetilde{t}_{\epsilon, \alpha})$ as follows.
\begin{align*}
    \widetilde{V}_{\alpha}(\widetilde{t}_{\epsilon, \alpha}) \geq \frac{\sigma^2}{2n} \sum_{i=1}^n \widehat{\mu}_i^{\alpha} \min \{ 1, \frac{\eta \widehat{t}_{\epsilon, \alpha}}{c^{\prime \prime}} \widehat{\mu}_i \} &= \frac{\sigma^2 \eta \widehat{t}_{\epsilon, \alpha}}{2 n c^{\prime \prime}} \sum_{i=1}^n \widehat{\mu}_i^{\alpha} \min \left\{ \frac{c^{\prime \prime}}{\eta \widehat{t}_{\epsilon, \alpha}}, \widehat{\mu}_i \right\} \\ 
    &\geq \frac{\sigma^2 \eta \widehat{t}_{\epsilon, \alpha}}{2 R^2 c^{\prime \prime}}\widehat{\mathcal{R}}_{n, \alpha}^2 \left( \frac{1}{\sqrt{\eta \widehat{t}_{\epsilon, \alpha}}}, \mathcal{H} \right)\\
    &= \frac{2 R^2}{c^{\prime \prime}}\widehat{\epsilon}_{n, \alpha}^{2(1 + \alpha)}. 
\end{align*}
By using the bound $B_{\alpha}^2(\widetilde{t}_{\epsilon, \alpha}) \leq \frac{R^2}{(\eta \widetilde{t}_{\epsilon, \alpha})^{1+\alpha}} \leq R^2 (c^{\prime \prime})^2 \widehat{\epsilon}_{n, \alpha}^{2(1+\alpha)}$, inserting this expression with (\ref{a_tilde_t_bound}) into (\ref{concentration_bounds}), and using the fact that $\mathbb{P}_{\varepsilon}(A) = \mathbb{P}_{X_1, \ldots, X_n}\left( \mathbb{I}(A) \right) = \mathbb{P}_{X_1, \ldots, X_n}(A)$ for the event $A$ from Corollary \ref{corollary_for_assumption}, we have
\begin{align*}
    \mathbb{P}_{\varepsilon}\left( \Sigma_1 \leq -\frac{\Delta_{\widetilde{t}_{\epsilon, \alpha}, \alpha}}{2} \right) &\leq \mathbb{P}_{\varepsilon}\left( \Sigma_1 \leq -\frac{\Delta_{\widetilde{t}_{\epsilon, \alpha}, \alpha}}{2} \mid A \right) + \mathbb{P}_{X_1, \ldots, X_n}\left( A^c \right),\\
    \mathbb{P}_{\varepsilon} \left( \tau_{\alpha} < \widetilde{t}_{\epsilon, \alpha} \right) &\leq 5 \exp \left[-c_2 \frac{R^2}{\sigma^2}n \widehat{\epsilon}_{n, \alpha}^{2(1+\alpha)} \right],
\end{align*}
where $c_2$ depends only on $\mathcal{M}$.
\end{proof}

\subsection{Bounding the stochastic part of the variance term at $\tau_{\alpha}$}

\begin{lemma} \label{variance_deviation}
Under Assumptions \ref{a1}, \ref{a2}, \ref{additional_assumption_gd_krr}, and \ref{sufficient_smoothing}, for any regular kernel, the stochastic part of the variance at $\tau_{\alpha}$ is bounded as follows.
\begin{equation*}
v(\tau_{\alpha}) \leq 8 (1 + C) R^2 \widehat{\epsilon}_{n, \alpha}^2
\end{equation*}
with probability at least $1 - 6 \exp \Big[ -c_1 n \frac{R^2}{\sigma^2} \widehat{\epsilon}_{n, \alpha}^{2(1+\alpha)} \Big]$, where a constant $c_1$ depends only on $\mathcal{M}$.
\end{lemma}

\begin{proof}[Proof of Lemma \ref{variance_deviation}]

$\mathbb{P}_{\varepsilon} \left( \tau_{\alpha} > \overline{t}_{\epsilon, \alpha} \right) \leq 5 \exp \left[ - c_1 \frac{R^2}{\sigma^2} n\widehat{\epsilon}_{n, \alpha}^{2(1 + \alpha)} \right]$ due to Lemma \ref{lemma_deviation_bound_1}. Therefore, thanks to the monotonicity of $\gamma_i^{(t)}$ in $t$, with probability at least $1 - 5 \ \exp \left[ - c_1 \frac{R^2}{\sigma^2} n \widehat{\epsilon}_{n, \alpha}^{2(1 + \alpha)} \right]$, $v(\tau_{\alpha}) \leq v(\overline{t}_{\epsilon, \alpha})$.

After that, due to the concentration inequality (\ref{hanson-wright_in_t}), 
\begin{equation*}
    \mathbb{P}_{\varepsilon} \Big( |v(\overline{t}_{\epsilon, \alpha}) - V(\overline{t}_{\epsilon, \alpha})| \geq \delta \Big) \leq 2 \exp \left[ - \frac{c n \delta}{\sigma^2} \min \left\{ 1, \frac{R^2 \delta}{\sigma^2 \eta \overline{t}_{\epsilon,  \alpha} \widehat{\mathcal{R}}_n^2 (\frac{1}{\sqrt{\eta \overline{t}_{\epsilon, \alpha}}}, \mathcal{H})} \right\}  \right].
\end{equation*}
Now, by setting $\delta = \frac{\sigma^2 \eta \widehat{t}_{\epsilon, \alpha}}{R^2}\widehat{\mathcal{R}}_n^2 \Big( \frac{1}{\sqrt{\eta \widehat{t}_{\epsilon, \alpha}}}, \mathcal{H} \Big) \geq \frac{\sigma^2 \eta \widehat{t}_{\epsilon, \alpha}}{R^2}\widehat{\mathcal{R}}_{n, \alpha}^2 \Big( \frac{1}{\sqrt{\eta \widehat{t}_{\epsilon, \alpha}}}, \mathcal{H} \Big)$ and recalling Lemma \ref{connection_t_epsilon}, it yields 
\begin{align}
    \begin{split}
    v(\overline{t}_{\epsilon, \alpha}) &\leq V(\overline{t}_{\epsilon, \alpha}) + \delta\\
    &\leq \widetilde{V}(\widehat{t}_{\epsilon, \alpha}) + 4(1 + C) R^2 \widehat{\epsilon}_{n, \alpha}^2 \\
    &\leq \frac{\sigma^2 \eta \widehat{t}_{\epsilon, \alpha}}{R^2} \widehat{\mathcal{R}}_n^2 \left( \frac{1}{\sqrt{\eta \widehat{t}_{\epsilon, \alpha}}}, \mathcal{H} \right) + 4 (1 + C) R^2 \widehat{\epsilon}_{n, \alpha}^2 \\
    &\leq 8 (1 + C) R^2 \widehat{\epsilon}_{n, \alpha}^2
    \end{split}
\end{align}
with probability at least $1 - \exp \Big[ - c n \frac{4 R^2}{\sigma^2} \widehat{\epsilon}_{n, \alpha}^{2(1+\alpha)} \Big]$. Combining all the pieces together, we get 
\begin{equation}
    v(\tau_{\alpha}) \leq 8 (1 + C) R^2 \widehat{\epsilon}_{n, \alpha}^2
\end{equation}
with probability at least $1 - 6 \exp \Big[ - c_1 n \frac{R^2}{\sigma^2} \widehat{\epsilon}_{n, \alpha}^{2(1+\alpha)}\Big]$.
\end{proof}

\subsection{Bounding the bias term at $\tau_{\alpha}$}

\begin{lemma} \label{bias_deviation}
Under Assumptions \ref{a1}, \ref{a2}, \ref{additional_assumption_gd_krr}, and \ref{sufficient_smoothing},
\begin{equation}
    B^2(\tau_{\alpha}) \leq c^{\prime \prime}R^2 \widehat{\epsilon}_{n, \alpha}^2
\end{equation}
with probability at least $1 - 5 \exp \Big[ - c_2 \frac{R^2}{\sigma^2}n \widehat{\epsilon}_{n, \alpha}^{2(1+\alpha)} \Big]$ for a positive numeric constant $c^{\prime \prime} \geq 1$ and constant $c_2$ that depends only on $\mathcal{M}$.
\end{lemma}

\begin{proof}[Proof of Lemma \ref{bias_deviation}]
$\mathbb{P}_{\varepsilon} \left( \tau_{\alpha} < \widetilde{t}_{\epsilon, \alpha} \right) \leq 5 \exp \left[ -c_2 \frac{R^2}{\sigma^2} n \widehat{\epsilon}_{n, \alpha}^{2(1 + \alpha)} \right]$ due to Lemma \ref{lemma_deviation_bound_2}. Therefore, thanks to the monotonicity of the bias term, with probability at least $1 - 5 \exp \left[ - c_2 \frac{R^2}{\sigma^2}n\widehat{\epsilon}_{n, \alpha}^{2(1 + \alpha)} \right]$, $B^2(\tau_{\alpha}) \leq B^2(\widetilde{t}_{\epsilon, \alpha}) \leq \frac{R^2}{\eta \widetilde{t}_{\epsilon, \alpha}} = c^{\prime \prime}R^2 \widehat{\epsilon}_{n, \alpha}^2.$
\end{proof}

\section{Proof of Theorem \ref{th:3}} \label{proof_for_smoothed_norm}
From Lemmas \ref{first_lemma}, \ref{variance_deviation}, and \ref{bias_deviation}, we get 
\begin{equation}
    \lVert f^{\tau_{\alpha}} - f^* \rVert_n^2 \leq 2 c^{\prime \prime}R^2 \widehat{\epsilon}_{n, \alpha}^2 + 16(1 + C)R^2 \widehat{\epsilon}_{n, \alpha}^2
\end{equation}
with probability at least $1 - 11 \exp \left[ - c_1 \frac{R^2}{\sigma^2}n \widehat{\epsilon}_{n, \alpha}^{2(1+\alpha)} \right]$, where $c_1$ depends only on $\mathcal{M}$. Moreover, by taking the expectation in Ineq. (\ref{wo_expectation}), it yields 
\begin{equation*}
    \mathbb{E}_{\varepsilon}\lVert f^{\tau_{\alpha}} - f^* \rVert_n^2 \leq 2 \mathbb{E}_{\varepsilon}[B^2(\tau_{\alpha})] + 2 \mathbb{E}_{\varepsilon}[v(\tau_{\alpha})].
\end{equation*}
Let us upper bound $\mathbb{E}_{\varepsilon}\left[ B^2(\tau_{\alpha}) \right]$ and $\mathbb{E}_{\varepsilon}\left[ v(\tau_{\alpha}) \right]$. First, define $\widetilde{a} \coloneqq B^2(\widetilde{t}_{\epsilon, \alpha})$, thus
\begin{align}
    \begin{split}
    \mathbb{E}_{\varepsilon}\left[ B^2(\tau_{\alpha}) \right] &= \mathbb{P}_{\varepsilon}\Big( B^2(\tau_{\alpha}) > \widetilde{a} \Big) \mathbb{E}_{\varepsilon}\Big[ B^2(\tau_{\alpha}) \mid B^2(\tau_{\alpha}) > \widetilde{a} \Big]\\
    &+ \mathbb{P}_{\varepsilon}\Big( B^2(\tau_{\alpha}) \leq \widetilde{a} \Big) \mathbb{E}_{\varepsilon}\Big[ B^2(\tau_{\alpha}) \mid B^2(\tau_{\alpha}) \leq \widetilde{a} \Big].
    \end{split}
\end{align}
Defining $\delta_1 \coloneqq 5 \exp \left[ -c_2 \frac{R^2}{\sigma^2}n \widehat{\epsilon}_{n, \alpha}^{2(1+\alpha)} \right]$ from Lemma \ref{bias_deviation} and using the upper bound $B^2(t) \leq R^2$
for any $t > 0$ gives the following.
\begin{equation}
    \mathbb{E}_{\varepsilon}\left[ B^2(\tau_{\alpha}) \right] \leq R^2 \delta_1 + B^2(\widetilde{t}_{\epsilon, \alpha}) \leq R^2 \left( \delta_1 + c^{\prime \prime}\widehat{\epsilon}_{n, \alpha}^2 \right).
\end{equation}
As for $\mathbb{E}_{\varepsilon}\left[ v(\tau_{\alpha}) \right]$,
\begin{align}
    \begin{split}
        \mathbb{E}_{\varepsilon}\left[ v(\tau_{\alpha}) \right] &= \mathbb{E}_{\varepsilon} \left[ v(\tau_{\alpha}) \mathbb{I}\left\{ v(\tau_{\alpha}) \leq 8(1 + C)R^2 \widehat{\epsilon}_{n, \alpha}^2 \right\} \right] \\
        &+ \mathbb{E}_{\varepsilon} \left[ v(\tau_{\alpha}) \mathbb{I} \left\{ v(\tau_{\alpha}) > 8(1 + C)R^2 \widehat{\epsilon}_{n, \alpha}^2 \right\} \right],
    \end{split}
\end{align}
and due to Lemma \ref{variance_deviation} and Cauchy-Schwarz inequality,
\begin{align} 
    &\mathbb{E}_{\varepsilon} \left[ v(\tau_{\alpha}) \right] \leq 8(1 + C)R^2 \widehat{\epsilon}_{n, \alpha}^2 + \mathbb{E}_{\varepsilon} \Big[ v(\tau_{\alpha}) \mathbb{I}\Big\{ v(\tau_{\alpha}) > 8(1 + C)R^2 \widehat{\epsilon}_{n, \alpha}^2 \Big\} \Big] \notag  \\
    &\leq 8(1+C)R^2 \widehat{\epsilon}_{n, \alpha}^2  + \sqrt{\mathbb{E}_{\varepsilon}v^2(\tau_{\alpha})} \sqrt{\mathbb{E}_{\varepsilon} \left[ \mathbb{I} \left\{ v(\tau_{\alpha}) > 8(1 + C)R^2 \widehat{\epsilon}_{n, \alpha}^2  \right\} \right]} \label{variance_inserting_to_}.
\end{align}
Notice that $v^2(\tau_{\alpha}) \leq \frac{1}{n^2}\left[ \sum_{i=1}^n \varepsilon_i^2 \right]^2$, and 
\begin{equation}
    \mathbb{E}_{\varepsilon}\left[ v^2(\tau_{\alpha}) \right] \leq \frac{1}{n^2} \left[ \sum_{i=1}^n \mathbb{E}_{\varepsilon}\varepsilon_i^4 + 2 \sum_{i<j} \mathbb{E}_{\varepsilon} \left( \varepsilon_i^2 \varepsilon_j^2 \right) \right] \leq \frac{3 \sigma^4}{n^2}n^2 \leq 3 \sigma^4.    
\end{equation}
At the same time, thanks to Lemma \ref{variance_deviation},
\begin{equation*}
    \mathbb{E}_{\varepsilon}\left[ \mathbb{I} \left\{ v(\tau_{\alpha}) > 8(1 + C)R^2 \widehat{\epsilon}_{n, \alpha}^2 \right\} \right] \leq 6 \exp \left( - c_1 n \frac{R^2}{\sigma^2}\widehat{\epsilon}_{n, \alpha}^{2(1+\alpha)} \right).
\end{equation*}
Thus, by inserting the last two inequalities into (\ref{variance_inserting_to_}), it gives
\begin{equation*}
\mathbb{E}_{\varepsilon}\left[ v(\tau_{\alpha}) \right] \leq 8(1 + C)R^2 \widehat{\epsilon}_{n, \alpha}^2 + 5 \sigma^2 \exp \left( - c_1 n \frac{R^2}{\sigma^2}\widehat{\epsilon}_{n, \alpha}^{2(1+\alpha)} \right) .    
\end{equation*}
%
Finally, summing up all the terms together,
\begin{align*}
\mathbb{E}_{\varepsilon} \lVert f^{\tau_{\alpha}} - f^* \rVert_n^2 &\leq  \left[16(1 + C) + 2c^{\prime \prime} \right] R^2 \widehat{\epsilon}_{n, \alpha}^2\\
&+ 20 \max \{ \sigma^2, R^2 \}\exp \left( -c_1 n \frac{R^2}{\sigma^2}\widehat{\epsilon}_{n, \alpha}^{2(1+\alpha)} \right),
\end{align*}
where a constant $c_1$ depends only on $\mathcal{M}$, constant $c^{\prime \prime}$ is numeric.

\section{Proof of Theorem \ref{th:2}} \label{proof_for_change_of_norm}
%

We will use the definition of $\tau$ (\ref{tau}) with the threshold $\kappa \coloneqq \frac{r \sigma^2}{n} $ so that, due to the monotonicity of the "reduced" empirical risk $\widetilde{R}_t$,
\begin{equation*}
    \mathbb{P}_{\varepsilon}\left( \tau > t \right) = \mathbb{P}_{\varepsilon} \Big( \widetilde{R}_t - \mathbb{E}_{\varepsilon}\widetilde{R}_t > \underbrace{\kappa - \mathbb{E}_{\varepsilon}\widetilde{R}_t}_{\Delta_{t}} \Big),
\end{equation*}
where
\begin{equation} \label{eq:delta_t}
    \Delta_t = -B^2(t) - V(t) + \underbrace{\frac{2 \sigma^2}{n}\sum_{i=1}^r \gamma_i^{(t)}}_{2 \widetilde{V}(t)}.
\end{equation}
Assume that $\Delta_t \geq 0$. Remark that
\begin{equation}
    \widetilde{R}_t - \mathbb{E}_{\varepsilon}\widetilde{R}_t = \underbrace{\frac{\sigma^2}{n}\sum_{i=1}^r (1 - \gamma_{i}^{(t)})^2 \left(\frac{\varepsilon_i^2}{\sigma^2} - 1\right)}_{\Sigma_1} + \underbrace{\frac{2}{n}\sum_{i=1}^r (1 - \gamma_i^{(t)})^2 G_i^* \varepsilon_i}_{\Sigma_2}.
\end{equation}
By applying \cite[Lemma 1]{laurent2000adaptive} to $\Sigma_1$, it yields 
\begin{equation} \label{fixed_rank_Lemma}
    \mathbb{P}_{\varepsilon} \left( \Sigma_1 > \frac{\Delta_t}{2} \right) \leq \exp \left[ \frac{-\Delta_t^2 / 4}{4 (\lVert a(t) \rVert^2 + \frac{\Delta_t}{2} \lVert a(t) \rVert_{\infty})} \right],
\end{equation}
where $a_i(t) \coloneqq \frac{\sigma^2}{n} (1 - \gamma_i^{(t)})^2, \ i \in [r]$. In addition, \cite[Proposition 2.5]{wainwright2019high} gives us 
\begin{equation} \label{fixed_rank_sum}
    \mathbb{P}_{\varepsilon} \left( \Sigma_2 > \frac{\Delta_t}{2} \right) \leq \exp \left[ - \frac{n \Delta_t^2}{32 \sigma^2 B^2(t)} \right].
\end{equation}
Define a stopping time $\overline{t}_{\epsilon}$ as follows.
\begin{equation}
    \overline{t}_{\epsilon} \coloneqq \inf \left\{ t > 0: B^2(t) = \frac{1}{2} \widetilde{V}(t) \right\}.
\end{equation}
Note that $\overline{t}_{\epsilon}$ serves as an upper bound on $t^*$ and as a lower bound on $\widehat{t}_{\epsilon}$. Moreover, $\overline{t}_{\epsilon}$ satisfies the critical inequality (\ref{fixed_point_rademacher_in_t}). Therefore, due to Lemma \ref{smallest_positive_solution} and continuity of (\ref{fixed_point_rademacher_in_t}) in $t$, there is a positive numeric constant $c^\prime \geq 1$, that do not depend on $n$, such that $\frac{1}{\eta \overline{t}_{\epsilon}} = c^\prime \frac{1}{\eta \widehat{t}_{\epsilon}}$. 

In what follows, we simplify two high probability bounds (\ref{fixed_rank_Lemma}) and (\ref{fixed_rank_sum}) at $t = \overline{t}_{\epsilon}$.

\vspace{0.2cm}

Since applying \cite[Section 4.3]{raskutti2014early}, $\widehat{\epsilon}_n^2 = c \frac{r \sigma^2}{n R^2}$, one can bound $\lVert a(\overline{t}_{\epsilon}) \rVert^2$ as follows.
\begin{equation}
    \lVert a(\overline{t}_{\epsilon}) \rVert^2 = \frac{\sigma^4}{n^2}\sum_{i=1}^r (1 - \gamma_i^{(\overline{t}_{\epsilon})})^4 \leq \frac{r \sigma^4}{n^2} = \frac{R^2 \sigma^2 \widehat{\epsilon}_n^2}{c n}.
\end{equation}
Remark that in (\ref{fixed_rank_Lemma}) $\lVert a(\overline{t}_{\epsilon}) \rVert_{\infty} = \frac{\sigma^2}{n} \underset{i \in [r]}{\max} \Big[ (1 - \gamma_i^{(\overline{t}_{\epsilon})}) \Big] \leq \frac{\sigma^2}{n}$,
and
\begin{equation*}
    \frac{\Delta_{\overline{t}_{\epsilon}}}{2} \leq 
    \frac{3}{4} \widetilde{V}(\overline{t}_{\epsilon}) \leq \frac{3}{4} \widetilde{V}(\widehat{t}_{\epsilon}) \leq \frac{3}{4} \frac{\sigma^2}{R^2} \eta \widehat{t}_{\epsilon} \widehat{\mathcal{R}}_{n}^2\left(\frac{1}{\sqrt{\eta \widehat{t}_{\epsilon}}}, \mathcal{H}\right) = 3R^2 \widehat{\epsilon}_n^2.
\end{equation*}
As for a lower bound on $\Delta_{\overline{t}_{\epsilon}}$,
\begin{align*}
    \Delta_{\overline{t}_{\epsilon}} \geq \frac{1}{2} \widetilde{V}(\overline{t}_{\epsilon}) \geq \frac{\sigma^2}{4n} \sum_{i=1}^r \min \left\{1, \frac{\eta \widehat{t}_{\epsilon}}{c^\prime} \widehat{\mu}_i \right\} &= \frac{\sigma^2 \eta \widehat{t}_{\epsilon}}{4n c^\prime} \sum_{i=1}^r \min \left\{ \frac{c^\prime}{\eta \widehat{t}_{\epsilon}}, \widehat{\mu}_i \right\}\\
    &\geq \frac{R^2}{c^\prime} \widehat{\epsilon}_n^2.
\end{align*}

By knowing that $B^2(\overline{t}_{\epsilon}) \leq \frac{R^2}{\eta \overline{t}_{\epsilon}} =  c^\prime R^2 \widehat{\epsilon}_n^2$ and summing up bounds (\ref{fixed_rank_Lemma}), (\ref{fixed_rank_sum}) with $t = \overline{t}_{\epsilon}$, it yields the following.

\begin{equation}  \label{high_proba_bound_finite_rank}
    \mathbb{P}_{\varepsilon}\left( \tau > \overline{t}_{\epsilon} \right) \leq 2 \ \exp \left[ - C \frac{R^2}{\sigma^2} n \widehat{\epsilon}_n^2 \right].
\end{equation}
From \cite[Lemma 9]{raskutti2014early}, $\lVert f^{\overline{t}_{\epsilon}} \rVert_{\mathcal{H}} \leq \sqrt{7}R$ with probability at least $1 - 4 \ \textnormal{exp}\Big[ - c_3 \frac{R^2}{\sigma^2} n \widehat{\epsilon}_n^2 \Big]$. Thus, Ineq. (\ref{high_proba_bound_finite_rank}) allows to say: 
\begin{equation*}
\lVert f^{\tau} \rVert_{\mathcal{H}} \leq \sqrt{7} R \textnormal{ with probability at least } 1 - 6 \ \exp \left( -\Tilde{c}_3 \frac{R^2}{\sigma^2} n \widehat{\epsilon}_n^2 \right) \textnormal{ for } \Tilde{c}_3 > 0.
\end{equation*}
It implies that $\lVert f^{\tau} - f^* \rVert_{\mathcal{H}} \leq \lVert f^{\tau} \rVert_{\mathcal{H}} + \lVert f^* \rVert_{\mathcal{H}} \leq \left( 1 + \sqrt{7} \right) R$ 
with the same probability. Thus, according to \cite[\textnormal{Theorem 14.1}]{wainwright2019high}, for some positive numeric constants $c_1, \Tilde{c}_{4}, \Tilde{c}_{5}:$
\begin{equation*}
    \lVert f^{\tau} - f^* \rVert_2^2 \leq 2 \lVert f^{\tau} - f^* \rVert_n^2 + c_1 R^2 \epsilon_n^2
\end{equation*}
with probability (w.r.t. $\varepsilon$) at least $1 - 6 \ \textnormal{exp}\left[ - \Tilde{c}_{3}\frac{R^2}{\sigma^2} n \widehat{\epsilon}_n^2 \right]$ and with probability (w.r.t. $\{ x_i \}_{i=1}^n$) at least $1 - \Tilde{c}_{4} \ \textnormal{exp}\left[ -\Tilde{c}_{5} \frac{R^2}{\sigma^2} n \epsilon_n^2 \right]$.

Moreover, the same arguments (with $\alpha = 0$ and without Assumption \ref{sufficient_smoothing}) as in the proof of Theorem \ref{th:3}, \cite[Proposition 14.25]{wainwright2019high} and \cite[Section 4.3.1]{raskutti2014early} yield
\begin{equation}
    \lVert f^{\tau} - f^* \rVert_n^2 \leq c_u R^2 \widehat{\epsilon}_n^2 \leq \widetilde{c}_u R^2 \epsilon_n^2 \lesssim \frac{r \sigma^2}{n}
\end{equation}
with probability at least $1 - c_1 \exp \left[ - c_2 \frac{R^2}{\sigma^2}n \epsilon_n^2 \right]$.
Then by the Cauchy-Schwarz inequality,
\begin{align*}
    \mathbb{E}\lVert f^{\tau} - f^* \rVert_2^2 &= \mathbb{E}\left[ \lVert f^{\tau} - f^* \rVert_2^2 \mathbb{I}\left\{ \lVert f^{\tau} - f^* \rVert_2^2 \leq \frac{c r \sigma^2}{n} \right\} \right] +\\
    &+ \mathbb{E}\left[ \lVert f^{\tau} - f^* \rVert_2^2 \mathbb{I}\left\{ \lVert f^{\tau} - f^* \rVert_2^2 > \frac{c r \sigma^2}{n} \right\} \right]\\
    &\leq \frac{cr\sigma^2}{n} + \sqrt{\mathbb{E}\lVert f^{\tau} - f^* \rVert_2^4}\sqrt{\mathbb{P}\left( \lVert f^{\tau} - f^* \rVert_2^2 > \frac{c r \sigma^2}{n} \right)}.
\end{align*}

Since $f^{\tau} = g_{\lambda(\tau)}(\Sigma_n)S_n^* Y$, where the empirical covariance operator 
\begin{align*}
    \Sigma_n &= \frac{1}{n}\sum_{i=1}^n \mathbb{K}(\cdot, x_i) \otimes \mathbb{K}(\cdot, x_i),\\
    \Sigma_n &= S_n^* S_n.
\end{align*}
and $\gamma_i^{(\tau)} = \widehat{\mu}_i g_{\lambda(\tau)}(\widehat{\mu}_i) \leq 1$, one has
\begin{equation*}
    f^* - f^{\tau} = (I - g_{\lambda(\tau)}(\Sigma_n)\Sigma_n)f^* - g_{\lambda(\tau)}(\Sigma_n)S_n^* \varepsilon,
\end{equation*}
and due to the definition of $\tau$,
\begin{align*}
    \sigma^2 = \lVert (I - g_{\lambda(\tau)}(\Sigma_n)S_n^*)Y \rVert_n^2.
\end{align*}
We know that
\begin{align*}
    \lVert f^{\tau} - f^* \rVert_2^2 &\leq \mu_1 \lVert (I - g_{\lambda(\tau)}(\Sigma_n)\Sigma_n)f^* - g_{\lambda(\tau)}(\Sigma_n)S_n^* \varepsilon \rVert_{\mathcal{H}}^2 \\
    &\leq \lVert (I - g_{\lambda(\tau)}(\Sigma_n)\Sigma_n)f^* - g_{\lambda(\tau)}(\Sigma_n)S_n^* \varepsilon \rVert_{\mathcal{H}}^2,
\end{align*}
and
\begin{align*}
    \sigma^2 &= \lVert (I - S_n g_{\lambda(\tau)}(\Sigma_n)S_n^*)S_n f^* \rVert_n^2 + \lVert (I - S_n g_{\lambda(\tau)}(\Sigma_n)S_n^*)\varepsilon \rVert_n^2 \\
    &+ \underbrace{2\langle (I - S_n g_{\lambda(\tau)}(\Sigma_n)S_n^*)S_n f^*, (I - S_n g_{\lambda(\tau)}(\Sigma_n)S_n^*)\varepsilon \rangle_n}_{\mathcal{A}_n}.
\end{align*}
Further,
\begin{align*}
\lVert (I - g_{\lambda(\tau)}(\Sigma_n)\Sigma_n)f^* &- g_{\lambda(\tau)}(\Sigma_n)S_n^* \varepsilon \rVert_{\mathcal{H}}^2 = \lVert (I - g_{\lambda(\tau)}(\Sigma_n)\Sigma_n)f^* \rVert_{\mathcal{H}}^2 \\ &+ \lVert g_{\lambda(\tau)}(\Sigma_n)S_n^* \varepsilon \rVert_{\mathcal{H}}^2 \\ &- \underbrace{2 \langle (I - g_{\lambda(\tau)}(\Sigma_n)\Sigma_n)f^*, g_{\lambda(\tau)}(\Sigma_n)S_n^* \varepsilon \rangle_{\mathcal{H}}}_{\mathcal{A}_{\mathcal{H}}}.
\end{align*}
Thus, subtracting the empirical term from the RKHS term, one gets
\begin{equation*}
    \lVert (I - g_{\lambda(\tau)}(\Sigma_n)\Sigma_n)f^* - g_{\lambda(\tau)}(\Sigma_n)S_n^* \varepsilon \rVert_{\mathcal{H}}^2 - \sigma^2 = \underbrace{-(\mathcal{A}_{\mathcal{H}} + \mathcal{A}_n)}_{\Delta \mathcal{A}} + \text{norm discrepancy},
\end{equation*}
where $\text{norm discrepancy} = \lVert g_{\lambda(\tau)}(\Sigma_n)S_n^* \varepsilon \rVert_{\mathcal{H}}^2 - \lVert (I - S_n g_{\lambda(\tau)}(\Sigma_n)S_n^*)\varepsilon \rVert_n^2$.

Firstly, $S_n g_{\lambda(\tau)}(\Sigma_n)S_n^* = K_n g_{\lambda(\tau)}(K_n)$, and
\begin{equation*}
    \lVert g_{\lambda(\tau)}(\Sigma_n)S_n^* \varepsilon \rVert_{\mathcal{H}}^2 = \frac{1}{n}\varepsilon^\top K_n^2 [g_{\lambda(\tau)}(K_n)]^2 \varepsilon.
\end{equation*}
Secondly,
\begin{align*}
   \Delta \mathcal{A} &= -2 \langle (I - g_{\lambda(\tau)}(\Sigma_n)\Sigma_n)f^*, g_{\lambda(\tau)}(\Sigma_n)S_n^* \varepsilon \rangle_{\mathcal{H}} \\ 
   &- 2 \langle (I - S_n g_{\lambda(\tau)}(\Sigma_n)S_n^*)S_n f^*, (I - S_n g_{\lambda(\tau)}(\Sigma_n)S_n^*)\varepsilon \rangle_n.
\end{align*}
Thirdly, since $\langle h, S_n^* \mathbf{l} \rangle_{\mathcal{H}} = \langle S_n h, \mathbf{l} \rangle_n$ for any $h \in \mathcal{H}$ and $\mathbf{l} \in \mathbb{R}^n$, and $\langle h, h \rangle_{\mathcal{H}} = \langle S_n h, S_n h \rangle_n$, one gets
\begin{align*}
    \langle (I - g_{\lambda(\tau)}(\Sigma_n)\Sigma_n)f^*, & g_{\lambda(\tau)}(\Sigma_n)S_n^* \varepsilon \rangle_{\mathcal{H}} =  \langle S_n (I - g_{\lambda(\tau)}(\Sigma_n)\Sigma_n)f^*, S_n g_{\lambda(\tau)}(\Sigma_n)S_n^* \varepsilon \rangle_n \\ 
    &= \langle \underbrace{(I - K_n g_{\lambda(\tau)}(K_n))F^*}_{\Tilde{b}_{\lambda(\tau)}^2}, K_n g_{\lambda(\tau)}(K_n)\varepsilon \rangle_n \\
    &\leq \lVert \Tilde{b}_{\lambda(\tau)}^2 \rVert_n \lVert \varepsilon \rVert_n \\
    &\leq  R \lVert \varepsilon \rVert_n.
\end{align*}
Fourthly,
\begin{equation*}
    2 \langle (I - K_n g_{\lambda(\tau)}(K_n))F^*, (I - K_n g_{\lambda (\tau)}(K_n))\varepsilon \rangle_n \leq 2 \lVert \Tilde{b}_{\lambda(\tau)}^2 \rVert_n \lVert \varepsilon \rVert_n \leq 2 R \lVert \varepsilon \rVert_n.
\end{equation*}
Combining everything together and using $\gamma_i^{(\tau)} \leq 1$ for each $i$,
\begin{equation*}
    \lVert (I - g_{\lambda(\tau)}(\Sigma_n)\Sigma_n)f^* - g_{\lambda(\tau)}(\Sigma_n)S_n^* \varepsilon \rVert_{\mathcal{H}}^2 - \sigma^2 \leq \lVert \varepsilon \rVert_n^2 + 4 R \lVert \varepsilon \rVert_n.
\end{equation*}
The last equation implies that
\begin{align*}
    \lVert f^{\tau} - f^* \rVert_2^2 &\leq \sigma^2 + \lVert \varepsilon \rVert_n^2 + 4R \lVert \varepsilon \rVert_n, \\
    \mathbb{E}\lVert f^{\tau} - f^* \rVert_2^4 &\leq 3 \sigma^4 + 16R^2 \sigma^2 + \mathbb{E}\lVert \varepsilon \rVert_n^4 + 8R\mathbb{E}\lVert \varepsilon \rVert_n^3 + 8R\sigma^2 \mathbb{E}\lVert \varepsilon \rVert_n\\
    &\leq 6 \sigma^4 + 24 R \sigma^3 + 16 R^2 \sigma^2.
\end{align*}
As a consequence of the last inequality,
\begin{equation*}
    \mathbb{E}\lVert f^{\tau} - f^* \rVert_2^2 \leq \frac{\widetilde{c} r \sigma^2}{n} + C(\sigma, R)\exp(-cr).
\end{equation*}

\section{Auxiliary results} \label{auxiliary}

\begin{lemma}{\cite[Appendix D]{raskutti2014early}} \label{smallest_positive_solution}
Under Assumptions \ref{a1} and \ref{a2}, for any $\alpha \in [0, 1]$, the function $\epsilon \mapsto \frac{\widehat{\mathcal{R}}_{n, \alpha}(\epsilon, \mathcal{H})}{\epsilon}$ is non-increasing (as a function of $\epsilon$) on the interval $(0, +\infty)$, and consequently, for any numeric constant $c > 0$, 
\begin{equation} \label{aux_f_p_ineq}
    \frac{\widehat{\mathcal{R}}_{n, \alpha}(\epsilon, \mathcal{H})}{\epsilon} \leq c \frac{R^2}{\sigma}\epsilon^{1+\alpha}
\end{equation}
has a smallest positive solution. In addition to that, $\widehat{\epsilon}_{n, \alpha}$ (\ref{empirical_rademacher_complexity_def}) exists, is unique, and satisfies equality in Eq. (\ref{aux_f_p_ineq}).
\end{lemma}

\begin{lemma} \label{connection_t_epsilon}
Under Assumptions \ref{a1}, \ref{a2}, \ref{additional_assumption_gd_krr}, any regular kernel and $\widehat{t}_{\epsilon, \alpha}$ from Definition \ref{t_epsilon_alpha_def} satisfy 
\begin{equation}
    \frac{\sigma^2 \eta \widehat{t}_{\epsilon, \alpha}}{4 R^2} \widehat{\mathcal{R}}_n^2 \left( \frac{1}{\sqrt{\eta \widehat{t}_{\epsilon, \alpha}}}, \mathcal{H} \right) \leq \frac{(1 + C) R^2}{\eta \widehat{t}_{\epsilon, \alpha}}.
\end{equation}

Thus, $\widehat{t}_{\epsilon, \alpha}$ provides a smallest positive solution to the non-smooth version of the critical inequality.
\end{lemma}

\begin{proof}[Proof of Lemma \ref{connection_t_epsilon}]
First, we recall that $\frac{\sigma^2 \eta \widehat{t}_{\epsilon, \alpha}}{4R^2} \widehat{\mathcal{R}}_{n, \alpha}^2 \left( \frac{1}{\sqrt{\eta \widehat{t}_{\epsilon, \alpha}}}, \mathcal{H} \right) = R^2 \widehat{\epsilon}_{n, \alpha}^{2(1+\alpha)}$. Then for $d_{n, \alpha} = \min \{ j \in [n]: \widehat{\mu}_j \leq \widehat{\epsilon}_{n, \alpha}^2 \}$,
\begin{align} \label{auxil_equations}
    \begin{split}
    \frac{\sigma^2 \eta \widehat{t}_{\epsilon, \alpha}}{4 R^2} \widehat{\mathcal{R}}_{n, \alpha}^2 \left( \frac{1}{\sqrt{\eta \widehat{t}_{\epsilon, \alpha}}}, \mathcal{H} \right) &= \frac{\sigma^2}{4 n \widehat{\epsilon}_{n, \alpha}^2} \sum_{i=1}^n \widehat{\mu}_i^{\alpha} \min \{ \widehat{\mu}_i, \widehat{\epsilon}_{n, \alpha}^2 \} \\
    &= \frac{\sigma^2}{4 n \widehat{\epsilon}_{n, \alpha}^2}\left[ \widehat{\epsilon}_{n, \alpha}^2 \sum_{i=1}^{d_{n, \alpha}} \widehat{\mu}_i^{\alpha} + \sum_{i=d_{n, \alpha} + 1}^n \widehat{\mu}_i^{1+\alpha} \right] \\
    &= R^2 \widehat{\epsilon}_{n, \alpha}^{2(1+\alpha)}.
    \end{split}
\end{align}
The last two lines of \eqref{auxil_equations} yield $\frac{\sigma^2}{4 n \widehat{\epsilon}_{n, \alpha}^2} = \frac{R^2 \widehat{\epsilon}_{n, \alpha}^{2(1+\alpha)}}{\widehat{\epsilon}_{n, \alpha}^2 \sum_{i=1}^{d_{n, \alpha}} \widehat{\mu}_i^{\alpha} + \sum_{i=d_{n, \alpha} + 1}^n \widehat{\mu}_i^{1+\alpha}}$.

Second, consider the left-hand part of the non-smooth version of the critical inequality (\ref{critical_inequality_in_t}) at $t = \widehat{t}_{\epsilon, \alpha}$:
\begin{align} \label{useful_lines_bounding_}
    \begin{split}
    \frac{\sigma^2 \eta \widehat{t}_{\epsilon, \alpha}}{4 R^2} \widehat{\mathcal{R}}_n^2 \left( \frac{1}{\sqrt{\eta \widehat{t}_{\epsilon, \alpha}}}, \mathcal{H} \right) &= \frac{\sigma^2}{4 n \widehat{\epsilon}_{n, \alpha}^2} \sum_{i=1}^n \min \{ \widehat{\mu}_i, \widehat{\epsilon}_{n, \alpha}^2 \}\\
    &\leq R^2 \frac{\sum_{i=1}^{d_{n, \alpha}} \widehat{\epsilon}_{n, \alpha}^{4+2\alpha} + \widehat{\epsilon}_{n, \alpha}^{2(1+\alpha)} \sum_{i = d_{n, \alpha} + 1}^n \widehat{\mu}_i}{\widehat{\epsilon}_{n, \alpha}^2 \sum_{i=1}^{d_{n, \alpha}} \widehat{\mu}_i^{\alpha} }.
    \end{split}
\end{align}
Notice that $\widehat{\mu}_i \geq \widehat{\epsilon}_{n, \alpha}^2$, and $\widehat{\mu}_i^{\alpha} \geq \widehat{\epsilon}_{n, \alpha}^{2\alpha}$, for $i \leq d_{n, \alpha}$. This implies $\sum_{i=1}^{d_{n, \alpha}} \widehat{\epsilon}_{n, \alpha}^{4+2\alpha} \leq \widehat{\epsilon}_{n, \alpha}^4 \sum_{i=1}^{d_{n, \alpha}}\widehat{\mu}_i^{\alpha}$, and also that $\sum_{i = d_{n, \alpha} + 1}^n \widehat{\mu}_i \leq C\widehat{\epsilon}_{n, \alpha}^{2(1-\alpha)}\sum_{i=1}^{d_{n, \alpha}}\widehat{\mu}_i^{\alpha}$ since the kernel is regular.
Hence, 
\begin{align*}
    \widehat{\epsilon}_{n, \alpha}^{2\alpha} \sum_{i=d_{n, \alpha} + 1}^n \widehat{\mu}_i &\leq C\widehat{\epsilon}_{n, \alpha}^2 \sum_{i=1}^{d_{n, \alpha}}\widehat{\mu}_i^{\alpha},
    \end{align*}
which leads to the desired upper bound with $\widehat{\epsilon}_{n, \alpha}^2=(\eta \widehat{t}_{\epsilon, \alpha})^{-1}$:
    \begin{equation*}
    \frac{\sigma^2 \eta \widehat{t}_{\epsilon, \alpha}}{4R^2}\widehat{\mathcal{R}}_n^2 \left( \frac{1}{\sqrt{\eta \widehat{t}_{\epsilon, \alpha}}}, \mathcal{H} \right) \leq (1 + C)R^2 \widehat{\epsilon}_{n, \alpha}^2.
\end{equation*}
\end{proof}
\section{Proof of Lemma \ref{var_est}} \label{lemma_var_est}
Let us prove the lemma only for kernel ridge regression. W.l.o.g. assume that $\eta = R = \sigma = 1$ and notice that 
\begin{align}
    \begin{split}
    \mathbb{E}_{\varepsilon} \left[ \frac{R_{t}}{1/n \sum_{i=1}^n (1 - \gamma_i^{(t)})^2} \right] = \sigma^2 + \frac{B^2(t)}{1/n \sum_{i=1}^n (1 - \gamma_i^{(t)})^2}.
    \end{split}
\end{align}
From Lemma \ref{smooth_bias_upper_bound}, $B^2(t) \leq \frac{1}{t}$. As for the denominator,
\begin{equation*}
\frac{1}{n}\sum_{i=1}^n (1 - \gamma_i^{(t)})^2 = \frac{1}{n}\sum_{i=1}^n \frac{1}{(1 + t\widehat{\mu}_i)^2}.  
\end{equation*}
With the parameterization $t = \frac{1}{\epsilon^2}$ and $d_{n, 0} = \min \left\{ j \in [n]: \widehat{\mu}_j \leq \widehat{\epsilon}_n^2 \right\}$, since $\gamma_i^{(t)}, i = 1, \ldots, n$, is a non-decreasing function in $t$,
\begin{align*}
    \frac{B^2(t)}{1/n \sum_{i=1}^n (1 - \gamma_i^{(t)})^2} &\leq \frac{1}{\frac{1}{n \widehat{\epsilon}_n^2}\sum_{i=1}^n \left( \frac{\widehat{\epsilon}_n^2}{ \widehat{\mu}_i + \widehat{\epsilon}_n^2}\right)^2}\\
    &\leq \frac{1}{\frac{n-d_{n, 0}}{4n\widehat{\epsilon}_n^2}}.
\end{align*}
From \cite[Section 2.3]{yang2017randomized}, $d_{n, 0} = c n \widehat{\epsilon}_n^2$, which implies
\begin{equation*}
    \frac{B^2(t)}{1/n \sum_{i=1}^n (1 - \gamma_i^{(t)})^2} \leq \frac{4\widehat{\epsilon}_n^2}{1 - c\widehat{\epsilon}_n^2} \to 0.
\end{equation*}

\begin{acks}[Acknowledgments]
The authors would like to thank the anonymous referees, an Associate
Editor and the Editor for their constructive comments that improved the
quality of this paper.
\end{acks}

\bibliographystyle{imsart-number} 
\bibliography{vtex-soft-texsupport.ims_cosponsored-ejs-735f6cf/ejs-sample}       

\begin{thebibliography}{43}

\bibitem{akaike1998information}
\begin{bincollection}[author]
\bauthor{\bsnm{Akaike},~\bfnm{Hirotogu}\binits{H.}}
(\byear{1998}).
\btitle{Information theory and an extension of the maximum likelihood principle}.
In \bbooktitle{Selected papers of hirotugu akaike}
\bpages{199--213}.
\bpublisher{Springer}.
\end{bincollection}
\endbibitem

\bibitem{2015arXiv151005684A}
\begin{barticle}[author]
\bauthor{\bsnm{{Angles}},~\bfnm{Tomas}\binits{T.}}, \bauthor{\bsnm{{Camoriano}},~\bfnm{Raffaello}\binits{R.}}, \bauthor{\bsnm{{Rudi}},~\bfnm{Alessandro}\binits{A.}} \AND \bauthor{\bsnm{{Rosasco}},~\bfnm{Lorenzo}\binits{L.}}
(\byear{2015}).
\btitle{{NYTRO: When Subsampling Meets Early Stopping}}.
\bjournal{arXiv e-prints}
\bpages{arXiv:1510.05684}.
\end{barticle}
\endbibitem

\bibitem{arlot2010survey}
\begin{barticle}[author]
\bauthor{\bsnm{Arlot},~\bfnm{Sylvain}\binits{S.}}, \bauthor{\bsnm{Celisse},~\bfnm{Alain}\binits{A.}} \betal{et~al.}
(\byear{2010}).
\btitle{A survey of cross-validation procedures for model selection}.
\bjournal{Statistics surveys}
\bvolume{4}
\bpages{40--79}.
\end{barticle}
\endbibitem

\bibitem{aronszajn1950theory}
\begin{barticle}[author]
\bauthor{\bsnm{Aronszajn},~\bfnm{Nachman}\binits{N.}}
(\byear{1950}).
\btitle{Theory of reproducing kernels}.
\bjournal{Transactions of the American mathematical society}
\bvolume{68}
\bpages{337--404}.
\end{barticle}
\endbibitem

\bibitem{bartlett2005local}
\begin{barticle}[author]
\bauthor{\bsnm{Bartlett},~\bfnm{Peter~L}\binits{P.~L.}}, \bauthor{\bsnm{Bousquet},~\bfnm{Olivier}\binits{O.}}, \bauthor{\bsnm{Mendelson},~\bfnm{Shahar}\binits{S.}} \betal{et~al.}
(\byear{2005}).
\btitle{Local rademacher complexities}.
\bjournal{The Annals of Statistics}
\bvolume{33}
\bpages{1497--1537}.
\end{barticle}
\endbibitem

\bibitem{bartlett2007adaboost}
\begin{barticle}[author]
\bauthor{\bsnm{Bartlett},~\bfnm{Peter~L}\binits{P.~L.}} \AND \bauthor{\bsnm{Traskin},~\bfnm{Mikhail}\binits{M.}}
(\byear{2007}).
\btitle{Adaboost is consistent}.
\bjournal{Journal of Machine Learning Research}
\bvolume{8}
\bpages{2347--2368}.
\end{barticle}
\endbibitem

\bibitem{bauer2007regularization}
\begin{barticle}[author]
\bauthor{\bsnm{Bauer},~\bfnm{Frank}\binits{F.}}, \bauthor{\bsnm{Pereverzev},~\bfnm{Sergei}\binits{S.}} \AND \bauthor{\bsnm{Rosasco},~\bfnm{Lorenzo}\binits{L.}}
(\byear{2007}).
\btitle{On regularization algorithms in learning theory}.
\bjournal{Journal of complexity}
\bvolume{23}
\bpages{52--72}.
\end{barticle}
\endbibitem

\bibitem{berlinet2011reproducing}
\begin{bbook}[author]
\bauthor{\bsnm{Berlinet},~\bfnm{Alain}\binits{A.}} \AND \bauthor{\bsnm{Thomas-Agnan},~\bfnm{Christine}\binits{C.}}
(\byear{2011}).
\btitle{Reproducing kernel Hilbert spaces in probability and statistics}.
\bpublisher{Springer Science \& Business Media}.
\end{bbook}
\endbibitem

\bibitem{blanchard2018optimal}
\begin{barticle}[author]
\bauthor{\bsnm{Blanchard},~\bfnm{Gilles}\binits{G.}}, \bauthor{\bsnm{Hoffmann},~\bfnm{Marc}\binits{M.}} \AND \bauthor{\bsnm{Rei{\ss}},~\bfnm{Markus}\binits{M.}}
(\byear{2018}).
\btitle{Optimal adaptation for early stopping in statistical inverse problems}.
\bjournal{SIAM/ASA Journal on Uncertainty Quantification}
\bvolume{6}
\bpages{1043--1075}.
\end{barticle}
\endbibitem

\bibitem{blanchard2018early}
\begin{barticle}[author]
\bauthor{\bsnm{Blanchard},~\bfnm{Gilles}\binits{G.}}, \bauthor{\bsnm{Hoffmann},~\bfnm{Marc}\binits{M.}}, \bauthor{\bsnm{Rei{\ss}},~\bfnm{Markus}\binits{M.}} \betal{et~al.}
(\byear{2018}).
\btitle{Early stopping for statistical inverse problems via truncated SVD estimation}.
\bjournal{Electronic Journal of Statistics}
\bvolume{12}
\bpages{3204--3231}.
\end{barticle}
\endbibitem

\bibitem{blanchard2016convergence}
\begin{barticle}[author]
\bauthor{\bsnm{Blanchard},~\bfnm{Gilles}\binits{G.}} \AND \bauthor{\bsnm{Kr{\"a}mer},~\bfnm{Nicole}\binits{N.}}
(\byear{2016}).
\btitle{Convergence rates of kernel conjugate gradient for random design regression}.
\bjournal{Analysis and Applications}
\bvolume{14}
\bpages{763--794}.
\end{barticle}
\endbibitem

\bibitem{blanchard2010conjugate}
\begin{barticle}[author]
\bauthor{\bsnm{Blanchard},~\bfnm{Gilles}\binits{G.}} \AND \bauthor{\bsnm{Math{\'e}},~\bfnm{Peter}\binits{P.}}
(\byear{2010}).
\btitle{Conjugate gradient regularization under general smoothness and noise assumptions}.
\bjournal{Journal of Inverse and Ill-posed Problems}
\bvolume{18}
\bpages{701--726}.
\end{barticle}
\endbibitem

\bibitem{blanchard2012discrepancy}
\begin{barticle}[author]
\bauthor{\bsnm{Blanchard},~\bfnm{Gilles}\binits{G.}} \AND \bauthor{\bsnm{Math{\'e}},~\bfnm{Peter}\binits{P.}}
(\byear{2012}).
\btitle{Discrepancy principle for statistical inverse problems with application to conjugate gradient iteration}.
\bjournal{Inverse problems}
\bvolume{28}
\bpages{115011}.
\end{barticle}
\endbibitem

\bibitem{buhlmann2003boosting}
\begin{barticle}[author]
\bauthor{\bsnm{B{\"u}hlmann},~\bfnm{Peter}\binits{P.}} \AND \bauthor{\bsnm{Yu},~\bfnm{Bin}\binits{B.}}
(\byear{2003}).
\btitle{Boosting with the L 2 loss: regression and classification}.
\bjournal{Journal of the American Statistical Association}
\bvolume{98}
\bpages{324--339}.
\end{barticle}
\endbibitem

\bibitem{article}
\begin{barticle}[author]
\bauthor{\bsnm{Caponnetto},~\bfnm{Andrea}\binits{A.}}
(\byear{2006}).
\btitle{Optimal Rates for Regularization Operators in Learning Theory}.
\end{barticle}
\endbibitem

\bibitem{caponnetto2010cross}
\begin{barticle}[author]
\bauthor{\bsnm{Caponnetto},~\bfnm{Andrea}\binits{A.}} \AND \bauthor{\bsnm{Yao},~\bfnm{Yuan}\binits{Y.}}
(\byear{2010}).
\btitle{Cross-validation based adaptation for regularization operators in learning theory}.
\bjournal{Analysis and Applications}
\bvolume{8}
\bpages{161--183}.
\end{barticle}
\endbibitem

\bibitem{cavalier2002oracle}
\begin{barticle}[author]
\bauthor{\bsnm{Cavalier},~\bfnm{Laurent}\binits{L.}}, \bauthor{\bsnm{Golubev},~\bfnm{GK}\binits{G.}}, \bauthor{\bsnm{Picard},~\bfnm{Dominique}\binits{D.}}, \bauthor{\bsnm{Tsybakov},~\bfnm{AB}\binits{A.}} \betal{et~al.}
(\byear{2002}).
\btitle{Oracle inequalities for inverse problems}.
\bjournal{The Annals of Statistics}
\bvolume{30}
\bpages{843--874}.
\end{barticle}
\endbibitem

\bibitem{celisse2021analyzing}
\begin{barticle}[author]
\bauthor{\bsnm{Celisse},~\bfnm{Alain}\binits{A.}} \AND \bauthor{\bsnm{Wahl},~\bfnm{Martin}\binits{M.}}
(\byear{2021}).
\btitle{Analyzing the discrepancy principle for kernelized spectral filter learning algorithms}.
\bjournal{Journal of Machine Learning Research}
\bvolume{22}
\bpages{1--59}.
\end{barticle}
\endbibitem

\bibitem{cucker2002mathematical}
\begin{barticle}[author]
\bauthor{\bsnm{Cucker},~\bfnm{Felipe}\binits{F.}} \AND \bauthor{\bsnm{Smale},~\bfnm{Steve}\binits{S.}}
(\byear{2002}).
\btitle{On the mathematical foundations of learning}.
\bjournal{Bulletin of the American mathematical society}
\bvolume{39}
\bpages{1--49}.
\end{barticle}
\endbibitem

\bibitem{engl1996regularization}
\begin{bbook}[author]
\bauthor{\bsnm{Engl},~\bfnm{Heinz~Werner}\binits{H.~W.}}, \bauthor{\bsnm{Hanke},~\bfnm{Martin}\binits{M.}} \AND \bauthor{\bsnm{Neubauer},~\bfnm{Andreas}\binits{A.}}
(\byear{1996}).
\btitle{Regularization of inverse problems}
\bvolume{375}.
\bpublisher{Springer Science \& Business Media}.
\end{bbook}
\endbibitem

\bibitem{gerfo2008spectral}
\begin{barticle}[author]
\bauthor{\bsnm{Gerfo},~\bfnm{L~Lo}\binits{L.~L.}}, \bauthor{\bsnm{Rosasco},~\bfnm{Lorenzo}\binits{L.}}, \bauthor{\bsnm{Odone},~\bfnm{Francesca}\binits{F.}}, \bauthor{\bsnm{Vito},~\bfnm{E~De}\binits{E.~D.}} \AND \bauthor{\bsnm{Verri},~\bfnm{Alessandro}\binits{A.}}
(\byear{2008}).
\btitle{Spectral algorithms for supervised learning}.
\bjournal{Neural Computation}
\bvolume{20}
\bpages{1873--1897}.
\end{barticle}
\endbibitem

\bibitem{gu2013smoothing}
\begin{bbook}[author]
\bauthor{\bsnm{Gu},~\bfnm{Chong}\binits{C.}}
(\byear{2013}).
\btitle{Smoothing spline ANOVA models}
\bvolume{297}.
\bpublisher{Springer Science \& Business Media}.
\end{bbook}
\endbibitem

\bibitem{hansen2010discrete}
\begin{bbook}[author]
\bauthor{\bsnm{Hansen},~\bfnm{Per~Christian}\binits{P.~C.}}
(\byear{2010}).
\btitle{Discrete inverse problems: insight and algorithms}
\bvolume{7}.
\bpublisher{Siam}.
\end{bbook}
\endbibitem

\bibitem{koltchinskii2006local}
\begin{barticle}[author]
\bauthor{\bsnm{Koltchinskii},~\bfnm{Vladimir}\binits{V.}} \betal{et~al.}
(\byear{2006}).
\btitle{Local Rademacher complexities and oracle inequalities in risk minimization}.
\bjournal{The Annals of Statistics}
\bvolume{34}
\bpages{2593--2656}.
\end{barticle}
\endbibitem

\bibitem{laurent2000adaptive}
\begin{barticle}[author]
\bauthor{\bsnm{Laurent},~\bfnm{Beatrice}\binits{B.}} \AND \bauthor{\bsnm{Massart},~\bfnm{Pascal}\binits{P.}}
(\byear{2000}).
\btitle{Adaptive estimation of a quadratic functional by model selection}.
\bjournal{Annals of statistics}
\bpages{1302--1338}.
\end{barticle}
\endbibitem

\bibitem{mathe2003geometry}
\begin{barticle}[author]
\bauthor{\bsnm{Math{\'e}},~\bfnm{Peter}\binits{P.}} \AND \bauthor{\bsnm{Pereverzev},~\bfnm{Sergei~V}\binits{S.~V.}}
(\byear{2003}).
\btitle{Geometry of linear ill-posed problems in variable Hilbert scales}.
\bjournal{Inverse problems}
\bvolume{19}
\bpages{789}.
\end{barticle}
\endbibitem

\bibitem{mendelson2002geometric}
\begin{binproceedings}[author]
\bauthor{\bsnm{Mendelson},~\bfnm{Shahar}\binits{S.}}
(\byear{2002}).
\btitle{Geometric parameters of kernel machines}.
In \bbooktitle{International Conference on Computational Learning Theory}
\bpages{29--43}.
\bpublisher{Springer}.
\end{binproceedings}
\endbibitem

\bibitem{raskutti2012minimax}
\begin{barticle}[author]
\bauthor{\bsnm{Raskutti},~\bfnm{Garvesh}\binits{G.}}, \bauthor{\bsnm{Wainwright},~\bfnm{Martin~J}\binits{M.~J.}} \AND \bauthor{\bsnm{Yu},~\bfnm{Bin}\binits{B.}}
(\byear{2012}).
\btitle{Minimax-optimal rates for sparse additive models over kernel classes via convex programming}.
\bjournal{Journal of Machine Learning Research}
\bvolume{13}
\bpages{389--427}.
\end{barticle}
\endbibitem

\bibitem{raskutti2014early}
\begin{barticle}[author]
\bauthor{\bsnm{Raskutti},~\bfnm{Garvesh}\binits{G.}}, \bauthor{\bsnm{Wainwright},~\bfnm{Martin~J}\binits{M.~J.}} \AND \bauthor{\bsnm{Yu},~\bfnm{Bin}\binits{B.}}
(\byear{2014}).
\btitle{Early stopping and non-parametric regression: an optimal data-dependent stopping rule.}
\bjournal{Journal of Machine Learning Research}
\bvolume{15}
\bpages{335--366}.
\end{barticle}
\endbibitem

\bibitem{rudi2015less}
\begin{binproceedings}[author]
\bauthor{\bsnm{Rudi},~\bfnm{Alessandro}\binits{A.}}, \bauthor{\bsnm{Camoriano},~\bfnm{Raffaello}\binits{R.}} \AND \bauthor{\bsnm{Rosasco},~\bfnm{Lorenzo}\binits{L.}}
(\byear{2015}).
\btitle{Less is more: Nystr{\"o}m computational regularization}.
In \bbooktitle{Advances in Neural Information Processing Systems}
\bpages{1657--1665}.
\end{binproceedings}
\endbibitem

\bibitem{scholkopf2001learning}
\begin{bbook}[author]
\bauthor{\bsnm{Scholkopf},~\bfnm{Bernhard}\binits{B.}} \AND \bauthor{\bsnm{Smola},~\bfnm{Alexander~J}\binits{A.~J.}}
(\byear{2001}).
\btitle{Learning with kernels: support vector machines, regularization, optimization, and beyond}.
\bpublisher{MIT press}.
\end{bbook}
\endbibitem

\bibitem{schwarz1978estimating}
\begin{barticle}[author]
\bauthor{\bsnm{Schwarz},~\bfnm{Gideon}\binits{G.}} \betal{et~al.}
(\byear{1978}).
\btitle{Estimating the dimension of a model}.
\bjournal{The annals of statistics}
\bvolume{6}
\bpages{461--464}.
\end{barticle}
\endbibitem

\bibitem{stankewitz2019smoothed}
\begin{bmisc}[author]
\bauthor{\bsnm{Stankewitz},~\bfnm{Bernhard}\binits{B.}}
(\byear{2019}).
\btitle{Smoothed residual stopping for statistical inverse problems via truncated SVD estimation}.
\end{bmisc}
\endbibitem

\bibitem{stone1985additive}
\begin{barticle}[author]
\bauthor{\bsnm{Stone},~\bfnm{Charles~J}\binits{C.~J.}} \betal{et~al.}
(\byear{1985}).
\btitle{Additive regression and other nonparametric models}.
\bjournal{The annals of Statistics}
\bvolume{13}
\bpages{689--705}.
\end{barticle}
\endbibitem

\bibitem{wahba1977practical}
\begin{barticle}[author]
\bauthor{\bsnm{Wahba},~\bfnm{Grace}\binits{G.}}
(\byear{1977}).
\btitle{Practical approximate solutions to linear operator equations when the data are noisy}.
\bjournal{SIAM Journal on Numerical Analysis}
\bvolume{14}
\bpages{651--667}.
\end{barticle}
\endbibitem

\bibitem{wahba1987three}
\begin{bincollection}[author]
\bauthor{\bsnm{Wahba},~\bfnm{Grace}\binits{G.}}
(\byear{1987}).
\btitle{Three topics in ill-posed problems}.
In \bbooktitle{Inverse and ill-posed problems}
\bpages{37--51}.
\bpublisher{Elsevier}.
\end{bincollection}
\endbibitem

\bibitem{wahba1990spline}
\begin{bbook}[author]
\bauthor{\bsnm{Wahba},~\bfnm{Grace}\binits{G.}}
(\byear{1990}).
\btitle{Spline models for observational data}
\bvolume{59}.
\bpublisher{Siam}.
\end{bbook}
\endbibitem

\bibitem{wainwright2019high}
\begin{bbook}[author]
\bauthor{\bsnm{Wainwright},~\bfnm{Martin~J}\binits{M.~J.}}
(\byear{2019}).
\btitle{High-dimensional statistics: A non-asymptotic viewpoint}
\bvolume{48}.
\bpublisher{Cambridge University Press}.
\end{bbook}
\endbibitem

\bibitem{wasserman2006all}
\begin{bbook}[author]
\bauthor{\bsnm{Wasserman},~\bfnm{Larry}\binits{L.}}
(\byear{2006}).
\btitle{All of nonparametric statistics}.
\bpublisher{Springer Science \& Business Media}.
\end{bbook}
\endbibitem

\bibitem{wei2017early}
\begin{binproceedings}[author]
\bauthor{\bsnm{Wei},~\bfnm{Yuting}\binits{Y.}}, \bauthor{\bsnm{Yang},~\bfnm{Fanny}\binits{F.}} \AND \bauthor{\bsnm{Wainwright},~\bfnm{Martin~J}\binits{M.~J.}}
(\byear{2017}).
\btitle{Early stopping for kernel boosting algorithms: A general analysis with localized complexities}.
In \bbooktitle{Advances in Neural Information Processing Systems}
\bpages{6067--6077}.
\end{binproceedings}
\endbibitem

\bibitem{yang2017randomized}
\begin{barticle}[author]
\bauthor{\bsnm{Yang},~\bfnm{Yun}\binits{Y.}}, \bauthor{\bsnm{Pilanci},~\bfnm{Mert}\binits{M.}} \AND \bauthor{\bsnm{Wainwright},~\bfnm{Martin~J}\binits{M.~J.}}
(\byear{2017}).
\btitle{Randomized sketches for kernels: Fast and optimal nonparametric regression}.
\bjournal{The Annals of Statistics}
\bvolume{45}
\bpages{991--1023}.
\end{barticle}
\endbibitem

\bibitem{Yao2007}
\begin{barticle}[author]
\bauthor{\bsnm{Yao},~\bfnm{Yuan}\binits{Y.}}, \bauthor{\bsnm{Rosasco},~\bfnm{Lorenzo}\binits{L.}} \AND \bauthor{\bsnm{Caponnetto},~\bfnm{Andrea}\binits{A.}}
(\byear{2007}).
\btitle{On Early Stopping in Gradient Descent Learning}.
\bjournal{Constructive Approximation}
\bvolume{26}
\bpages{289--315}.
\bdoi{10.1007/s00365-006-0663-2}
\end{barticle}
\endbibitem

\bibitem{zhang2005boosting}
\begin{barticle}[author]
\bauthor{\bsnm{Zhang},~\bfnm{Tong}\binits{T.}}, \bauthor{\bsnm{Yu},~\bfnm{Bin}\binits{B.}} \betal{et~al.}
(\byear{2005}).
\btitle{Boosting with early stopping: Convergence and consistency}.
\bjournal{The Annals of Statistics}
\bvolume{33}
\bpages{1538--1579}.
\end{barticle}
\endbibitem

\end{thebibliography}

%
%
%
%

\end{document}